\documentclass[accepted]{uai2022} % after acceptance, for a revised
\usepackage[american]{babel}
\usepackage{natbib} % has a nice set of citation styles and commands
\usepackage{mathtools} % amsmath with fixes and additions
\usepackage{booktabs} % commands to create good-looking tables
\usepackage{tikz} % nice language for creating drawings and diagrams

\usepackage{macros}

\title{Safety Aware Changepoint Detection for Piecewise i.i.d. Bandits}

\author[1]{\href{mailto:<smukherjee27@wisc.edu>?Subject=Your UAI 2022 paper}{Subhojyoti Mukherjee}{}}
\affil[1]{%
    Electrical \& Computer Engineering Department\\
    UW-Madison\\
    Madison, Wisconsin, USA
}
\begin{document}
\maketitle

\begin{abstract}
In this paper, we consider the setting of piecewise i.i.d. bandits under a safety constraint. In this piecewise i.i.d. setting, there exists a finite number of changepoints where the mean of some or all arms change simultaneously. We introduce the safety constraint studied in \citet{wu2016conservative} to this setting such that at any round the cumulative reward is above a constant factor of the default action reward. We propose two actively adaptive algorithms for this setting that satisfy the safety constraint, detect changepoints, and restart without the knowledge of the number of changepoints or their locations. We provide regret bounds for our algorithms and show that the bounds are comparable to their counterparts from the safe bandit and piecewise i.i.d. bandit literature. We also provide the first matching lower bounds for this setting.  Empirically, we show that our safety-aware algorithms perform similarly to the state-of-the-art actively adaptive algorithms that do not satisfy the safety constraint.
\end{abstract}

\section{Introduction}
Consider a startup XYZ that wants to maximize revenue collection from ad placements when users land on their webpage. Revenue is generated when users click on the ads. The user preferences change over time but not quickly enough so that XYZ can focus on maximizing the revenue collection for some time before changing its strategy again. To do this XYZ must detect the user's new preferences and modify its suggestions. Due to budget constraints XYZ must make sure that the aggregate revenue collection must not fall below a certain threshold. The difficulty is that XYZ does not know placing which ads will surely result in revenue above the threshold. This constrains XYZ from randomly placing different ads on their landing webpage. On consultation with the industry experts XYZ comes up with a default action that is known historically to be a highly favored by users. Hence, XYZ comes up with a new safety constraint such that when their algorithm is unsure of which ad to place for some user it can fall back on this default action. The learning algorithm now has to balance between exploration under safety constraints and exploitation in this slowly changing environment.

% can explore safely without violating the budgetary constraints in this slowly changing user preference  environment.

% But XYZ is not sure of this default choice as user preferences are changing over time and so it creates weighted combination of these two factors (default action and threshold) to satisfy the budgetary constraints. 

The dilemma faced by XYZ can be modeled as a sequential decision making problem in the piecewise i.i.d. bandit setting under safety constraints. In the piecewise i.i.d. bandit setting the learner is provided with a set of arms $i\in \{0, 1,2,\ldots,K\}$ where we index the default arm (baseline) as $i=0$, and there exists a finite number of changepoints where the mean $\mu_i$ of one or more arms may change simultaneously. At every round $s\in \{1,2,\ldots,T\}$ the learner selects an action $I_s\in \{0, 1,2,\ldots,K\}$ and observes the feedback $X_{I_s}(s)$ where $\E[X_{I_s}(s)] =: \mu_{I_s}(s)$. Define $i^*$ as the optimal arm such that $\mu_{i^*}(s) > \mu_i(s)$ for all $i$. The goal of the learner is to maximize reward by quickly finding the optimal arm $i^*$ under the following safety constraint 
\begin{align}
    \sum_{s=1}^{t} X_{I_{s}}(s) \geq(1-\alpha) \sum_{s=1}^{t} X_{0}(s) \label{eq:constraint1}
\end{align}
for all $t \in\{1, \ldots, T\}$, where $\alpha\in (0,1]$ is the risk parameter, and $T$ is the horizon. The constraint in \cref{eq:constraint1} represents how much the learner is allowed to risk in conducting the exploration. For example, if $\alpha \rightarrow 0$ the learner is expected to sample arms that are at least better than the baseline arm $0$.  
%We assume trivially that $(1-\alpha)\mu_0(s) + \alpha B < \mu_{i^*}(s)$ because otherwise there is no learning as the constraint \eqref{eq:constraint1} will never be satisfied. 
% Note that 
% setting 
%$B=0$ retrieves the conservative bandit setup of \citet{wu2016conservative} and setting 
% $\alpha = 0$ retrieves the setup of \citet{AmaniAT19} for the stochastic bandit setting. So our setting generalizes their setup. 
The baseline arm represents expert's belief over the current user preferences and may change over time. Similar to the setting of \citet{wu2016conservative} we assume that the mean of the baseline arm is not known to the learner. Note that \citet{wu2016conservative} is not suited for the piecewise i.i.d setting. By change of belief of the expert we mean that only the value of the baseline arm changes and not the index.

The challenge in our setting is three-fold: \textbf{1)} Ensure that the safety constraint \eqref{eq:constraint1} on cumulative reward is satisfied. 
%by choosing optimal or near-optimal arms or fall back to the baseline arm. 
Consider the scenario where the risk parameter $\alpha = 1$. In this case satisfying the \cref{eq:constraint1} is easy as choosing any arm satisfies the constraint. However as $\alpha \rightarrow 0$ maintaining the safety constraint becomes difficult as exploration becomes limited or it will violate the safety constraints. \textbf{2)} Adapt to the piecewise i.i.d. nature of the environment. Observe that as the means of arms change abruptly at changepoints the algorithm must adapt or the safety constraint will be violated. Further, to detect changepoints the algorithm must conduct additional exploration without violating the \cref{eq:constraint1}. Note that \citet{wu2016conservative} do not consider any such piecewise i.i.d. setting. \textbf{3)} Finally, minimize the cumulative regret by quickly finding the optimal arm for each of the time segment between two changepoints. Our contributions are as follows:

1) We formulate the novel piecewise i.i.d. bandit setting under safety constraints. We show that the current state-of-the-art conservative algorithms \citep{wu2016conservative} as well as the changepoint detection algorithms are not equipped to handle the safety constraint in \cref{eq:constraint1} in this setting. 

2) 
%We extend the piecewise i.i.d. setting to the safe exploration under bandit feedback. 
We propose two actively adaptive algorithms that detect changepoints and restart by erasing past history of interactions. Simultaneously these algorithms ensure that the safety constraint is satisfied. The current changepoint detection algorithms \citep{besson2019generalized, mukherjee2019distribution, besson2020efficient} do not take into account the safety constraint and hence are not suited for our setting.

3) We provide theoretical guarantees for both of our algorithms and uncover new problem dependent terms that depends on the optimality gaps, changepoint gaps, the gap of the baseline arm, and the risk parameter $\alpha$. We also provide the first matching lower bounds for this setting. Empirically we show that our proposed methods perform comparably against safety oblivious changepoint detection algorithms.

% \fxnote{explain baseline arm may change at changepoints}
% \fxnote{contributions}

%Consider a startup XYZ that is conducting drug test against a resistant virus strain like COVID-19. XYZ wants to quickly find the best vaccine that can cure th

\vspace*{-1em}
\section{Related Works}
\vspace*{-1em}
Our work lies at the intersection of two interesting areas: $1)$ Changepoint detection in piecewise i.i.d bandits, and $2)$ Safe Sequential Decision Making. In piecewise i.i.d bandits it is assumed that the change of mean (drift) of an arm are well separated and significant enough to be detected. The previous works in this setting are broadly classified into two groups, viz. passively adaptive and actively adaptive algorithms. Passively adaptive algorithms such as Discounted UCB (\ducb) \citep{kocsis2006discounted}, Sliding Window UCB (\swucb) \citep{garivier2011upper}, and Discounted Thompson Sampling (\dts) \citep{raj2017taming} do not try to detect the changepoints and only focus on minimizing the regret over a short window of the time horizon. On the contrary, actively adaptive algorithms such as \expr \citep{allesiardo2017non}, \cducb \citep{liu2017change}, \cusum \citep{liu2017change}, \mucb \citep{cao2018nearly}, \glrucb \citep{besson2019generalized, besson2020efficient}, \adswitch \citep{auer2019achieving}, and \ucbcpd \citep{mukherjee2019distribution}  try to detect the changepoints and restart by erasing all the past history of interactions. Actively adaptive algorithm like \glrucb\!\!, \mucb\!\!, \ucbcpd has several advantages over passively adaptive algorithms. In environments where the changepoint gaps are large and well-separated the passively adaptive algorithms perform poorly (see \citep{besson2019generalized}). The \expr (an adaptive version of EXP3.S \citep{auer2002nonstochastic}) is more pessimistic than other actively adaptive algorithms like \ucbcpd\!\!, \glrucb as it uses the conservative exponential weighting algorithm EXP3 for changepoint detection, and hence, performs poorly in practice. \glrucb\!\! uses the Bernoulli generalized likelihood ratio test involving Kullback Leibler (KL) based divergence function as changepoint detector. The KL divergence function of \glrucb\!\! better exploits the geometry of (sub-)Bernoulli distributions and so it outperforms \mucb\!\!, \adswitch\!. Note that none of the above algorithms are safety aware. 
%well in practice compared to

The safe sequential decision making setup has recently garnered a lot of attention in machine learning \citep{amodei2016concrete, TurchettaB019}. 
%Our setting is closer to exploration under safety constraints in the piecewise i.i.d bandit setting. 
Closer to our setting are the works that study regret minimization in bandits under safety constraints such as \citet{wu2016conservative, KazerouniGAR17, AmaniAT19, GarcelonGLP20}. These works encode their safety requirements in the form of constraints on the cumulative rewards observed by the learner. This  setup also called conservative bandits as the exploration is constrained by the constraints on the cumulative reward. Note that while \citet{wu2016conservative} studies the unstructured stochastic and adversarial bandit setting, the \citet{KazerouniGAR17, AmaniAT19, GarcelonGLP20} study the linear bandit (structured) setting. Another line of related work \citep{MoradipariTA20, Pacchiano2006, khezeli2020safe} focuses on the idea of stagewise safety constraint where at every stage (round) the reward should be higher than a predetermined safety  threshold with high probability. Note that our setting of safety constraints on cumulative rewards cannot be directly applied to the stagewise setting. None of the above works deal with the setting of piecewise i.i.d bandits with slow drift (change of means). Note that while \citet{wu2016conservative} studies an adversarial setting, they do not take any assumption on the reward distributions and hence their exploration scheme is highly conservative for piecewise i.i.d. bandits. Similarly, conservative bandits studying the contextual (structured) bandit setting \citep{KazerouniGAR17} do not use the known  information about slow drift of means. Finally note that our setting is different than the thresholding bandit problem \citep{locatelli2016optimal, MukherjeeNSR17} where the goal is to find all the arms above a fixed threshold $B$ under the fixed budget setting.

% \fxnote{different than thresholding bandit}

\textbf{Notations:} Denote $[n] \coloneqq \{1,2,\ldots,n\}$. Define the set of arms as $[K]$ indexed from $i=1,2,\ldots,K$. The baseline arm is denoted by the index $0$. Note that the learner knows the index of the baseline arm but it does know the mean of the baseline arm. This is similar to the setting in \citet{wu2016conservative}. We define the set $[K]^+ \coloneqq [K]\cup\{0\}$ to indicate that baseline arm $0$ is included as well. We define the mean of the arm $i$ at round $s$ as $\mu_i(s)$, and the empirical mean of the arm till round $t$ as $\wmu_i(t)$. We denote the optimal arm as $i^*$ and the mean of the optimal arm at round $s$ as $\mu_{i^*}(s)$. The reward of the arm $i$ sampled at round $s$ is denoted by $X_i(s)$. We assume that the rewards are coming from a bounded distribution supported on $[0,1]$. We further denote the distribution of the $i$-th arm with mean $\mu_i(t)$ as $\nu(\mu_i(t))$. We denote the horizon (total rounds) as $T$. The safety threshold is denoted by $B$ and the risk parameter is denoted by $\alpha\in [0,1]$. 
% We denote 
% %the true safety set at round $s$ as 
% %\begin{align}
%     $\S(s) \coloneqq \{i\in [K]: \mu_{i}(s) \geq \alpha B\}$ %\label{eq:true-safety-set}
% %\end{align}
% % the number of times an arm $i$ is sampled till round $t$ as $N_i(t)$
% as the true safety set at round $s$ which contains all arms with mean above the risk threshold $\alpha B$. 
%We use the term safe and threshold arm interchangeably.
% Similarly we define the empirical safety set as 
% \begin{align}
%     \wS(s) \coloneqq \{i\in [K]: \wmu_{i}(s) \geq \alpha B\}. \label{eq:emp-safety-set}
% \end{align}
%We use the term policy and algorithm interchangeably. 
For brevity we denote the sequence of rounds between $s$ to $t$ as $s:t$. For clarity of presentation we overload the notation $\mu_i(\cdot)$ to denote either the mean at round $s$ as $\mu_i(s)$ or the mean over the rounds $1:t$ as $\mu_i(1:t)$. Similarly, $N(1:t)$ denotes that number of pulls of $i$ from $1:t$.
%Note that $N_i(s) = 0$ if $I_s\neq i$. So 
Define the empirical mean $\wmu_i(1:t) \coloneqq \frac{\sum_{s=1}^t X_{I_s}\indic{I_s = i}}{\sum_{s=1}^t\indic{I_s = i}} = \frac{\sum_{s=1}^t X_{I_s}\indic{I_s = i}}{N_i(1:t)}$ where $I_s$ denotes the arm pulled at round $s$. 

% We denote the history of empirical means $\bmu(\mathbf{1}, t)$ till round $t$ as 
% \begin{align*}
%     \bmu(\mathbf{1}, t) \coloneqq \{\{\wmu_1(1:s)\}_{s=1}^t, \{\wmu_2(1:s)\}_{s=1}^t, \ldots, \{\wmu_K(1:s)\}_{s=1}^t\}
% \end{align*}
% and the history of pulls $\bN(\mathbf{1},t)$ till round $t$ as
% \begin{align*}
%     \bN(\mathbf{1}, t) \coloneqq \{\{N_1(1:s)\}_{s=1}^t, \{N_2(1:s)\}_{s=1}^t, \ldots, \{N_K(1:s)\}_{s=1}^t\}
% \end{align*}
% %$\bN(\mathbf{1},t)$ at round $t$ as $\bN(\mathbf{1}, t) \coloneqq \{N_1(1:t), N_2(1:t), \ldots, N_K(1:t)\}$, 
% where $\mathbf{1} \coloneqq \{1\}_{i=1}^K$ denotes the starting index. 

% \fxnote{Talk about Pareto bandits. You observe a reward vector in $\mathbb{R}^d$}

\vspace*{-1em}
\section{\!\!Global Changepoint Detection}
We now define the setup for the Global Changepoint Setting (\gcs\!\!). Let the total number of changepoints till round $T$ be denoted by $G_T$ such that
\begin{align}
    G_T \!\coloneqq\! \#\left\{1 \leq s \leq T \mid \exists i\in [K]^+: \mu_{i}(s-1) \!\neq\! \mu_{i}(s)\right\}. \label{eq:G-T}
\end{align}
We define the global changepoints $t_{c_{0}} \!<\! t_{c_1} \!<\! t_{c_2} \!<\! \ldots \!<\! t_{c_{G_T}}$ such that the $g$-th  global changepoint is defined as:
\begin{align*}
    t_{c_g} \coloneqq \inf\{s > t_{c_{g-1}}: \forall i\in [K]^+, \mu_i(s-1) \neq \mu_i(s)\}.
\end{align*}
Hence at a global changepoint $t_{c_g}$ the mean of all the arms $i\in[K]^+$ change simultaneously. Let $t_{c_{0}} = 1$ by convention. Note, that the baseline mean changes at changepoints to signify the new belief of the experts based on updated user preferences. 
%So the baseline arm $0$ plays no part in changepoint detection though the learner may fall back on sampling it if the safety budget goes down. 
%We assume that at any round $s\in[T]$ the $\mu_0(s) < B < \mu_{i^*}(s)$. 
We define the changepoint segment between $t_{c_{g}}$ to $t_{c_{g+1}}-1$ as $\rho_{g}$. 
Note that the optimal arm for each changepoint segment $\rho_g$ may or may not be same. 
We further define a few notations. Let the confidence width of arm $i$ for rounds $1:t$ be defined as
\begin{align}
    \beta_i(1:t,\delta) = \sqrt{ \frac{2\log(4\log_2(t+1)/\delta)}{N_i(1:t)}} \label{eq:beta-delta}
\end{align}
with the standard condition that if $N_i(1:t) = 0$ then $\beta_i(1:t, \delta) = \infty$. In our case it  suffices to take the leading constant of $\beta_i(1:t,\delta)$ as $2$, though tighter bounds are known and can be used in practice, e.g. \citet{balsubramani2014sharp, tanczos2017kl, howard2021time}. These type of anytime bounds constructed with $\beta_i(1:t,\delta)$ are known to be tight in the sense that $\mathbb{P}\left(\bigcup_{t=1}^{\infty}\left\{\left|\widehat{\mu}_{i}(1: t)-\mu_{i}(1:t)\right| \geq \beta_i(1:t,\delta)\right\}\right) \leq \delta$ and that there exists an absolute constant $C \in(0,1)$ such that $\mathbb{P}\left(\left\{\left|\widehat{\mu}_{i}(1:t)-\mu_{i}(1:t)\right| \geq C \beta_i(1:t,\delta)\right.\right.$ for infinitely many $\left.\left.t \in \mathbb{N}\right\}\right)=1$ by the Law of the Iterated Logarithm \citep{hartman1941law}. Next we define the upper confidence bound for $i$ as 
\begin{align}
    U_i(1:t) \coloneqq \wmu_i(1:t) + \beta_i(1:t, \delta)
\end{align}
and the lower confidence bound from $1:t$ as
\begin{align}
    L_i(1:t) \coloneqq \wmu_i(1:t) - \beta_i(1:t,\delta).
\end{align}
We define the UCB arm $u_t$ at round $t$ as 
\begin{align}
    u_t \coloneqq \argmax_{i\in[K]} U_i(1:t) \label{eq:ucb-arm}
\end{align}
which is the arm with the highest uncertainty and needs to be explored more to get a better estimate of its true mean \citep{agrawal1995sample,auer2002finite}. Finally, we define the empirical safety budget as %$\wZ(1:t) \coloneqq$
\begin{align}
    \hspace*{-2em}\wZ(1\!:\!t) \!\!\coloneqq\!\! 
    %\hspace*{-2em}\wZ(1\!\!:\!\!t)
    & \sum_{s=1}^{t-1}\! \! L_{I_s}(1\!:\!s) \!+\!  L_{u_t}(1\!:\!t) 
    %\nonumber\\
    %%%%%%%%%%%%%%%%%%
    \!-\!  (1\!-\!\alpha)\!\sum_{s=1}^t\!\! U_{0}(1\!:\!t)\!  \hspace*{-1.5em}\label{eq:safe-budget}
\end{align}
which quantifies by how much the safety constraint is being violated. Also recall that the baseline arm is indexed as $0$, and the learner does not know the mean of the baseline arm. This is similar to the second setting in \citet{wu2016conservative}.

\vspace*{-1.1em}
\subsection{Safe Global Restart Algorithm}
\vspace*{-0.6em}
In this section we introduce the Safe Global Restart (\glb\!\!) algorithm which is a safety aware global changepoint detection algorithm. \glb is an actively adaptive algorithm and so it restarts by erasing the history of interactions once it detects a changepoint. We define the parameter $r_i$ for the $i$-th arm as the last restart round when a changepoint was detected and the arm $i$ history was erased. We define the last restart vector
\begin{align*}
    \br \coloneqq \{r_1, r_2, \ldots, r_K\} \cup \{r_0\}.
\end{align*}
% In the \gcs the restarting round $r_i = r_j$ for any $i,j\in[K]^+$. 
% The safety set for \gcs is
% \begin{align}
%     \wS(1:t) \coloneqq \{i\in [K]^+: \widehat{\mu}_i(r_i:t) \geq \alpha B \} \label{eq:emp-safety-set-gcs}
% \end{align}
% such that it is only estimated from the last restarting round of the arms. 
The safety budget for the \gcs is $\wZ(1:t)$ %\nonumber\\$
\begin{align}
    & \coloneqq \!\sum_{s=1}^{t-1}\! L_{I_s}(1:s)\! + \! U_{u_t}(1:t) 
    %\nonumber\\
    %%%%%%%%%%%%%%%%%%
    \!-\!  (1-\alpha)\sum_{s=1}^t\! U_{0}(1:t)  \nonumber\\
    %%%%%%%%%%%%%%%%%
    %%%%%%%%%%%%%%%%%
    & \overset{(a)}{=}\! \sum_{s=1}^{t-1}\! L_{I_s}(r_{I_s}:s) \!+\!  U_{u_t}(r_{u_t}:t) 
    %\nonumber\\
    %%%%%%%%%%%%%%%%%%
    \!- \! (1-\alpha)\sum_{s=1}^t\! U_{0}(r_{0}:t)\!
    \label{eq:safe-budget-gcs}
\end{align}
where, $(a)$ follows because when a changepoint is detected the history is erased for that arm $i \in [K^+]$.
%and so $\wZ(t) = \infty$.

We now state the main aspects of \glb.
%\fxnote{A note on Z and S for GCD}
\glb is initialized by sampling each arm once. Then at every round \glb decides to pull the UCB arm $u_t$ if $\wZ(1:t) \geq 0$ or the baseline arm $0$ if $\wZ(1:t) < 0$.
%calls the \sfe sub-routine with the empirical safety set $\wS(1:t)$ in \cref{eq:emp-safety-set-gcs}, the safety budget $\wZ(1:t)$ in \cref{eq:safe-budget-gcs}, and receives the next arm $I_{t} \in \{u_t, 0\}$ to sample. 
Then \glb samples the next arm, observes the reward $X_{I_{t}}(t)$ and updates the problem parameters. Finally \glb calls the \cpd changepoint detector sub-routine to detect any changepoint. If a changepoint is detected at round $t$ by \cpd then it erases the history of interactions for all arms (including baseline arm) and sets the restarting time for all arms $i\in[K]^+$ as $r_i = t$. We state the pseudo-code of the policy \glb in Algorithm \ref{alg:glb} and the key idea behind \cpd in the following \Cref{section:changepoint-detection}. 

\renewcommand{\algorithmiccomment}[1]{\hfill$\triangleright$\textit{#1}}

\begin{algorithm}[!th]
\caption{Safe Global Restart (\glb\!\!)}
\label{alg:glb}
\begin{algorithmic}[1]
\State \textbf{Input: } Risk parameter $\alpha\in [0,1)$
\State Set $r_i \!=\! 1, \forall i \in [K]^+$. Pull each arm once.
\For{$t= K^+ +1, K^+ +2,\ldots$}
\If{$\widehat{Z}(t) \geq 0$} 
\State Set $I_{t} = u_t$ from \cref{eq:ucb-arm} \Comment{Pull UCB arm}
\ElsIf{$\widehat{Z}(t) < 0$}
\State Set $I_{t} = 0$ \Comment{Baseline arm}
\EndIf
% \State $I_{t}$ = sfe($\wZ(1:t), \wS(1:t)$) 
\State Pull $I_{t}$ and observe $X_{I_{t}}(t)$. 
\State Update $\wmu_{I_{t}}(r_{I_{t}}:t), N_{I_{t}}(r_{I_{t}}:t)$, and $\wZ(r_{I_{t}}:t)$ in \cref{eq:safe-budget}.
\State Call \cpd$(\br, t, \text{global})$ \Comment{Call \cpd}
%\State Set $r_i(t) = r_i(t-1), \forall i \in[K]$.
\EndFor
\end{algorithmic} 
\end{algorithm}
\vspace*{-2em}
%\bmu(r:t), N(\br:t), 

\subsection{Changepoint Detection}
\label{section:changepoint-detection}

The sequential changepoint detection has a long history in the statistical community \citep{basseville1993detection, wu2007inference}. We explain the sequential changepoint detection through the following example: Consider a single arm $i$. Let at some round $t$ we have a collection of i.i.d. samples $X_i(1), X_i(2), \ldots, X_i(t)$ from a bounded distribution that is supported on $[0,1]$. The goal of changepoint detection is to find out whether all the $t$ samples have come from the same distribution with mean $\mu_i(1:t)$ or there exist a changepoint $\tau_{c_g}\in\mathbb{N}$ such that $X_i(1), X_i(2), \ldots, X_i(\tau_{c_g}-1)$ have mean $\mu_i(1:\tau_{c_g}-1)$ while $X_i({\tau_{c_g}}), X_i({\tau_{c_g}+1}), \ldots, X_i({t})$ have a different mean $\mu_i(\tau_{c_g}:t) \neq \mu_i(1:\tau_{c_g}-1)$. For notational convenience let us denote $\mu_i(1:\tau_{c_g})$ as $\mu'$ and $\mu_i(\tau_{c_g}+1:t)$ as $\mu^{''}$ respectively.  Hence, a sequential changepoint detector is defined as a stopping time ${\tau}^{chg} < \infty$ that rejects the null hypothesis $\mathcal{H}_0: (\exists \mu': \forall i \in \mathbb{N}, \E[X_i] = \mu')$ in favor of the alternate hypothesis $\mathcal{H}_1: (\exists \mu^{''}\neq \mu', \tau_{c_g} \in \mathbb{N}: X_i(1), X_i(2), \ldots, X_i({\tau_{c_g}}) \sim \nu(\mu'), \text{ and }  X_i({\tau_{c_g} \!+\! 1}), X_i({\tau_{c_g} \!+\! 2}),\! \ldots,\! X_i({t})\! \sim\! \nu(\mu^{''}))$. Previous works have studied the Generalized Likelihood Ratio Test (GLRT) to detect the changepoints using this hypothesis testing idea. Many of the previous works \citep{wilks1938large, siegmund1995using,maillard2019sequential,besson2019generalized, besson2020efficient} that have studied this setting used Generalized Likelihood Ratio Test (GLRT) to detect changepoints. The GLRT test works as follows: We first calculate the GLRT statistic defined by %$\operatorname{GLRT}(1:t)\coloneqq $
\begin{align*}
    \log \dfrac{\sup_{\mu', \mu^{''}, \tau_{c_g}<t}  L\left(X_i({1}), \ldots, X_i({t}) ; \mu', \mu^{''}, \tau_{c_g}\right)}{ \sup_{\mu'} L\left(X_i({1}), \ldots, X_i({t}) ; \mu'\right) }
\end{align*}
where we denote the term $L(X_i({1}), \ldots, X_i({t}) ; \mu')$ and $L(X_i({1}), \ldots, X_i({t}) ; \mu', \mu^{''}, \tau)$ as the likelihood of the first $t$ observations under hypothesis $\mathcal{H}_{0}$ and $\mathcal{H}_{1}$ respectively. Now if the GLRT statistic crosses a threshold $\widetilde{\beta}(t,\delta)$ then it indicates that there exists a changepoint and the null hypothesis $\mathcal{H}_0$ is rejected. A similar type of test, called the CUSUM test \citep{page1954continuous, liu2017change} has also been studied where the distributions $\nu(\mu_1)$ and $\nu(\mu_2)$ are completely known. Note that GLRT works in the case when both distributions are unknown but they come from the same canonical exponential family. A detailed discussion on this can be found in \citet{maillard2019sequential}.

\textbf{Confidence-based scan statistic:} An alternative to the GLRT based scan statistics is the confidence-based scan statistic that have been studied in \citet{mukherjee2019distribution}. In the confidence based scan statistic the total number of samples of arm $i$ from $1:t$ is divided into slices, and for each slice $s$ a confidence interval is built of the form
\begin{align}
    \hspace{-1em}\wmu_i(1: s) \pm \beta_i(1:s, \delta) \text { and }\wmu_i(s+1: t) \pm \beta_i(s+1:t, \delta) \label{eq:confidence-scan}
\end{align}
where  $\beta_i(1:s, \delta)$ is from \cref{eq:beta-delta}. Now if there exists some $s$ at which the confidence intervals do \emph{not} overlap such that
\begin{align}
    \tau_{c_g} \!\!\coloneqq\!\! \inf &\bigg\{t \in \mathbb{N}: \!\exists i\! \in [K]^+, \!\exists s \in[1, t],\left|\wmu_i(1\!:\! s)\!-\! \wmu_i(s+1\!:\! t)\right| \nonumber\\
    %%%%%%%%%%%%
    &>\beta_i(1:s, \delta)+\beta_i(s+1:t, \delta)\bigg\} \label{eq:confidence-chng-scan}
\end{align}
then report a changepoint at $s$. \citet{besson2020efficient} show that GLRT outperforms the confidence-based scan statistic as it better exploits the geometry of the Bernoulli distributions. In our work we use the confidence-based scan statistic as our goal is not to compete in the vanilla changepoint detection setting but to derive novel bounds for the safety aware piecewise i.i.d. setup proposed in this work.  

%the history of empirical means $\bmu(\br:t)$, the empirical pulls $\bN(\br:t)$,
Finally, we propose the \cpd changepoint detector sub-routine which is similar to the \ucbcpd algorithm in \citet{mukherjee2019distribution}. The \cpd takes input the restart vector $\br$, current time $t$, and the type $\in\{\text{global}, \text{local}\}$ indicating whether it is a global or a local changepoint setting. In this section we only discuss the global setting while the local setting is discussed in Section 5. The \cpd divides the total rounds $r_i:t$ into $(t - r_i)$ slices for each arm $i\in[K]^+$ and then proceeds to conduct the confidence-based scan statistics as discussed in \eqref{eq:confidence-chng-scan}. If there is a disjoint slice $s$ then \cpd reports a changepoint at $s$, then erases the history of interactions (including the baseline arm) and resets the restarting round counter $r_i, \forall i \in [K]^+$ to the current time $t$. We also reset the safe budget $\wZ(1:t)$ to $0$. Ideally in practice we can still continue the accrued safe budget to the next changepoint section from $\tau_{c_{g+1}} +1$ without setting it to $0$, but this makes our theoretical analysis more tedious.

% In certain cases we can ask the \cpd not to use detection on baseline, but this makes the pseudo-code tedious and so we omit it. The pseudo-code of \cpd is shown in \Cref{alg:ucbcpd}. 

% An alternative to the GLR also based on scan statistics, used by Mukherjee and Maillard (2019) consists in building individual confidence intervals for the mean in each segment, of the form
% $$
% \left[\hat{\mu}_{1: s} \pm \sqrt{\frac{\tilde{\beta}(s, \delta)}{2 s}}\right] \text { and }\left[\hat{\mu}_{s+1: n} \pm \sqrt{\frac{\tilde{\beta}(n-s, \delta)}{2(n-s)}}\right]
% $$
% and report that there is a change point if there exists $s$ such that these confidence interval are disjoint, i.e.
% $$
% \hat{\tau}_{\delta}^{\prime}=\inf \left\{n \in \mathbb{N}^{\star}: \exists s \in[1, n],\left|\hat{\mu}_{1: s}-\hat{\mu}_{s+1: n}\right|>\sqrt{\frac{\tilde{\beta}(s, \delta)}{2 s}}+\sqrt{\frac{\tilde{\beta}(n-s, \delta)}{2(n-s)}}\right\}
% $$
% By measuring distances with the appropriate KL divergence function, the Bernoulli GLR test better exploits the geometry of (sub-)Bernoulli distributions.
%\bmu(t), \bN(1:t), 

\begin{algorithm}[!th]
\caption{\cpd$(\br, t,\text{type})$}
\label{alg:ucbcpd}
\begin{algorithmic}[1]
\For{$i = 1,2,\ldots,K^+$}
\For{$t'= r_i, r_i + 1,\ldots, t$}
\If{$L_i(r_i:t') < U_i(t'+1:t)$ or $U_i(r_i:t') > \indent \indent L_i(t'+1:t)$}
\If{type $=$ global}
\State $\{\!\wmu_j(r_i\!:\!s), N_j(r_i\!:\!s)\}_{s = r_j}^{t} \!\!=\!\! \{0, 0\}_{s=r_j}^t$, \indent\indent\indent $\!\!\forall j\in\! [K^+]$. Set $ r_j \!=\! t, \!\forall\! j\in [K^+]$. Set \indent\indent\indent$\wZ(1:t) = 0$.
\Else
\State $\{\wmu_i(r_i\!\!:\!\!s), N_i(r_i\!\!:\!\!s)\}_{s = r_i}^{t} \!\!=\!\! \{0, 0\}_{s=r_i}^t$, \indent\indent\indent $r_i \!=\! t$. Set $\wZ(1:t) = 0$.
\EndIf
\EndIf
\EndFor
\EndFor
\end{algorithmic} 
\end{algorithm}
\vspace*{-2em}

\vspace*{-1em}
\subsection{Regret Analysis for SGR}
\vspace*{-1em}
\label{section:glbchange}
We denote the time interval segment $\rho_g \coloneqq [t_{c_{g}}, t_{c_{g+1}}-1]$ so that the segment $\rho_g$ starts at round $t_{c_g}$ and ends at $t_{c_{g+1}-1}$. 
Let $\mu_{i,g}$ denote the mean of arm $i$ for the segment $\rho_g$. 
%Let $\S_g\coloneqq \{i\in[K]^+: \exists s\in \rho_g, \mu_{i}(s) \geq\alpha B\}$ denote the true safe set for $\rho_g$. 
Let the changepoint gap between the segments $\rho_{g}$ and $\rho_{g + 1}$ be $\Delta^{chg}_{i,g}\coloneqq |\mu_{i,g} -\mu_{i,g+1}|$. We redefine the optimality gap for the segment $\rho_g$ as $\Delta^{opt}_{i,g} \coloneqq \mu_{i^*,g} - \mu_{i,g}$. 
%and threshold gap as $\Delta^{thd}_{i,g} \coloneqq |\mu_{i,g} - \alpha B|$. 
Let $\tau_{c_{g}}$ denote the first round when the changepoint $t_{c_g}$ is detected and \glb is restarted. Then we define the quantity $N^{chg}_{0,g}$ as the number of times the baseline arm is sampled from rounds $t_{c_g}$ till the detection of changepoint at $\tau_{c_g}$. Finally, we define the delay of detection of the $g$-th changepoint as 
\begin{align}
    \hspace*{-0.8em} d_{g} \!\!\coloneqq\!\! \left\lceil\! K \!+\!\!\! \left(\!\! \max\limits_{i\in[K]}\frac{B(T,\! \delta)}{(\Delta^{chg}_{i,g})^{2}} \!+\!  \frac{B(T,\! \delta)}{(\Delta^{chg}_{0,g})^{2}} + N^{bse}_{0,g}\!\right)\!4K \!\right\rceil\!\! \label{eq:global-delay0}
\end{align}
such that \glb detects the change at $t_{c_g}$ within $t_{c_g}+1$ till $t_{c_g} + d_{g}$ rounds with probability greater than $1-\delta$. We define the quantity $B(T,\delta) = 16\log(4\log_2(T/\delta))$. The quantity $N^{chg}_{0,g}$ denotes the number of samples of the baseline arms after the changepoint $t_{c_g}$ has occurred but not detected and is defined by 
\begin{align*}
    N^{bse}_{0,g} &\coloneqq \dfrac{1}{\alpha\mu_{0,g}}\sum_{i\in[K]}\dfrac{B(T,\delta)}{\max\{\Delta^{opt}_{i,g}, \Delta^{opt}_{0,g} - \Delta^{opt}_{i,g}\}}.
\end{align*}
Intuitively, $N^{chg}_{0,g}$ is the number of samples required after $t_{c_g}$ has occurred and the safe budget $\wZ(1:t)$ falls below $0$. In $N^{chg}_{0,g}$ if $\alpha$ is very small, we can still explore other arms as long as the baseline arm $0$ is close to the optimal arm $\mu_{i,g}^{*}$ (so that $\Delta^{opt}_{0,g}$ is small) while the other arms are clearly sub-optimal (i.e. the $\Delta^{opt}_{i,g}$ are large). If this happens then the sub-optimal arms are quickly discarded, while the $\wZ(1:t)$ stays positive and the regret penalty is small. 
% More precisely, if $\Delta_{0} \approx$ $T^{-b_{0}}$ and $\min _{i>0: \Delta_{i}>0} \Delta_{i} \approx T^{-b}$, then the regret penalty is
% $$
% O\left(T^{a+\min \left\{0, b-b_{0}\right\}}\right)
% $$
% small $\Delta_{0}$ and large $\Delta_{i}$ means $b-b_{0}<0$, giving a smaller penalty than the worst case of $O\left(T^{a}\right)$.
% such that \textbf{write the meaning.}
% \begin{aligned}
% &\text { Assumption } 4 \text { Define } d^{(k)}=d^{(k)}(\alpha, \delta)=\left\lceil\frac{4 K}{\alpha\left(\Delta^{(k)}\right)^{2}} \beta(T, \delta)+\frac{K}{\alpha}\right\rceil . \text { Then we assume that for all }\\
% &k \in\left\{1, \ldots, \Upsilon_{T}\right\}, \tau^{(k)}-\tau^{(k-1)} \geq 2 \max \left(d^{(k)}, d^{(k-1)}\right)
% \end{aligned}
We now define a mild assumption on separation of changepoints which is  standard in changepoint detection settings (see \citet{besson2019generalized, besson2020efficient}). Without this assumption the changepoints can be too frequent and cannot be detected before the next change happens. We require this assumption for our theoretical guarantees. Note that in the experiments we show that even when this assumption does not hold our proposed algorithms performs well. 
\begin{assumption}\textbf{(Separation of changepoints for \gcs\!)}
\label{assm:sgr} We assume that the for all $g\in \{0,1,2, \ldots, G_T\}$ two consecutive changepoints $t_{c_g}$ and $t_{c_{g+1}}$ are separated as
%\begin{align*}
    $t_{c_{g+1}} - t_{c_{g}} \geq 2\max\{d_{g}, d_{g+1}\}$, \text{ where $d_g$ is stated in $\eqref{eq:global-delay0}$}.
%\end{align*}
\end{assumption}
The \Cref{assm:sgr} assumes that two consecutive changepoints are separated enough to be detected by the changepoint detector. Note that our detection delay $d_g$ is larger than \citet{besson2020efficient} because between $t_{c_g}:\tau_{c_g}$ the budget $\wZ(1:t)$ may fall below $0$ and \glb may need to sample the baseline arm from the next segment $\rho_{g+1}$. We denote an event by $\xi$ and its complement by $\overline{\xi}$. Define the good event $\xi^{del}_g$ that all changepoints $g'\leq g$ have been detected with delay at most $d_{g'}$. Let the safe budget time set $\Qs(1: t) \coloneqq \left\{s\in [1:t]: \widehat{Z}(1:s) \geq 0 \right\}$ be the set of all rounds $1: t$ when $\widehat{Z}(1: s) \geq 0$. 
We can decompose the expected regret as
\begin{align}
    & \sum_{g=1}^{G_T}\bigg[\underbrace{\sum_{i=1}^K\sum_{\substack{s\in\Qs(\tau_{c_{g-1}}:t_{c_{g}}-1)}} \Delta^{opt}_i(s)\E[N_i(s)|\xi_g^{del}(s)]\Pb\left(\xi_g^{del}(s)\right)}_{\textbf{Part (A), UCB arm pulled, Safe budget $\wZ(\tau_{c_{g-1}}:s) \geq 0$}} \nonumber\\
    &+ \underbrace{\sum_{\substack{s\in\overline{\Qs}(\tau_{c_{g-1}}:t_{c_{g}}-1)}} \Delta^{opt}_0(s)\E[N_0(s)|\xi_g^{del}(s)]\Pb\left(\xi_g^{del}(s)\right)}_{\textbf{Part (B), Baseline arm pulled, Safe budget $\wZ(\tau_{c_{g-1}}:s) < 0$}} \nonumber\\
    %%%%%%%%%%%%%%%%%%%%%%%%
    & + \underbrace{\sum_{i=1}^K\sum_{s\in\Qs(t_{c_g}:\tau_{c_g}-1)}^{}\Delta^{opt}_i(s) \E[N_i(s)|\xi_g^{del}(s)]\Pb\left(\xi_g^{del}(s)\right)}_{\textbf{Part (C), Changepoint Pulls, Safe budget $\wZ(\tau_{c_{g-1}}:s) \geq 0$}}\nonumber\\
    %%%%%%%%%%%%%%%%%%%%%%%%
%     \end{align}
% \begin{align}
    &+ \underbrace{\sum_{s\in\overline{\Qs}(t_{c_g}:\tau_{c_g}-1)}^{}\Delta^{opt}_0(s) \E[N_0(s)|\xi_g^{del}(s)]\Pb\left(\xi_g^{del}(s)\right)}_{\textbf{Part (D), Changepoint Baseline Pulls, Safe budget $\wZ(\tau_{c_{g-1}}:s) < 0$}}\nonumber\\
    %%%%%%%%%%%%%%%%%%%%%%%
    &+ \sum_{s=\tau_{c_{g-1}}}^{T}\underbrace{\Pb(\overline{\xi_g^{del}(s)})}_{\textbf{Part (E), Total Detection Delay Error}}\bigg], \label{eq:g-regret-decomp-glb-chg0}
\end{align}
which follows by dividing the total rounds till $T$ into $G_T$ segments when the changepoint $t_{c_g}$ is detected at $\tau_{c_{g}}$. We then further subdivide it into two parts $\tau_{c_{g-1}}:t_{c_g}-1$ (rounds before $t_{c_g}$) and $t_{c_g}:\tau_{c_{g}}-1$ (rounds before detection of $t_{c_{g}}$. The four parts (A)-(D) further divides the two time segments $\tau_{c_{g-1}}:t_{c_g}-1$ and $t_{c_g}:\tau_{c_{g}}-1$ based on the available safe budget and using the definition of $\Qs(1:t)$. 
% In \cref{eq:g-regret-decomp-glb-chg0} the Parts (A)-(D) are similar to \Cref{thm:no-change}.
Now using \Cref{assm:sgr} we can show that two consecutive changepoints are separated enough to correctly control the pulls of the baseline arm, and detect the optimal arm given that $\xi^{del}_g$ holds. The main difference from previous changepoint detection works like \citet{besson2019generalized, mukherjee2019distribution} lies in controlling the detection delay and false alarm under the safety budget constraint. 
Using the \Cref{assm:sgr} and changepoint detection \Cref{lemma:conc-chg} we can show that the detection delay is bounded by $2\max\{d_g,d_{g+1}\}$ with high probability. We now define a few problem dependent parameters which is key to analyze the regret of \glb\!. We define the quantity $H^{(1)}_{i,g} \coloneqq \max\left\{\frac{1}{\Delta^{opt}_{i,g-1}}, \frac{\Delta^{opt}_{i,g-1}}{\left(\Delta^{chg}_{i,g-1}\right)^2}\right\}$ as the hardness of discarding the sub-optimal arm $i$ and avoiding false detection, the quantity $H^{(2)}_{i,g} \coloneqq \max_{j\in[K]^+}\frac{\Delta^{opt}_{i, g}}{\left(\Delta^{chg}_{j,g}\right)^2}$ as the hardness for detecting the changepoint $t_{c_g}$ due to $i$ after the changepoint has happened. Finally, the quantity $H^{(3)}_{i,g} \coloneqq \frac{\Delta^{opt}_{\max, g}}{\max\{\Delta^{opt}_{i,g}, \Delta^{opt}_{0,g} - \Delta^{opt}_{i,g}\}}$ captures the trade-off of selecting the baseline arm $0$ once the changepoint $t_{c_g}$ occurred. The regret of \glb is shown below.
% This result is summarized below. \textbf{write about key challenge}
%We now state the theorem for global restart algorithm \glb below.
\begin{customtheorem}{2}\label{thm:glb-change}
Let $H^{(1)}_{i,g}, H^{(2)}_{i,g}, H^{(3)}_{i,g}$ is defined above for the segment $\rho_{g}$. 
%Let $\Delta^{chg}_{i, g}$ be the changepoint gap at changepoint $t_{c_g}$ between $\rho_{g-1}$ and $\rho_g$. 
Then the expected regret of \glb is upper bounded by
\begin{align}
    \!\!\!\E[R_T] \!&\leq\! O\!\bigg(\!\!\bigg(\!\sum_{g=1}^{G_T}\sum_{i=1}^{K^+}\left(H^{(1)}_{i,g-1} 
    \!+\! H^{(2)}_{i,g}\right) \!\!+\!\! \sum_{g=1}^{G_T}\!\dfrac{1}{\alpha\mu_{0,g-1}}\!\!\sum_{i=1}^K \!H^{(3)}_{i, g-1} \nonumber\\
    %%%%%%%%%%%%%
    & + K\sum_{g=1}^{G_T}\dfrac{1}{\alpha\mu_{0,g}}\sum_{i=1}^K H^{(3)}_{i,g}\bigg)\log\left(\dfrac{\log_2T}{\delta}\!\right)\!\!\bigg). \label{eq:glb-regret-result0}
    %%%%%%%%%%%%%%%%%%%%%%%%%
    % &\!+\!\!\! KG_T\Delta^{opt}_{\max, g}\sum_{t=1}^T\delta\bigg). \label{eq:glb-regret-result}
    % \E[R_T] \!&\leq\! O\bigg(\sum_{g=1}^{G_T}\bigg(\sum_{i=1}^{K}\dfrac{1}{\Delta^{opt}_{i, g-1}} +  \sum_{i=1}^{K}\dfrac{\Delta^{opt}_{i, g-1}}{(\Delta^{thd}_{i, g-1})^2}  + D^{bse}_{0, g-1} \nonumber\\
    % %%%%%%%%%%%%%%%%%
    % &\!+\!\!\! \sum_{j\in\S_{g-1}}^{}\!\! D^{thd}_{j, g-1} \!+\! \sum_{i=1}^{K^+}\dfrac{\Delta^{opt}_{i, g-1}}{(\Delta^{chg}_{i, g})^2} \bigg)\log T\bigg). \label{eq:glb-regret-result}
\end{align}
\end{customtheorem}
In the result of \eqref{eq:glb-regret-result0} the first term is the optimality regret suffered before discarding the arm $i$ when the safety budget $\wZ(1:t) \geq 0$. The second term denotes the regret suffered for the changepoint detection due to arm $i$. The third term is the regret suffered for the section $\rho_{g-1}$ when the safety budget $\wZ(1:t) < 0$. Finally, the fourth term is the regret suffered due to the changepoint $t_{c_g}$ and safety budget $\wZ(1:t) < 0$.
% This is a worst case term resulting from the scenario that there could be just one safe arm (optimal arm $\mu_{i^*} \geq \alpha B)$ and \sfb may have to reject all arms to find that arm. The third term is the total number of times the baseline is chosen to satisfy the safety constraint \eqref{eq:constraint1}, and finally the fourth term is the total regret suffered for selecting the safe arms in $\S(1:t)$ to satisfy \eqref{eq:constraint1}. Now we state the following corollary in the case when the threshold $B=0$ and our setting matches that of \citet{wu2016conservative}
% has the same interpretation as \Cref{thm:no-change} but restricted to each individual segment $\rho_g$. 
% The first four terms are same as in \Cref{thm:no-change} (but defined for each segment $\rho_g$). The last term denotes the regret suffered till detection of the changepoint at $t_{c_g}$. 
\glb conducts no forced exploration which results in a fully gap-dependent bound. This result is different than the gap-dependent bound in \citet{mukherjee2019distribution} which does not contain the third and fourth terms in \eqref{eq:glb-regret-result0}. The bound in \eqref{eq:glb-regret-result0} is more informative than \Cref{corollary:glb-loc-change} as it correctly captures the dependence with respect to gaps for each segment $\rho_g$.

%A similar result can be found in \citet{mukherjee2019distribution} where \ucbcpd also conducts no forced exploration in the \gcs setting.

% $\Delta^{\text {opt }}:=\min _{k=1, \ldots, \Upsilon_{T}\left\{a: \Delta_{a}^{(k)}>0\right\}} \min _{a}^{(k)}, \text { and } \Delta^{\text {change }}:=\min _{k=1, \ldots, \Upsilon_{T}} \Delta^{c,(k)}=\min _{k=1, \ldots, \Upsilon_{T}} \max _{a=1, \ldots, A}\left|\mu_{a}^{(k)}-\mu_{a}^{(k-1)}\right|$ 

% \fxnote{drawback, throws away safety budget at change, does not know horizon, does not know G}
% \fxnote{No forced exploration in \gcs}
% \fxnote{Safety aware \gcs seems as hard as normal \gcs}

\vspace*{-1em}
\section{Local Changepoint Detection}
In the Local Changepoint Setting (\lcs\!\!) at any changepoint at least one arm has a change of mean. Recall $G_T$ from \eqref{eq:G-T}. 
% \begin{align*}
%     G_T \coloneqq \#\left\{1 \leq s \leq T \mid \exists i\in [K]: \mu_{i}(s-1) \neq \mu_{i}(s)\right\}.
% \end{align*}
We then define the local changepoints $t_{c_0} < t_{c_1} < \ldots < t_{c_{G_T}}$ such that the $g$-th  local changepoint is defined as 
\begin{align*}
    t_{c_g} \coloneqq \inf\{s > t_{c_{g-1}}: \exists i\in [K], \mu_i(s-1) \neq \mu_i(s)\}.
\end{align*}
So at a local changepoint the mean of one or more arms may change simultaneously. Let
% \begin{align*}
    $G^i_T \coloneqq \#\left\{1 \leq s \leq T \mid \mu_{i}(s-1) \neq \mu_{i}(s)\right\}$
% \end{align*}
denote the number of changepoints only for the $i$-th arm. It follows that $G^i_T \leq G_T$ but for some arms there could be arbitrary difference between these two quantities. Note that $J_T \coloneqq \sum_{i=1}^K G^i_T \leq K G_T$. Define $t^{i}_{c_g} \coloneqq \inf\{s > t^i_{c_{g-1}}: \mu_i(s-1) \neq \mu_i(s)\}$ as the $g$-th changepoint for the $i$-th arm. Again $t^i_{c_0} \!=\! 1$ for all arms $i\in [K]^+$ by convention. We denote the segment between rounds $t^i_{c_g}$ and $t^i_{c_{g+1}}\!-\!1$ as $\rho^i_{g}$. 

Consider the scenario that a learner has figured out the best arm $i$ in a segment $\rho^i_g$ but then at a local changepoint $t^j_{c_g}$, an arm $j\neq i$ becomes the new optimal arm (but arm $i$ does not change). So it will not be able to detect $t^j_{c_g}$ and will continue sampling arm $i$. Hence the leaner need to conduct forced exploration of all arms to have a good estimate of all arms. This idea is shown in \Cref{alg:loc} where at every round \loc first checks that the safety budget is positive and then either conducts forced exploration of arms with exploration parameters $\gamma$ (to be defined later) or samples the UCB arm $u_t$. If the safety budget is negative then \loc samples the baseline arm so that the budget becomes positive and \loc can explore again. 
% calls the \sfe to get a new arm to sample. 
Finally \loc calls the \cpd sub-routine with type as "local" to  detect local changepoints for an arm. Once the changepoint is detected it restarts only that arm as this is the local changepoint setting. The crucial thing to note is that we conduct forced exploration only when safety budget is available (positive).

\begin{algorithm}[!th]
\caption{Safe Local Restart (\loc\!)}
\label{alg:loc}
\begin{algorithmic}[1]
\State \textbf{Input: } Risk parameter $\alpha$, exploration factor $\gamma$
\State Set $r_i \!=\! 1, \forall i \in [K^+]$. Pull each arm once.
\For{$t= K^+ + 1, K^+ + 2,\ldots$}
\If{$\widehat{Z}(t) \geq 0$} 
\If{$t \!\mod\! \lfloor \frac{K}{\gamma}\rfloor \!\notin\! [K]$} 
\State  Set $I_{t} = u_t$ from \cref{eq:ucb-arm} \Comment{Pull UCB arm}
\ElsIf{ $t \!\mod\! \lfloor \frac{K}{\gamma}\rfloor \!\in\! [K]$}
\State Set $I_{t} = t\mod\lfloor \frac{K}{\gamma}\rfloor$ \Comment{Forced Exploration}
\EndIf
% \State Set $I_{t} = u_t$ from \cref{eq:ucb-arm} \COMMENT{Pull UCB arm}
\ElsIf{$\widehat{Z}(t) < 0$}
\State Set $I_{t} = 0$ \Comment{Baseline arm}
\EndIf
\State Pull $I_{t}$ and observe $X_{I_{t}}(t)$. 
\State Update $\wmu_{I_{t}}(r_{I_{t}}:t), N_{I_{t}}(r_{I_{t}}:t)$,and  $\wZ(r_{I_{t}}:t)$ in \cref{eq:safe-budget}.
\State Call \cpd$(\br, t, \text{local})$ \Comment{Call \cpd}
%\State Set $r_i(t) = r_i(t-1), \forall i \in[K]$.
\EndFor
\end{algorithmic} 
\end{algorithm}

\subsection{Regret Analysis for SLR}

We can extend the analysis of \glb to also bound the regret for \loc\!. The key difference between the two analysis is that \loc needs to bound the regret for each segment $\rho^i_{g}$ for all arms $i\in K^+$. To this effect we first redefine changepoint gap for any arm $i$ between the segments $\rho^i_{g}$ and $\rho^i_{g + 1}$ as $\Delta^{chg}_{i,g}\coloneqq |\mu_{i,g} -\mu_{i,g+1}|$, and the optimality gap as $\Delta^{opt}_{i,g} \coloneqq \mu_{i^*,g} - \mu_{i,g}$.
% and threshold gap as $\Delta^{k,thd}_{i,g} \coloneqq |\mu_{i,g} - \alpha B|$. Let $\S^k_g\coloneqq \{i\in[K]^+: \exists s\in \rho^k_g, \mu_{i}(s) \geq\alpha B\}$ be the true safe set for segment $\rho^k_g$.
Let $\tau^i_{c_{g}}$ denote the first round when the changepoint $t^i_{c_g}$ is detected and \loc is restarted for the arm $i$. Then we define detection delay for the changepoint at $t^i_{c_g}$ as 
\begin{align}
\hspace{-1em} d_{i,g} \!\coloneqq\! \left\lceil\! \dfrac{K}{\gamma} \!+\! \dfrac{4}{\gamma}\left( \frac{B(T, \delta)}{(\Delta^{chg}_{i,g})^{2}} \!+\!  \frac{B(T, \delta)}{(\Delta^{chg}_{0,g})^{2}} +N^{bse}_{0,g}\right) \!\right\rceil\!\!.\label{eq:local-delay0}
    %  d_{i,g} &\!\coloneqq\!\! \bigg\lceil\! \dfrac{K}{\gamma} \!\!+\!\!\! \bigg( \max_{i\in[K^+]}\frac{\beta_k(\!1\!:\!T,\! \delta)}{(\Delta^{k,chg}_{i,g})^{2}} \!+\!   \max_{i\in[K^+]}\frac{\beta_k(\!1\!:\!T,\! \delta)}{(\Delta^{k, thd}_{k,g})^{2}}\nonumber\\
    % %%%%%%%%%%%%%%%%
    % & \quad \! +\! \dfrac{\beta_0(\!1\!:\!T,\! \delta)}{\Delta^{k, bse}_{0,g}}\bigg)\!\!\dfrac{5K}{\gamma} \!\!\bigg\rceil \hspace*{-1.5em}\label{eq:local-delay0}
\end{align}
such that $t^i_{c_g}$ is detected within $t^i_{c_g}+1:t^i_{c_g} + d_{i,g}$ rounds and $\gamma$ is the exploration rate of \loc. Again we denote $B(T,\delta) = 16\log(4\log_2(T/\delta))$. Note that the delay $d_{i,g}$ scales with the exploration rate $\gamma$ so that \loc while conducting forced exploration can detect $t^i_{c_g}$. Similar assumption has also been taken in \citet{besson2019generalized, besson2020efficient}. We then define the following assumption for the separation of changepoints between $t^i_{c_g}$ and $t^i_{c_{g+1}}$. Again note that this assumption is only required for theoretical guarantees. Empirically we show that \loc performs well even when the \Cref{assm:lgr} is violated.
\begin{assumption}\textbf{(Separation of changepoints for \lcs)}
\label{assm:lgr} We assume that the for all $g\in \{0,1,2, \ldots, G^i_T\}$ two consecutive changepoints $t^i_{c_g}$ and $t^i_{c_{g+1}}$ are separated as
%\begin{align*}
    $t^i_{c_{g+1}} - t^i_{c_{g}} \geq 2\max\{d_{i,g}, d_{i,g+1}\}$, where $d_{i,g}$ is defined in \eqref{eq:local-delay0}.
%\end{align*}
\end{assumption}
Next we introduce the quantity $\overline{H^{(2)}_{i,g}} \coloneqq \frac{\Delta^{opt}_{i, g}}{\left(\Delta^{chg}_{i,g}\right)^2}$ as the hardness for detecting the $g$-th changepoint for the arm $i$. Note that in the \lcs setting the \loc algorithm is restarted only for the arm $i$ and so in the hardness we do not see the $\max$ over all arms like the \glb setting. 
Finally using the \Cref{assm:lgr} and the same analysis as in \Cref{thm:glb-change} but for each segment $\rho^i_g$ and each arm $i\in[K^+]$ we bound the regret for \loc in \Cref{thm:loc-change}. 

\begin{customtheorem}{3}\label{thm:loc-change}
Let $H^{(1)}_{i,g}, \overline{H^{(2)}_{i,g}}, H^{(3)}_{i,g}$ is defined above for the segment $\rho_{g}$. 
Then the expected regret of \loc is bounded by
\begin{align}
    \!\!\E[&R_T] \!\leq\! O\!\bigg(\!\!\bigg(\!\sum_{i=1}^{K^+}\!\sum_{g=1}^{G^i_T}\!\!\left(\!H^{(1)}_{i,g-1} 
    \!\!+\!\! \overline{H^{(2)}_{i,g}}\right) \!\!+\!\! \sum_{i=1}^K\!\!\sum_{g=1}^{G^i_T}\!\dfrac{1}{\alpha\mu_{0,g-1}}\!\!\sum_{i=1}^K\!\!H^{(3)}_{i, g-1} \nonumber\\
    %%%%%%%%%%%%%
    & + K\sum_{g=1}^{G_T}\dfrac{1}{\alpha\mu_{0,g}}\sum_{i=1}^K H^{(3)}_{i,g}\bigg)\log\left(\dfrac{\log_2T}{\delta}\right)\bigg) \!\!+\! \gamma T. \hspace*{-1em}\label{eq:loc-regret-result}
\end{align}
\end{customtheorem}
The result in \eqref{eq:loc-regret-result} has a similar interpretation to \eqref{eq:glb-regret-result0} (but with respect to each arm segment $\rho^i_g$ instead of global segment $\rho_g$) except the gap-independent term of $\gamma T$ which results from the forced exploration of arms. We state the following corollary to summarize the result of \glb and \loc in the "easy" case when all the gaps are same.
%\fxnote{write in simple way thm3}
\begin{customcorollary}{1}\label{corollary:glb-loc-change}\textbf{(Gap independent bound)}
Setting $\Delta^{opt}_{i,g} = \Delta^{chg}_{i,g} = \sqrt{\frac{K\log T}{T}}$ for all $i \in [K]^+$ and exploration rate  $\gamma=\sqrt{\frac{\log T}{T}}$ we obtain the gap independent regret upper bound of \glb and \loc as
\begin{align*}
    \E[R_T] &\!\leq\! O\!\left(G_TK\sqrt{KT\log T}\! +\! \dfrac{G_T \log T}{\alpha \mu_{0,\min}}\!\right), \textbf{(\glb)}\\
    %%%%%%%%%%%%%%%%%%
    \E[R_T] &\!\leq\! O\!\left(G_T\sqrt{KT\log T}\! +\! \dfrac{G_T \log T}{\alpha \mu_{0,\min}}\!\right),\textbf{(\loc\!\!)}
\end{align*}
where $\alpha$ is the risk parameter.
\end{customcorollary}
Comparing the above result with \glrucb (see \Cref{prop:changepoint}) we see that \glb (or \loc\!) picks up an additional factor of $1/(\mu_{0,min}\alpha)$ per changepoint which signifies the hardness of finding the safe set of actions for  maintaining the safety constraint \eqref{eq:constraint1}. 
%Note that $\S_{\max} \leq K$ and so the worst case regret may scale with $K^2$ (\loc) as opposed to $K$ in \glrucb. 
Further note that \glb suffers an extra factor of $K$ in its bound compared to \loc\!. This is because in the \gcs setting the algorithm restarts by erasing the history of interactions for all arms. Hence, our result mirrors a similar observation in \citet{besson2020efficient}. 
Moreover as $\alpha \rightarrow 0$ (risky setting) the regret increases proportionally. This is similar to the gap-independent bound in \citet{wu2016conservative} shown in \Cref{prop:conservative-upper} which holds for the stochastic setting without any changepoints. The key takeaway from this result is that the piecewise i.i.d. setting under safety constraints is no harder than the conservative stochastic setting of \citet{wu2016conservative} and piecewise i.i.d. setting given the changepoints are sufficiently separated. Finally we state the lower bound in the safe \gcs setting.
\begin{customtheorem}{3}\textbf{(Lower Bound)}
\label{thm:lower-bound}
Let $\mathcal{E}$, $\overline{\mathcal{E}}$ be two bandit environment and there exits a global changepoint at $t_{c_1} = T/2$. Let $\alpha>0$ be the safety parameter and $\mu_{0,\min}$ be the mean of the minimum safety mean over the changepoint segments. Then the lower bound is given by
% \begin{align*}
    $\E_{\mathcal{E}, \overline{\mathcal{E}}} [R_{T}] \geq \left\{\frac{K}{(16 e+8) \alpha \mu_{0, \min}} + \dfrac{\log T}{\alpha \mu_{0,\min}}, \frac{\sqrt{K T}}{\sqrt{32 e + 16}} + \dfrac{\log T}{\alpha \mu_{0,\min}}\right\}$.
% \end{align*}
\end{customtheorem}
The proof is given in \Cref{app:lower-bound} and follows from the change of measure argument. Additionally, we use the lower bound results from safe bandit setting of \citet{wu2016conservative} and changepoint detection setting of \citet{gopalan2021bandit} to arrive at the final result. Note that both of these works do not take into account the safe \gcs setting. Finally, comparing the results of \Cref{thm:lower-bound} and \Cref{corollary:glb-loc-change} we see that \glb matches the lower bound when $G_T=1$ except a factor of $O(K\sqrt{\log T})$. Similarly, since \gcs is a special case of \lcs, we see that \loc also matches the lower bound except a factor of $O(K\sqrt{\log T})$.

% Our proposed algorithms match the performance of \cucb and \glrucb

% our proposed algorithm in the piecewise i.i.d. setting achieves a similar per
%Next comparing our result againt 

\vspace*{-1em}
\section{Experiments}
\vspace*{-1em}
In this section we test \glb and \loc against safety oblivious actively adaptive algorithms \glrucb, \ucbcpd as well as passive  algorithm \ducb, and safety aware algorithms \cucb, and \umoss.  
A detailed discussion on the algorithms, hyper-parameter tuning, and time complexity of the algorithms is given in \Cref{app:addl-expt}. 
% Another experiment showing the performance of \sfb vs. \cucb in \emph{no changepoint setting} is shown in \Cref{app:addl-expt}. 
One further experiment showing the performance of \glb, \loc under different values of $\alpha$ is shown in \Cref{app:addl-expt}. All codes are provided in supplementary material. 

\textbf{Global Changepoint:} In this setting all the arms (including baseline) change at every changepoint. The environment consist of $6$ arms (including baseline) and the evolution of means with respect to rounds is shown in \Cref{fig:exptenv} (Left). The three global changepoints are at $t=2000, 4000$ and $6000$. We set risk parameter $\alpha = 0.7$. The performance of all the algorithms is shown in \Cref{fig:expt} (Left). The adaptive algorithms like \ucbcpd, \glrucb perform well as they detecting the changepoints and restart but they do not satisfy the safety constraints. Note that \glb performs similar to \glrucb, \ucbcpd as it also detects the changepoints and restarts as well as satisfy the safety constraint. It outperforms passive algorithm \ducb, and safety aware algorithm \cucb. The safety aware algorithm \cucb is not suited for the safety constraint \eqref{eq:constraint1} under piecewise i.i.d. setting as it always chooses the baseline arm and fail to achieve sub-linear regret. 
% We also implement the \loc algorithm which works for both \gcs and \lcs environment. We see that \loc outperforms passive \ducb but is beaten by \glb which conducts no forced exploration.

\begin{figure}[!ht]
\centering
\begin{tabular}{cc}
\label{fig:global_env}\hspace{-1.2em}\includegraphics[scale = 0.30]{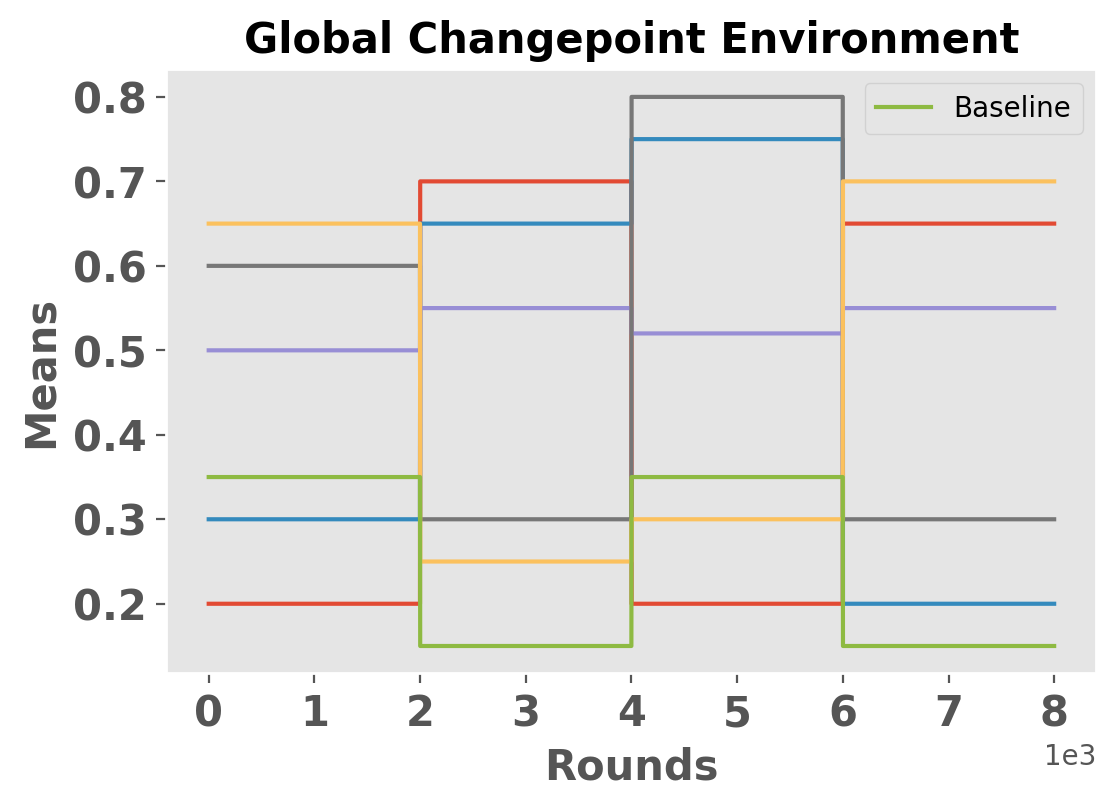} &
%%%%%%%%%%
\label{fig:local_env}\hspace{-1.2em}\includegraphics[scale = 0.30]{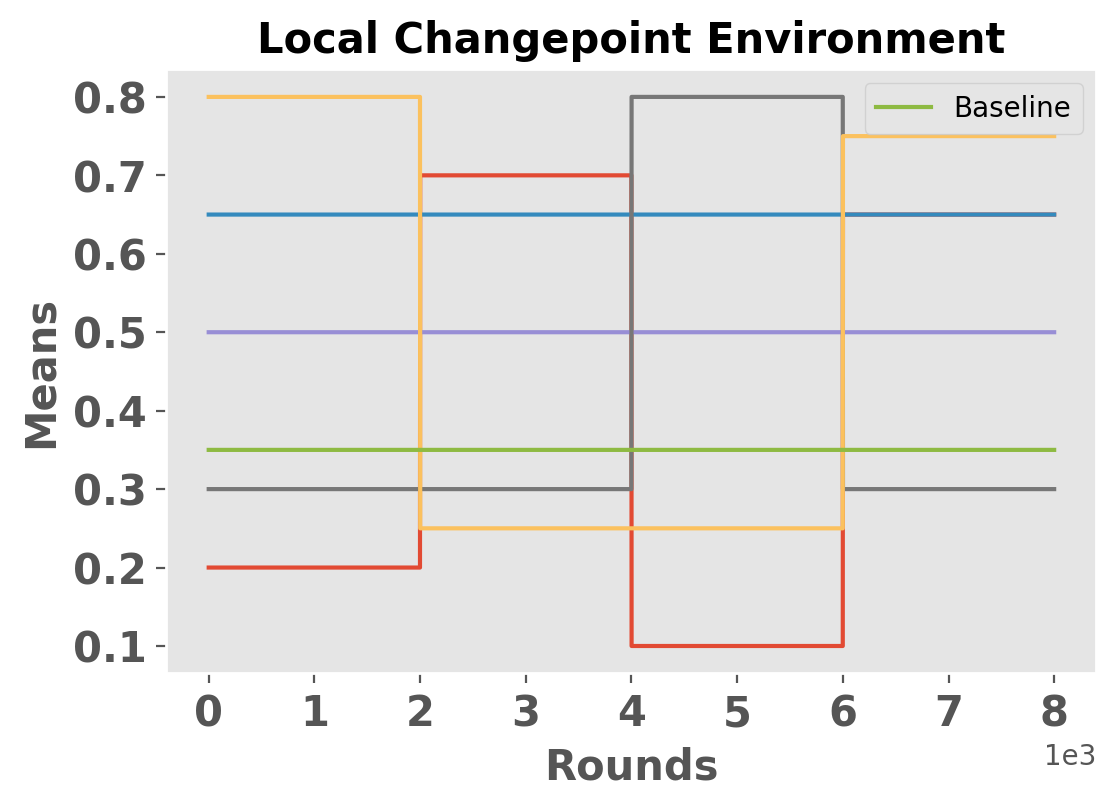}  
\end{tabular}
\caption{(Left) Global changepoint environment with $T=8000$, $K^+ = 6$ and changepoints at $t = 2000, 4000$ and $6000$. (Middle) Local changepoint environment with $T=8000$, $K^+ = 6$ and changepoints at $t = 2000, 4000$ and $6000$. Note that some arms do not change at these changepoints.}
\label{fig:exptenv}
%\vspace{-1.em}
\end{figure}

\begin{figure}[!ht]
\centering
\begin{tabular}{cc}
\label{fig:global}\hspace*{-1.2em}\includegraphics[scale = 0.31]{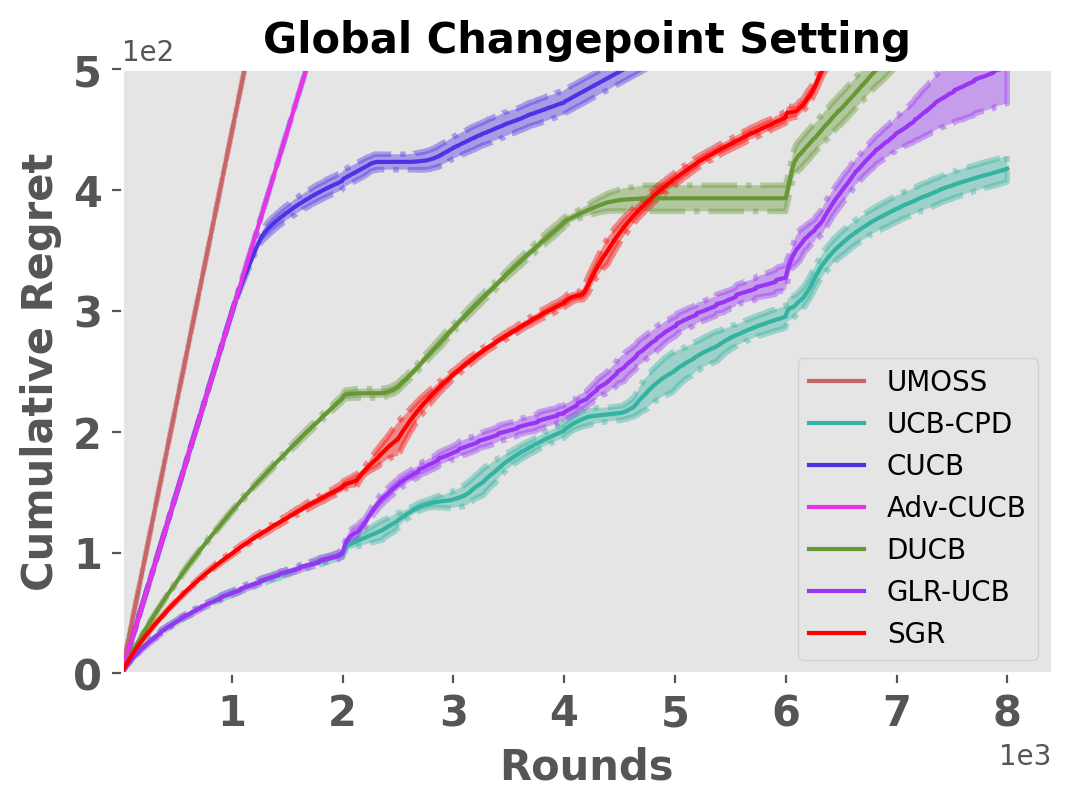} &
%%%%%%%%%%
\label{fig:local}\hspace*{-1.2em}\includegraphics[scale = 0.31]{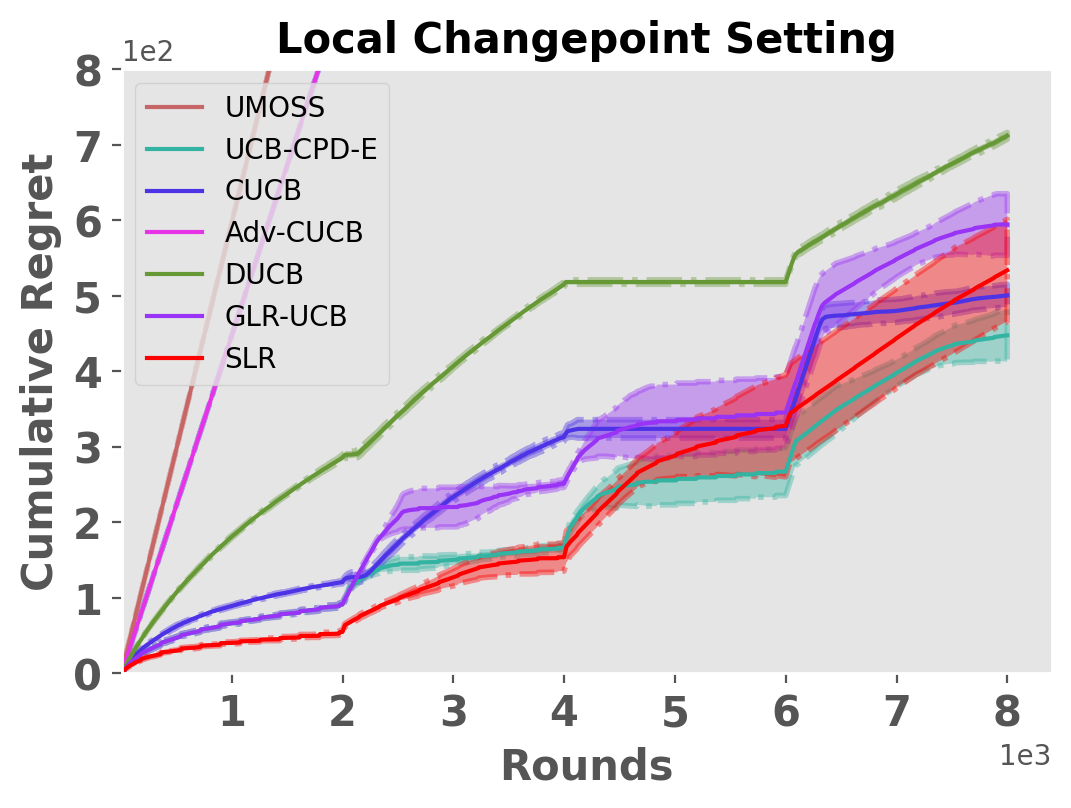} 
\end{tabular}
\caption{(Left) \gcs setting with $3$ changepoints and $6$ arms. (Right) \lcs setting with $3$ changepoints and $6$ arms.}
\label{fig:expt}
\vspace{-1.0em}
\end{figure}

\textbf{Local Changepoint:} In this setting at least one arm changes at every changepoint. 
%Note that in this environment the baseline arm does not change at all. 
We show that the environment in \Cref{fig:exptenv} (Right) and the performance of all the algorithms in \Cref{fig:expt} (Right). The three local changepoints are at $t=2000, 4000$ and $6000$. We set a constant baseline $\mu_0 = 0.35$, and risk parameter $\alpha =0.7$. 
Again we see that the safety aware algorithm \cucb fail to achieve sub-linear regret as it always chooses the baseline arm. On the other hand adaptive algorithms like \ucbcpde, \glrucb performs well in detecting the changepoints but they do not satisfy the safety constraints. Note that \glb performs similar to \glrucb, \ucbcpd as it also detects the changepoints and restarts as well as satisfy the safety budget. It outperforms passive algorithm \ducb, and safety aware algorithm \cucb. 

\textbf{Real Setting:} We show a real world experiment on the Movielens Dataset. In this experiment none of our modeling assumptions hold. We experiment with the Movielens dataset from February 2003 \citep{harper2016movielens}, where there are 6k users who give 1M ratings to 4k movies. We obtain a rank-$4$ approximation of the dataset over $128$ users and $128$ movies such that all  users prefer either movies $7$, $13$, $16$, or $20$ ($4$ user groups). The movies are the arms and we choose $30$ movies that have been rated by all the users. Hence, this testbed consists of $30$ arms and is run over $T = 8000$. The changepoints are at $t = 2000$, $t = 4000$, and $t = 6000$. Note that at each changepoint the means of some arms may or may not change so this is \lcs. For every changepoint segment, we uniform randomly sample an user from different user groups to simulate the piecewise i.i.d environment such that there is a change in the optimal arm. In this environment each arm has has a Gaussian distribution associated with it, where its mean evolve as shown in Figure \ref{fig:expt3} (Left). The baseline arm is set as $0.35$. As shown in Figure \ref{fig:expt3} (Right), in this environment \loc outperforms all the other algorithms including \cucb and \acucb. This is because the means of the arms are close to each other and the baseline arm mean is close to them.
\begin{figure}[!ht]
\centering
\begin{tabular}{cc}
\label{fig:movielens}\hspace*{-1.2em}\includegraphics[scale = 0.31]{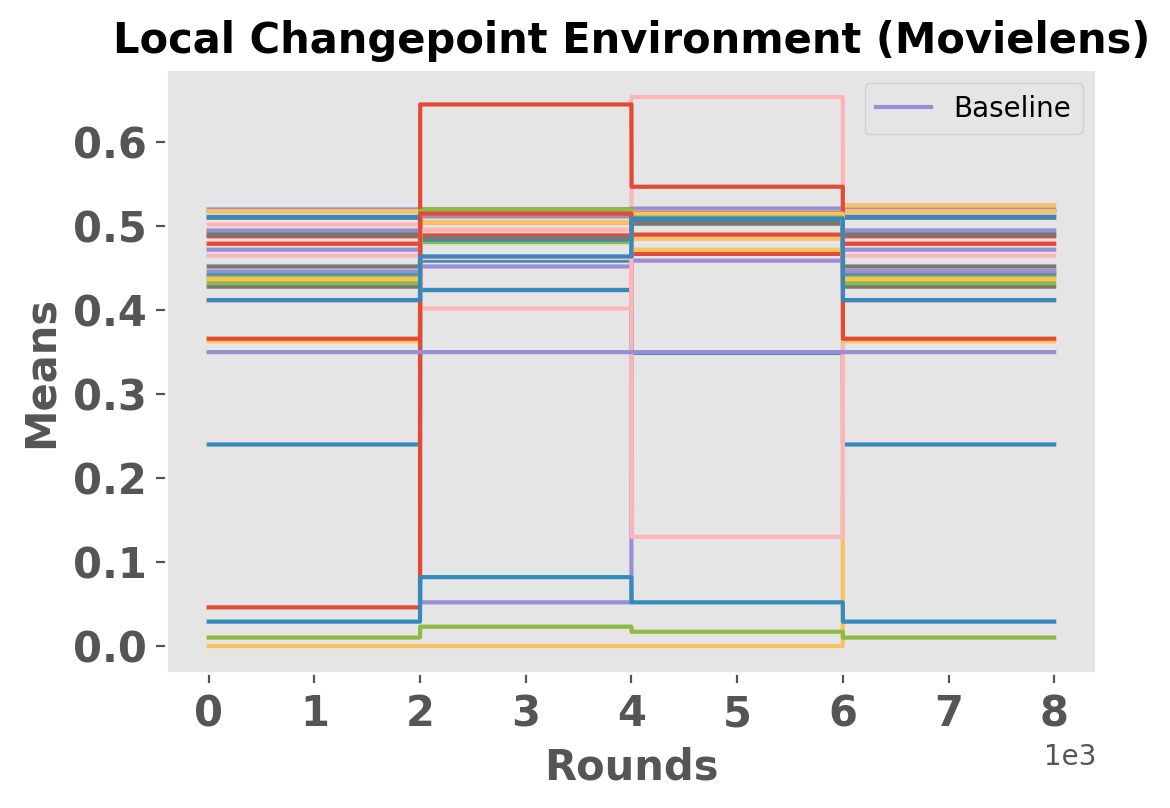} &
%%%%%%%%%%
\label{fig:loc-movielens}\hspace*{-1.2em}\includegraphics[scale = 0.31]{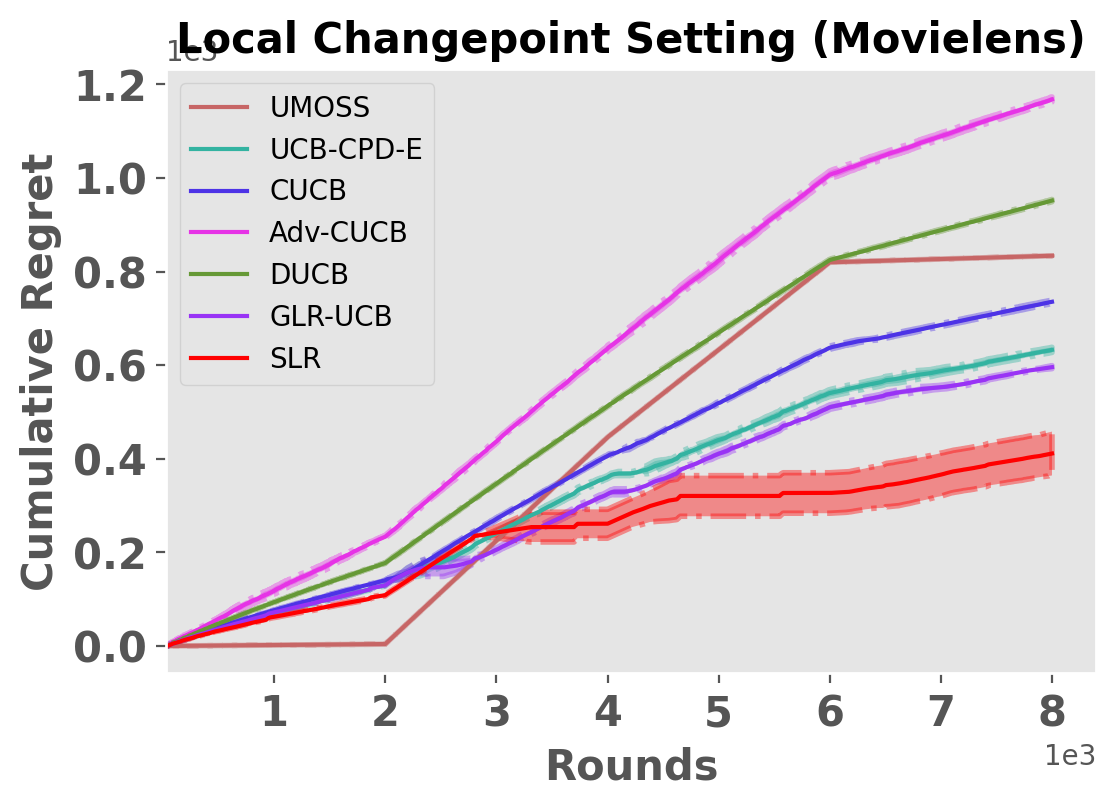} 
\end{tabular}
\vspace*{-1em}
\caption{(Left) \lcs setting with $3$ changepoints and $30$ arms. (Right) Regret in Movielens dataset.}
\label{fig:expt30}
\vspace{-1.0em}
\end{figure}
\vspace*{-1.0em}

\vspace*{-1em}
\section{Conclusion and Future Works} 
\vspace*{-1em}
In this paper we studied the safety aware piecewise i.i.d. bandits under a new safety constraint. We proposed two actively adaptive algorithms \glb and \loc which satisfy the safety constraints as well as detect changepoints and restart. We provided regret bounds on our algorithms and showed how the bounds compare with respect to safety aware bandits as well as adaptive algorithms. We also provided the first matching lower bounds for this setting. Future works include extending our setting to the rested and sleeping bandit setting under safety constraints. We also intend to explore experimental design approaches to piecewise i.i.d settings as in \citet{pukelsheim2006optimal, mason2021nearly, mukherjee2022chernoff}. Finally, incorporating variance aware techniques \citep{audibert2009exploration, mukherjee2018efficient} may further improve the the performance of our proposed algorithms.

\newpage
\bibliography{biblio}

\newpage
\onecolumn
\appendix
\section{Appendix}
\subsection{Probability Tools, Previous Theoretical Results and GLRT Discussion}
\label{app:prob-tools}
\begin{proposition}\label{prop:doob}\textbf{(Restatement of Theorem 9.2 in \citet{lattimore2020bandit})}
\label{prop:Doob}
Let \(X_{1}, X_{2}, \ldots, X_{T}\) be a sequence of independent \(\sigma\)-subgaussian
random variables with $\E[X_1] = \E[X_2] = \ldots = \E[X_T] = 0$ and \(M_{t}=\sum\limits_{s=1}^{t} X_{s} .\) Then for any \(\varepsilon>0\),
\begin{align*}
    \mathbb{P}\left(\text {exists } t \leq T: M_{t} \geq \varepsilon\right) \leq \exp \left(-\frac{\varepsilon^{2}}{2 T \sigma^{2}}\right).
\end{align*}
\end{proposition}

\begin{proposition}\label{prop:conservative-upper}\textbf{(Restatement of Theorem 2 in \citet{wu2016conservative})}
In any stochastic environment where the mean of the arms $\mu_{i} \in[0,1]$ with 1-subgaussian noise, then Conservative UCB (\cucb\!\!) satisfies the following with probability at least $1-\delta$ 
\begin{align*}
 \sum_{s=1}^{t} \mu_{I_{s}} &\geq(1-\alpha) \mu_{0} t \quad\text { for all } t \in\{1, \ldots, T\}, \quad \textbf{(Safety Constraint)}\\
 \E[R_T] &\leq\!\! O\bigg(\sum_{i=1}^K\dfrac{\log T}{\Delta^{opt}_i}
    %%%%%%%%%%%%%%%%%
    \!+\! \dfrac{1}{\alpha\mu_0}\!\!\sum_{i=1}^K\dfrac{\log T}{\max\{\Delta^{opt}_{i}, \Delta^{opt}_0 \!-\! \Delta^{opt}_{i}\}}\!\bigg), \quad \textbf{(Gap-dependent bound)}
    \\
\E[{R}_{T}] &\leq O\left(\sqrt{T K \log T}+\frac{K \log T}{\alpha \mu_{0}}\right). \quad \textbf{(Gap-independent bound)}
\end{align*}
\end{proposition}

\begin{proposition}
\label{prop:conservative-lower}\textbf{(Restatement of Theorem 9 in \citet{wu2016conservative})}
Let $\mu_{i} \in[0,1]$ for all $i\in[K]$ and $\mu_{0}$ satisfies the following
\begin{align*}
\min \left\{\mu_{0}, 1-\mu_{0}\right\} \geq \max \{1 / 2 \sqrt{\alpha}, \sqrt{e+1 / 2}\} \sqrt{K / T}.
\end{align*}
Then any algorithm satisfies the safety constraint $\E_{\mu} \sum\limits_{s=1}^{T} X_{s, I_{s}} \geq(1-\alpha) \mu_{0} T$. Moreover there is some $\mu \in[0,1]^{K}$ such that the expected regret of the algorithm satisfies 
\begin{align*}
    \E_{\mu}[R_{T}] \geq \max \left\{\frac{K}{(16 e+8) \alpha \mu_{0}}, \frac{\sqrt{K T}}{\sqrt{16 e+8}}\right\}.
\end{align*}
% $\mathbb{E}_{\mu} R_{n} \geq B$ where
% $$
% B=\max \left\{\frac{K}{(16 e+8) \alpha \mu_{0}}, \frac{\sqrt{K n}}{\sqrt{16 e+8}}\right\}
% $$
\end{proposition}

\begin{proposition}\label{prop:changepoint}\textbf{(Restatement of Corollary 6 and Corollary 9 in \citet{besson2020efficient})}
The gap-independent bound of \glrucb when the number of changepoints is unknown, exploration parameter $\gamma = \sqrt{\frac{\log T}{T}}$, and confidence $\delta = \frac{1}{\sqrt{T}}$ is given by
\begin{align*}
    \E[R_T] \leq O\left(KG_T\sqrt{KT\log T}\right), \quad\textbf{(Global Setting)}\\
    %%%%%%%%%%%%%%%%%%
    \E[R_T] \leq O\left(G_T\sqrt{KT\log T}\right). \quad\textbf{(Local Setting)}
\end{align*}

\begin{proposition}\label{prop:lower-bound-changepoint}\textbf{(Restatement of Theorem 2 in \citet{gopalan2021bandit})}
Let $0<\delta \leq \frac{1}{10}$ and $m \geq 1$ be a priori fixed time such that the probability of an admissible change detector stopping before time $m$ under the null hypothesis (no change) is at most $\delta \in(0,1)$. Let $\tau$ be the stopping time for an algorithm. Let \(\mathcal{E} \in \mathbb{R}^K\) and \(\overline{\mathcal{E}}\in \mathbb{R}^K\) be two bandit environments consisting of $K$ arms. Define $\mathcal{E}^i$ as the $i$-th component of $\mathcal{E}$. Let \(\Pb_{\mathcal{E}}\) and \(\Pb_{\overline{\mathcal{E}}}\) be two probability measures induced by some $T$-round interaction of the algorithm with \(\mathcal{E}\) and \(\overline{\mathcal{E}}\) respectively. Then for any bandit changepoint algorithm satisfying $\mathbb{P}^{}[\tau<m] \leq \delta$, we have
\begin{align*}
    \mathbb{E}^{}[\tau] \geq \min \left\{\frac{\frac{1}{20} \log \frac{1}{\delta}}{\max _{i \in [K]}\KL\left(\Pb_{\mathcal{E}^i}\| \Pb_{\overline{\mathcal{E}}^i}\right)}, \frac{m}{2}\right\}.
\end{align*}
\end{proposition}

% $$
% \begin{array}{l}
% \text { 1. Choosing } \alpha=\sqrt{\frac{\ln (T)}{T}}, \delta=\frac{1}{\sqrt{T}} \text { gives } R_{T}=\mathcal{O}\left(\frac{K}{\left(\Delta^{\text {change }}\right)^{2}} \Upsilon_{T} \sqrt{T \ln (T)}+\frac{(K-1)}{\Delta^{\text {opt }}} \Upsilon_{T} \ln (T)\right) \\
% \text { 2. Choosing } \alpha=\sqrt{\frac{\Upsilon_{T} \ln (T)}{T}}, \delta=\frac{1}{\sqrt{\Upsilon}_{T} T} \text { gives } R_{T}=\mathcal{O}\left(\frac{K}{\left(\Delta^{\text {chang } t}\right)^{2}} \sqrt{\Upsilon_{T} T \ln (T)}+\frac{(K-1)}{\Delta^{\text {opt }}} \Upsilon_{T} \ln (T)\right) \text { . }
% \end{array}
% $$
\end{proposition}

\subsection{Proof of Regret bound for Safety Aware Global Restart}
\label{app:regret-glb-change}

\begin{customtheorem}{2}\textbf{(Restatement)} 
Let $H^{(1)}_{i,g}, H^{(2)}_{i,g}, H^{(3)}_{i,g}$ is defined above for the segment $\rho_{g}$. 
%Let $\Delta^{chg}_{i, g}$ be the changepoint gap at changepoint $t_{c_g}$ between $\rho_{g-1}$ and $\rho_g$. 
Then the expected regret of \glb is upper bounded by
\begin{align}
    \E[R_T] \!&\leq\! O\bigg(\bigg(\sum_{g=1}^{G_T}\sum_{i=1}^{K^+}\left(H^{(1)}_{i,g-1} 
    + H^{(2)}_{i,g}\right) \!\!+\!\! \sum_{g=1}^{G_T}\dfrac{1}{\alpha\mu_{0,g-1}}\sum_{i=1}^KH^{(3)}_{i, g-1} 
    %%%%%%%%%%%%%
     + K\sum_{g=1}^{G_T}\dfrac{1}{\alpha\mu_{0,g}}\sum_{i=1}^K H^{(3)}_{i,g}\bigg)\log\left(\dfrac{\log_2T}{\delta}\right)\bigg). \label{eq:glb-regret-result}
    %%%%%%%%%%%%%%%%%%%%%%%%%
\end{align}
\end{customtheorem}

\begin{proof}
\textbf{Step 1 (Regret Decomposition) :} First recall that the safe budget time set $\Qs(1: t) \coloneqq \left\{s\in [1:t]: \widehat{Z}(1:s) \geq 0 \right\}$ is the set of all rounds $1: t$ when $\widehat{Z}(1: s) \geq 0$. Also recall that the first round the global changepoint $t_{c_g}$ is detected is denoted by $\tau_{c_g}$ defined as follows:
\begin{align*}
    \tau^{}_{c_g}:=\inf \{t \in \mathbb{N}: \exists s \in[1, t], \exists i\in [K],|\widehat{\mu}_i(1: s)-\widehat{\mu}_i(s+1: t)|>\beta_i(1: s, \delta)+\beta_i(s+1: t, \delta)\}.
\end{align*}
Let $\xi^{del}_{g}(s)$ (to be defined later) denote the good event that all changepoints $g'\leq g$ has been successfully detected before the round $s$. We then define the expected regret till round $T$ as follows:
%for some $t\in [\tau_{C_{G-1}}, t_{C_G}] $
\begin{align}
    \E[R_T] &= \sum_{s=1}^T\left(\mu_{i^*}(s) - \E[X_{I_s}(s)]\right) \nonumber\\
    %%%%%%%%%%%%%%%%%%%%%%%
    % &\overset{(a)}{\leq} \sum_{s=1}^T\left(\mu_{i^*}(s) - \E[X_{I_s}(s)|\xi^{del}_g]\Pb(\xi^{del}_g)\right) + \sum_{s=1}^T\left(\mu_{i^*}(s) - \E[X_{I_s}(s)|\overline{\xi^{del}_g}]\Pb(\overline{\xi^{del}_g})\right) \nonumber\\
    %%%%%%%%%%%%%%%%%%%%%%%
    &\overset{(a)}{\leq} \sum_{g=1}^{G_T}\left[\!\sum_{s=\tau_{c_{g-1}}}^{t_{c_g}-1}\left(\!\mu_{i^*}(s) \!-\! \E[\mu_{I_s}(s)|\xi_g^{del}(s)]\right)\Pb\left(\xi_g^{del}(s)\right) \!+\! \sum_{s=t_{c_g} }^{\tau_{c_{g}}-1}\left(\mu_{i^*}(s) \!-\! \E[\mu_{I_s}(s)|\xi_g^{del}(s)]\right)\Pb\left(\xi_g^{del}(s)\right) +\!\! \sum_{s=\tau_{c_{g-1}}}^{T}\Pb(\overline{\xi_g^{del}(s)}) \!\!\right] \nonumber\\
    %%%%%%%%%%%%%%%%%%%%%%%%
    %%%%%%%%%%%%%%%%%%%%%
    &\overset{(b)}{=} \sum_{g=1}^{G_T}\bigg[\sum_{s\in\Qs(\tau_{c_{g-1}}:t_{c_{g}}-1)} \left(\mu_{i^*}(s) - \E[\mu_{I_s}(s)|\xi_g^{del}(s)]\right)\Pb\left(\xi_g^{del}(s)\right) + \sum_{s\in\overline{\Qs}(\tau_{c_{g-1}}:t_{c_{g}}-1)} \left(\mu_{i^*}(s) - \E[\mu_{I_s}(s)|\xi_g^{del}(s)]\right)\Pb\left(\xi_g^{del}(s)\right)\nonumber\\
    %%%%%%%%%%%%%%%%%%%%%%%
    &\qquad + \sum_{s=t^k_{c_g}}^{\tau_{c_{g}} - 1}\left(\mu_{i^*}(s) - \E[ \mu_{I_s}(s)|\xi_g^{del}(s)]\right)\Pb\left(\xi_g^{del}(s)\right)] + \sum_{s=\tau_{c_{g-1}}}^{T}\Pb(\overline{\xi_g^{del}(s)})\bigg]\nonumber\\
    %%%%%%%%%%%%%%%%%%%%%%%%%
    %%%%%%%%%%%%%%%%%%%%%%%%
    &\overset{(c)}{=} \sum_{g=1}^{G_T}\bigg[\underbrace{\sum_{i=1}^K\sum_{\substack{s\in\Qs(\tau_{c_{g-1}}:t_{c_{g}}-1)}} \Delta^{opt}_i(s)\E[N_i(s)|\xi_g^{del}(s)]\Pb\left(\xi_g^{del}(s)\right)}_{\textbf{Part (A), UCB arm pulled, Safe budget $\wZ(\tau_{c_{g-1}}:s) \geq 0$}} 
    + \underbrace{\sum_{\substack{s\in\overline{\Qs}(\tau_{c_{g-1}}:t_{c_{g}}-1)}} \Delta^{opt}_0(s)\E[N_0(s)|\xi_g^{del}(s)]\Pb\left(\xi_g^{del}(s)\right)}_{\textbf{Part (B), Baseline arm pulled, Safe budget $\wZ(\tau_{c_{g-1}}:s) < 0$}} \nonumber\\
    %%%%%%%%%%%%%%%%%%%%%%%%
    &\qquad + \underbrace{\sum_{i=1}^K\sum_{s\in\Qs(t_{c_g}:\tau_{c_g}-1)}^{}\Delta^{opt}_i(s) \E[N_i(s)|\xi_g^{del}(s)]\Pb\left(\xi_g^{del}(s)\right)}_{\textbf{Part (C), Changepoint Pulls, Safe budget $\wZ(\tau_{c_{g-1}}:s) \geq 0$}} + \underbrace{\sum_{s\in\overline{\Qs}(t_{c_g}:\tau_{c_g}-1)}^{}\Delta^{opt}_0(s) \E[N_0(s)|\xi_g^{del}(s)]\Pb\left(\xi_g^{del}(s)\right)}_{\textbf{Part (D), Changepoint Baseline Pulls, Safe budget $\wZ(\tau_{c_{g-1}}:s) < 0$}}\nonumber\\
    %%%%%%%%%%%%%%%%%%%%%%%
    &\qquad + \sum_{s=\tau_{c_{g-1}}}^{T}\underbrace{\Pb(\overline{\xi_g^{del}(s)})}_{\textbf{Part (E), Total Detection Delay Error}}\bigg],
    \label{eq:g-regret-decomp-glb-chg}
\end{align}
where $(a)$ follows by introducing the good changepoint detection event $\xi_g^{del}(s)$, $(b)$ follows by introducing the safe budget time set  $\Qs(\tau_{c_{g-1}}:t_{c_g} - 1)$, and $(c)$ follows by taking into account the safe budget time after the changepoint $t_{c_g}$ has occurred.

\textbf{Step 2 (Bounding UCB pulls of sub-optimal arm $i$ in part A): } In this step we bound the total number of samples of sub-optimal arm pulled by using the maximum UCB index $u_t$.
\begin{align*}
\E\left[\sum_{s\in\Qs(\tau_{c_{g-1}}:t_{c_{g}}-1)}N_{i}(s)\bigg |\xi^{del}_g(s)\right] &= \E\left[\sum_{s\in\Qs(\tau_{c_{g-1}}:t_{c_{g}}-1)} \I{I_s = i, N_i(s) \leq \max\{N^{opt}_{i,g-1}, N^{chg}_{i,g-1}\}, \tau_g^{chg}\in [t_{c_g} + 1,t_{c_g} + d_g]}\right]\\
%%%%%%%%%%%%%%%%
&\qquad+ \E\left[\sum_{s\in\Qs(\tau_{c_{g-1}}:t_{c_{g}}-1)} \I{I_s = i, N_i(s) > \max\{N^{opt}_{i,g-1}, N^{chg}_{i,g-1}\}, \tau_g^{chg}\in [t_{c_g} + 1,t_{c_g} + d_g]}\right]\\
%%%%%%%%%%%%%%%%
&\overset{(a)}{\leq} \max\{N^{opt}_{i,g-1}, N^{chg}_{i,g-1}\} + d_g + \Pb\left(\overline{\xi^{opt}_{i,g-1}(s)}\bigcap \overline{\xi^{del}_{g}(s)}\right)\\
%%%%%%%%%%%%%%%
&\overset{(b)}{\leq} \max\{N^{opt}_{i,g-1}, N^{chg}_{i,g-1}\} + d_g + \Pb\left(\overline{\xi^{opt}_{i,g-1}(s)}\right).
\end{align*}
where, in $(a)$ the $N^{opt}_{i,g-1}$ is the maximum number of samples before a sub-optimal arm $i$ is discarded in favor of the optimal arm and $N^{chg}_{i,g-1}$ is maximum number of samples before the changepoint is detected due to arm $i$ (with high probability). Next the term $d_g$ is the maximum delay for detecting the changepoint $t_{c_g}$ due to some other arm $i$. Note that we have assumed in \Cref{assm:sgr} that each changepoint $g\in [G_T]$ are separated by $2\max\{d_g, d_{g+1}\}$. The inequality in $(b)$ follows by dropping the event $\overline{\xi^{del}_{g}(s)}$.

\textbf{Step 3 (Bounding Baseline pulls in part B): } In this step we bound the pulls of the baseline arm when the safety budget $\wZ(1:t) < 0$. The breakdown of the total pulls follows the same way as in the previous step.
\begin{align*}
\E&\left[\sum_{s\in\overline{\Qs}(\tau_{c_{g-1}}:t_{c_{g}}-1)}N_{0}(s)|\xi^{del}_{g}(s)\right] = \E\left[\sum_{s\in\overline{\Qs}(\tau_{c_{g-1}}:t_{c_{g}}-1)} \I{I_s = 0, N_0(s) \leq \max\{N^{bse}_{0,g-1}, N^{chg}_{0,g-1}\}, \tau_g^{chg}\in [t_{c_g} + 1,t_{c_g} + d_g]}\right]\\
%%%%%%%%%%%%%%%%
&\qquad+ \E\left[\sum_{s\in\overline{\Qs}(\tau_{c_{g-1}}:t_{c_{g}}-1)} \I{I_s = 0, N_0(s) > \max\{N^{bse}_{0,g-1}, N^{chg}_{0,g-1}\}, \tau_g^{chg}\in [t_{c_g} + 1,t_{c_g} + d_g]}\right]\\
%%%%%%%%%%%%%%%%%%%%%%%
&\overset{(a)}{\leq} \E\left[\sum_{s\in\overline{\Qs}(\tau_{c_{g-1}}:t_{c_{g}}-1)} \I{I_s = 0, N_0(s) \leq \max\{N^{bse}_{0,g-1}, N^{opt}_{0,g-1} N^{chg}_{0,g-1}\}, \tau_g^{chg}\in [t_{c_g} + 1,t_{c_g} + d_g]}\right]\\
%%%%%%%%%%%%%%%%
&\qquad+ \E\left[\sum_{s\in\overline{\Qs}(\tau_{c_{g-1}}:t_{c_{g}}-1)} \I{I_s = 0, N_0(s) > \max\{N^{bse}_{0,g-1}, N^{opt}_{0,g-1}, N^{chg}_{0,g-1}\}, \tau_g^{chg}\in [t_{c_g} + 1,t_{c_g} + d_g]}\right]\\
%%%%%%%%%%%%%%%%
&\overset{(b)}{\leq} \max\{N^{opt}_{0,g-1}, N^{chg}_{0,g-1}, N^{chg}_{0,g-1}\} + d_g + \Pb\left(\overline{\xi^{opt}_{0,g-1}(s)}\bigcap\overline{\xi^{del}_{g}(s)}\right)\\
%%%%%%%%%%%%%%%%%
&\overset{(c)}{\leq} \max\{N^{opt}_{0,g-1}, N^{chg}_{0,g-1}, N^{bse}_{0,g-1}\} + d_g + \Pb\left(\overline{\xi^{opt}_{0,g-1}(s)}\right)
% %%%%%%%%%%%%%%%%
 \leq \max\{N^{opt}_{0,g-1}, N^{chg}_{0,g-1}\} +  N^{bse}_{0,g-1} + d_g + \Pb\left(\overline{\xi^{opt}_{0,g-1}(s)}\right)
\end{align*}
where, in $(a)$ we introduce the optimality pulls of the baseline arm, $(b)$ follows as when the good event $\xi^{opt}_{0,g-1}(s)$ holds then the baseline arm cannot be sampled more than $\max\{N^{opt}_{0,g-1}, N^{bse}_{0,g-1}\}$, and $(c)$ follows by dropping one event from the intersection.

\textbf{Step 4 (Bounding Part (C)): } Note that under \Cref{assm:sgr} the changepoints are separated enough such that the detection delay for the $g$-th changepoint is defined as
\begin{align*}
    \hspace{-1em} d_{g} \!\coloneqq\! \left\lceil\! K \!+\! 4\left( K\max\limits_{i\in[K]}\frac{B(T, \delta)}{(\Delta^{chg}_{i,g})^{2}} \!+\!  \frac{B(T, \delta)}{(\Delta^{chg}_{0,g})^{2}} +N^{bse}_{0,g}\right) \!\right\rceil\!\!. 
\end{align*}
where $B(T,\delta) = 16\log(4\log_2(T/\delta))$ 
and two consecutive global changepoints are separated by
% \begin{align*}
    $t_{c_{g+1}} - t_{c_{g}} \geq 2\max\{d_{g}, d_{g+1}\}$.
% \end{align*}
% We further define the quantity $d'_g$ as follows:
% \begin{align*}
%     \hspace{-1em} d'_{g} \!\coloneqq\! \left\lceil\! \left( \max\limits_{i\in[K]}\frac{\beta_i(\!1\!:\!T,\! \delta)}{(\Delta^{chg}_{i,g})^{2}} \!+\!  \max\limits_{i\in[K]}\frac{\beta_i(\!1\!:\!T,\! \delta)}{(\Delta^{thd}_{i,g})^{2}} + \dfrac{\beta_0(\!1\!:\!T,\! \delta)}{\Delta^{bse}_{0,g}}\right) \!\right\rceil\!\!.
% \end{align*}
Then we can bound the total pulls of a sub-optimal arm $i$ in \textbf{Part (C)} as follows:
\begin{align*}
    \sum_{s\in\Qs(t_{c_g}:\tau_{c_g}-1)}&\E\left[N_i(s)|\xi^{del}_g\right]\Pb\left(\xi_g^{del}\right)
    %%%%%%%%%%%%%%%%%%%%%%%
    \leq \E\left[\sum_{s\in\Qs(t_{c_g}:\tau_{c_g}-1)} \I{I_s = i, N_i(s) < N^{chg}_{i, g}, \tau_g^{chg}\in [t_{c_g} + 1,t_{c_g} + d_g]}\right]\\
    %%%%%%%%%%%%%%%%%%%%%%
    &\qquad + \E\left[\sum_{s\in\Qs(t_{c_g}:\tau_{c_g}-1)}\I{I_s = i, N_i(s) \geq  N^{chg}_{i, g}, \tau_g^{chg}\in [t_{c_g} + 1,t_{c_g} + d_g]}\right] \\
    %%%%%%%%%%%%%%%%%%
    &= \E\left[\sum_{s\in\Qs(t_{c_g}:\tau_{c_g}-1)} \I{I_s = i, N_i(s) < N^{chg}_{i, g}, \tau_{c_{g}} = \tau^i_{c_{g}}, \tau_g^{chg}\in [t_{c_g} + 1,t_{c_g} + d_g]  }\right]\\
    %%%%%%%%%%%%%%%%%%%%%%%
    &\qquad + \E\left[\sum_{s\in\Qs(t_{c_g}:\tau_{c_g}-1)} \I{I_s = i, N_i(s) < N^{chg}_{i, g}, \tau_{c_{g}} > \tau^i_{c_{g}}, \tau_g^{chg}\in [t_{c_g} + 1,t_{c_g} + d_g]  }\right]\\
    %%%%%%%%%%%%%%%%%%%
    &\qquad + \E\left[\sum_{s\in\Qs(t_{c_g}:\tau_{c_g}-1)}\I{I_s = i, N_i(s) \geq  N^{chg}_{i, g}, \tau_{c_{g}} = \tau^i_{c_{g}}, \tau_g^{chg}\in [t_{c_g} + 1,t_{c_g} + d_g]}\right]\\
\end{align*}
\begin{align*}
    %%%%%%%%%%%%%%%%%%%
    &\qquad+ \sum_{i=1}^K\E\left[\sum_{s\in\Qs(t_{c_g}:\tau_{c_g}-1)}\I{I_s = i, N_i(s) \geq  N^{chg}_{i, g}, \tau_{c_{g}} > \tau^i_{c_{g}}, \tau_g^{chg}\in [t_{c_g} + 1,t_{c_g} + d_g]}\right]  \\
    %%%%%%%%%%%%%%%%%%%%
    &\overset{(a)}{\leq} N^{chg}_{i, g} + \max_{j\in[K]}N^{chg}_{j, g} + \sum_{s\in\Qs(t_{c_g}:\tau_{c_g}-1)}\Pb\left(\overline{\xi^{chg}_{i, g}(s)}\right) + \sum_{j=1}^K\sum_{s\in\Qs(t_{c_g}:\tau_{c_g}-1)}\Pb\left(\overline{\xi^{chg}_{j, g}(s)}\right)  + 4d_g + \sum_{s\in\overline{\Qs}(t_{c_g}:\tau_{c_g}-1)}\Pb(\overline{\xi^{del}_{g}(s)})\\
    %%%%%%%%%%%%%%%%%%%%
% \end{align*}
% \begin{align*}
    &\leq N^{chg}_{i, g} + \max_{j\in[K]}N^{chg}_{j, g} + \sum_{s=1}^{T}\Pb\left(\overline{\xi^{chg}_{i, g}(s)}\right) + \sum_{j=1}^K\sum_{s=1}^{T}\Pb\left(\overline{\xi^{chg}_{j, g}(s)}\right)  + 4d_g + \sum_{s=1}^T\Pb(\overline{\xi^{del}_{g}(s)})
    %%%%%%%%%%%%%%%%%%%%
\end{align*}
where in $(a)$ the $\max_{j\in[K]}N^{chg}_{j, g}$ is the maximum number of samples required to detect the $g$-th changepoint due to some other arm than $i$.

\textbf{Step 5 (Bounding Part (D)): } Again, note that under \Cref{assm:sgr} the changepoints are separated enough such that the detection delay for the $g$-th changepoint is defined as
\begin{align*}
    \hspace{-1em} d_{g} \!\coloneqq\! \left\lceil\! K \!+\! 4\left( K\max\limits_{i\in[K]}\frac{B(T, \delta)}{(\Delta^{chg}_{i,g})^{2}} \!+\!  \frac{B(T \delta)}{(\Delta^{chg}_{0,g})^{2}} +N^{bse}_{0,g}\right) \!\right\rceil\!\!. 
\end{align*}
where, $B(T,\delta) = 16\log(4\log_2(T/\delta))$ and two consecutive global changepoints are separated by
\begin{align*}
    t_{c_{g+1}} - t_{c_{g}} \geq 2\max\{d_{g}, d_{g+1}\}.
\end{align*}
Then we can bound the total pulls of a sub-optimal arm $i$ in \textbf{Part (D)} as follows:
\begin{align*}
    &\sum_{s\in\overline{\Qs}(t_{c_g}:\tau_{c_g}-1)}\E\left[N_0(s)|\xi^{del}_g\right]\Pb\left(\xi_g^{del}\right)
    %%%%%%%%%%%%%%%%%%%%%%%
    \leq \E\left[\sum_{s\in\overline{\Qs}(t_{c_g}:\tau_{c_g}-1)} \I{I_s = 0, N_0(s) < N^{chg}_{0, g}, \tau_g^{chg}\in [t_{c_g} + 1,t_{c_g} + d_g]}\right]\\
    %%%%%%%%%%%%%%%%%%%%%%
    &\qquad + \E\left[\sum_{s\in\overline{\Qs}(t_{c_g}:\tau_{c_g}-1)}\I{I_s = 0, N_0(s) \geq  N^{chg}_{0, g}, \tau_g^{chg}\in [t_{c_g} + 1,t_{c_g} + d_g]}\right] \\
    %%%%%%%%%%%%%%%%%%%
    &= \E\left[\sum_{s\in\overline{\Qs}(t_{c_g}:\tau_{c_g}-1)} \I{I_s = 0, N_0(s) < N^{chg}_{0, g}, \tau_{c_{g}} = \tau^0_{c_{g}}, \tau_g^{chg}\in [t_{c_g} + 1,t_{c_g} + d_g]  }\right]\\
    %%%%%%%%%%%%%%%%%%%%%%%
    &\qquad + \E\left[\sum_{s\in\overline{\Qs}(t_{c_g}:\tau_{c_g}-1)} \I{I_s = 0, N_0(s) < N^{chg}_{0, g}, \tau_{c_{g}} > \tau^0_{c_{g}}, \tau_g^{chg}\in [t_{c_g} + 1,t_{c_g} + d_g]  }\right]\\
    %%%%%%%%%%%%%%%%%%%
    &\qquad + \E\left[\sum_{s\in\overline{\Qs}(t_{c_g}:\tau_{c_g}-1)}\I{I_s = 0, N_0(s) \geq  N^{chg}_{0, g}, \tau_{c_{g}} = \tau^0_{c_{g}}, \tau_g^{chg}\in [t_{c_g} + 1,t_{c_g} + d_g]}\right]\\
    %%%%%%%%%%%%%%%%%%%
    &\qquad+ \sum_{i=1}^K\E\left[\sum_{s\in\overline{\Qs}(t_{c_g}:\tau_{c_g}-1)}\I{I_s = 0, N_0(s) \geq  N^{chg}_{i, g}, \tau_{c_{g}} > \tau^0_{c_{g}}, \tau_g^{chg}\in [t_{c_g} + 1,t_{c_g} + d_g]}\right]  \\
    %%%%%%%%%%%%%%%%%%%%
    &\leq N^{chg}_{0, g} + \max_{j\in[K]}N^{chg}_{j, g} + \sum_{s\in\overline{\Qs}(t_{c_g}:\tau_{c_g}-1)}\Pb\left(\overline{\xi^{chg}_{0, g}(s)}\right) + \sum_{j=1}^K\sum_{s\in\overline{\Qs}(t_{c_g}:\tau_{c_g}-1)}\Pb\left(\overline{\xi^{chg}_{j, g}(s)}\right)  + 4d_g + \sum_{s\in\overline{\Qs}(t_{c_g}:\tau_{c_g}-1)}\Pb(\overline{\xi^{del}_{g}(s)})\\
    %%%%%%%%%%%%%%%%%%%%
    &\leq N^{chg}_{0, g} + \max_{j\in[K]}N^{chg}_{j, g} + \sum_{s=1}^{T}\Pb\left(\overline{\xi^{chg}_{0, g}(s)}\right) + \sum_{j=1}^K\sum_{s=1}^{T}\Pb\left(\overline{\xi^{chg}_{j, g}(s)}\right)  + 4d_g + \sum_{s=1}^T\Pb(\overline{\xi^{del}_{g}(s)})
    %%%%%%%%%%%%%%%%%%%%
\end{align*}

% \textbf{Step 6 \textbf{(Bounding Part (E))}:} Then we can bound the total pulls of a sub-optimal arm $i$ in \textbf{Part (D)} as follows:
% \begin{align*}
%     &\sum_{s=t_{c_g}}^{\tau_{c_{g}}-1}\E\left[N_i(s)\right]\E[\I{\xi_g^{del}}]
% \end{align*}

\textbf{Step 6 (Controlling the total delay, Part E): } We first define the total detection delay good event as follows:
\begin{align*}
    \xi^{del} \coloneqq\left\{\forall i \in\{1, \ldots, K\}, \forall g \in\left\{0, 1,2, \ldots, G_T\right\}, \tau^{chg}_{c_g} \in\left[t_{c_g} + 1, t_{c_g} + d_g\right]\right\}
\end{align*}
such that the learner detects all the changepoint $g\leq G_T$ with a detection delay of at most $d_g$ where
\begin{align}
    \hspace{-1em} d_{g} \!\coloneqq\! \left\lceil\! K \!+\! 4\left( K\max\limits_{i\in[K]}\frac{\beta_i(\!1\!:\!T,\! \delta)}{(\Delta^{chg}_{i,g})^{2}} \!+\!  \frac{\beta_0(\!1\!:\!T,\! \delta)}{(\Delta^{chg}_{0,g})^{2}} +N^{bse}_{0,g}\right) \!\right\rceil\!\!.
\end{align}
We define a slightly stronger event than $\xi^{del}$ as follows:
\begin{align*}
    \xi^{del}_{g} \coloneqq\left\{\forall g' \leq g , \tau^{chg}_{c_{g'}} \in\left[t_{c_{g'}} + 1, t_{c_{g'}} + d_{g'}\right]\right\}
\end{align*}
which signifies that all changepoints before $g$ has been successfully detected. It follows that $\xi^{del} \subseteq \xi^{del}_{g}$ and $\xi^{del}_{g}$ is $\F_{\tau^{chg}_{c_g}} - 1$-measurable. Then we can show that
\begin{align*}
     \Pb(\overline{\xi^{del}_{g}(s)}) &\leq \sum_{g': g' \leq g}\E\left[\I{\overline{\xi^{del}_{g'}(s)}} \sum_{s=\tau^{}_{c_{g'}}}^{t_{c_{g'+1}}} \I{\left(\cup_{i\in[K] }\overline{\xi^{chg}_{i,g'}}(\tau_{c_{g'}}:s)\right)} \mid \F_{\tau^{}_{c_g}}\right] \\
    %%%%%%%%%%%%%%%%%%%%%
    &\overset{(a)}{\leq} \sum_{g': g' \leq g}\I{\overline{\xi^{del}_{g'}(s)}} \E\left[\sum_{s=\tau^{}_{c_{g'}}}^{t_{c_{g'+1}}} \I{\left(\cup_{i\in[K] }\overline{\xi^{chg}_{i,g'}}(\tau_{c_{g'}}:s)\right)} \mid \F_{\tau^{}_{c_g}}\right]\\
    %%%%%%%%%%%%%%%%%%
    &\leq \sum_{g': g' \leq g}\I{\overline{\xi^{del}_{g'}(s)}} \left[\sum_{s=\tau^{}_{c_{g'}}}^{t_{c_{g'+1}}} \Pb{\left(\cup_{i\in[K] }\overline{\xi^{chg}_{i,g'}}(\tau_{c_{g'}}:s)\right)} \right]
%%%%%%%%%%%%%%%%%%
\overset{(b)}{\leq} G_T K\sum_{t=1}^{T} \frac{1}{t \ln (t)} \leq G_T K\log (\log (T))
\end{align*}
where, $(a)$ follows from the definition of $\xi^{chg}_{i,g}(\tau_{c_{g}}:t_{c_g}:s)$ in \Cref{lemma:conc-chg}, and $(b)$ follows from \Cref{lemma:conc-chg} and setting $\delta = 1/t$.

\textbf{Step 7 (Combining everything): } Combining everything and plugging it in \eqref{eq:g-regret-decomp-glb-chg} we can show that
\begin{align*}
    \E[R_T] & \leq \sum_{g=1}^{G_T}\bigg[\underbrace{\sum_{i=1}^K\max\{N^{opt}_{i,g-1}, N^{chg}_{i,g-1}\}\Delta^{opt}_{i, g-1} + d_g\Delta^{opt}_{i, g-1} + \sum_{s=\tau_{c_{g-1}} }^{t_{c_g}-1}\Pb\left(\overline{\xi^{opt}_{i,g-1}(s)}\right)\Delta^{opt}_{i, g-1}}_{\textbf{Part (A)}} \\
    %%%%%%%%%%%%%%%%%
    &\qquad + \underbrace{\max\{N^{opt}_{0,g-1}, N^{chg}_{0,g-1}\}\Delta^{opt}_{0, g-1} + N^{bse}_{0,g-1}\Delta^{opt}_{0, g-1} + d_g\Delta^{opt}_{0, g-1} +  \sum_{s=\tau_{c_{g-1}} }^{t_{c_g}-1}\Pb\left(\overline{\xi^{opt}_{0,g-1}(s)}\right)\Delta^{opt}_{0, g-1}}_{\textbf{Part (B)}} \\
    %%%%%%%%%%%%%%%%%
    &\qquad +  \underbrace{\sum_{i=1}^K\bigg[N^{chg}_{i, g} + \max_{j\in[K]}N^{chg}_{j, g} + 4d_g +  \sum_{s=t_{c_g} }^{\tau_{c_{g}}-1}\Pb\left(\xi^{chg}_{i, g}(s)\right) + \sum_{j=1}^K\sum_{s=t_{c_g} }^{\tau_{c_{g}}-1}\Pb\left(\xi^{chg}_{j, g}(s)\right)\bigg]\Delta^{opt}_{i, g}}_{\textbf{Part (C)}}\\
    %%%%%%%%%%%%%%%%%
    &\qquad + \underbrace{\bigg[N^{chg}_{0, g} + \max_{j\in[K]}N^{chg}_{j, g} + 4d_g +  \sum_{s=t_{c_g}}^{\tau_{c_{g}}-1}\Pb\left(\xi^{chg}_{0, g}\right) + \sum_{j=1}^K\sum_{s=t_{c_g} }^{\tau_{c_{g}}-1}\Pb\left(\xi^{chg}_{j, g}(s)\right)\bigg]\Delta^{opt}_{0, g}}_{\textbf{Part (D)}}\\
    %%%%%%%%%%%%%%%%%%%
    &\qquad +\underbrace{\sum_{i=1}^K\left(d_g\Delta^{opt}_{i, g} + \sum_{s=t_{c_g} }^{\tau_{c_{g}}-1}\Pb(\overline{\xi^{del}_{g}(s)})\Delta^{opt}_{i, g}\right)}_{\textbf{Part (E)}}\bigg] %\\
\end{align*}
\begin{align*}
    %%%%%%%%%%%%%%%%%
    %%%%%%%%%%%%%%%%%%
    %%%%%%%%%%%%%%%%%
    &\overset{(a)}{\leq} \sum_{g=1}^{G_T}\bigg[\sum_{i=1}^K\max\left\{\dfrac{1}{\Delta^{opt}_{i,g-1}}, \dfrac{\Delta^{opt}_{i, g-1}}{\left(\Delta^{chg}_{i,g-1}\right)^2}\right\}8\log\left(\dfrac{4\log_2(T+1)}{\delta}\right) + d_g K \Delta^{opt}_{\max,g-1} + K\Delta^{opt}_{\max, g-1}\sum_{t=1}^T\delta \\
    %%%%%%%%%%%%%%%%%%%
    &\qquad + \max\left\{\dfrac{1}{\Delta^{opt}_{0,g-1}}, \dfrac{\Delta^{opt}_{0, g-1}}{\left(\Delta^{chg}_{0,g-1}\right)^2}\right\}8\log\left(\dfrac{4\log_2(T+1)}{\delta}\right) + N^{bse}_{0,g-1}\Delta^{opt}_{\max, g-1} + d_g\Delta^{opt}_{0,g-1} + K\Delta^{opt}_{\max, g-1}\sum_{t=1}^T\delta\\
    %%%%%%%%%%%%%%%%%%
    &\qquad + \left(\sum_{i=1}^K \left(\dfrac{\Delta^{opt}_{i, g}}{\left(\Delta^{chg}_{i,g}\right)^2} + \max_{j\in[K]}\dfrac{\Delta^{opt}_{i, g}}{\left(\Delta^{chg}_{j,g}\right)^2}\right)\right)8\log(\dfrac{4\log_2(T+1)}{\delta}) +  \left(\dfrac{\Delta^{opt}_{0, g}}{\left(\Delta^{chg}_{0,g}\right)^2} + K\max_{j\in[K]}\dfrac{\Delta^{opt}_{j, g}}{\left(\Delta^{chg}_{j,g}\right)^2}\right)8\log(\dfrac{4\log_2(t+1)}{\delta})\\
    %%%%%%%%%%%%%%%%%%%
    & \qquad + 8d_g\Delta^{opt}_{\max, g} + 2\left(K\sum_{t=1}^T\delta + K^2\sum_{t=1}^T\delta  + d_g + 6 G_T K\log (\log (T))\right)\Delta^{opt}_{\max, g}\bigg]\\
    %%%%%%%%%%%%%%%%%%%%%%%%%
    %%%%%%%%%%%%%%%%%%%%%%
    %%%%%%%%%%%%%%%%%%%%%%%%%%
% \end{align*}
% \begin{align*}
    & \leq O\bigg(\sum_{g=1}^{G_T}\left(\sum_{i=1}^{K^+}\max\left\{\dfrac{1}{\Delta^{opt}_{i,g-1}}, \dfrac{\Delta^{opt}_{i,g-1}}{\left(\Delta^{chg}_{i,g-1}\right)^2}\right\} 
    + \sum_{i=1}^{K^+}\max_{j\in[K]^+}\dfrac{\Delta^{opt}_{i, g}}{\left(\Delta^{chg}_{j,g}\right)^2} \right)\log\left(\dfrac{\log_2T}{\delta}\right)\\
    %%%%%%%%%%%%%%%%%%%%%%%%%
    &\qquad + N^{bse}_{0,g-1}\Delta^{opt}_{\max, g-1} + K\sum_{g=1}^{G_T}\Delta^{opt}_{\max, g}d_g + KG_T\Delta^{opt}_{\max, g}\sum_{t=1}^T\delta\bigg)\\
    %%%%%%%%%%%%%%%%%%%%%%%%%
    %%%%%%%%%%%%%%%%
    & \overset{(b)}{=} O\bigg(\sum_{g=1}^{G_T}\sum_{i=1}^{K^+}\left(H^{(1)}_{i,g-1} 
    + H^{(2)}_{i,g} \right)\log\left(\dfrac{\log_2T}{\delta}\right)
    %%%%%%%%%%%%%%%%%%%%%%%%%
    + \sum_{g=1}^{G_T}\Delta^{opt}_{\max, g-1}\dfrac{1}{\alpha\mu_{0,g-1}}\sum_{i\in[K]}\dfrac{\log(\log_2T/\delta)}{\max\{\Delta^{opt}_{i,g-1}, \Delta^{opt}_{0,g-1} - \Delta^{opt}_{i,g-1}\}}\\
    %%%%%%%%%%%%%%%%%%%
    &\qquad+ \sum_{g=1}^{G_T}\Delta^{opt}_{\max, g}\dfrac{1}{\alpha\mu_{0,g}}\sum_{i\in[K]}\dfrac{\log(\log_2T/\delta)}{\max\{\Delta^{opt}_{i,g}, \Delta^{opt}_{0,g} - \Delta^{opt}_{i,g}\}} + KG_T\Delta^{opt}_{\max, g}\sum_{t=1}^T\delta\bigg)\\
    %%%%%%%%%%%%%%%%%%%%%%%%%
    %%%%%%%%%%%%%%%%%%%%%
    % & \overset{(c)}{=} O\bigg(\sum_{g=1}^{G_T}\sum_{i=1}^{K^+}\left(H^{(1)}_{i,g-1} 
    % + H^{(2)}_{i,g} \right)\log\left(\dfrac{\log_2T}{\delta}\right)
    % %%%%%%%%%%%%%%%%%%%%%%%%%
    % + K\sum_{g=1}^{G_T} H_{3,g} \log\left(\dfrac{\log_2( t + 1)}{\delta}\right)
    % + KG_T\Delta^{opt}_{\max, g}\sum_{t=1}^T\delta\bigg)\\
    %%%%%%%%%%%%%%%%%%%%%%%%%%%
    %%%%%%%%%%%%%%%%%%%%%%%%%%
    & \overset{(c)}{=} O\bigg(\sum_{g=1}^{G_T}\sum_{i=1}^{K^+}\left(H^{(1)}_{i,g-1} 
    + H^{(2)}_{i,g} \right)\log\left(\dfrac{\log_2T}{\delta}\right)
    %%%%%%%%%%%%%%%%%%%%%%%%%
    + K\sum_{g=1}^{G_T}\dfrac{1}{\alpha\mu_{0,g-1}}\sum_{i=1}^K H^{(3)}_{i,g-1} \log\left(\dfrac{\log_2T}{\delta}\right) + K\sum_{g=1}^{G_T}\dfrac{1}{\alpha\mu_{0,g}}\sum_{i=1}^K H^{(3)}_{i,g} \log\left(\dfrac{\log_2T}{\delta}\right)
    %+ KG_T\Delta^{opt}_{\max, g}\sum_{t=1}^T\delta
    \bigg)\\
    %%%%%%%%%%%%%%%%%%%%%%%%%%
    %%%%%%%%%%%%%%%%%%%%%%%%%
    & = O\bigg(\left(\sum_{g=1}^{G_T}\sum_{i=1}^{K^+}\left(H^{(1)}_{i,g-1} 
    + H^{(2)}_{i,g} \right) + K\sum_{g=1}^{G_T}\dfrac{1}{\alpha\mu_{0,g-1}}\sum_{i=1}^K H^{(3)}_{i,g-1} + K\sum_{g=1}^{G_T}\dfrac{1}{\alpha\mu_{0,g}}\sum_{i=1}^K H^{(3)}_{i,g}\right)\log\left(\dfrac{\log_2T}{\delta}\right)
    %%%%%%%%%%%%%%%%%%%%%%%%%
    % + KG_T\Delta^{opt}_{\max, g}\sum_{t=1}^T\delta
    \bigg)
\end{align*}
where, $(a)$ follows from combining all the steps above, $(b)$ follows by using the definition of 
\begin{align*}
    H^{(1)}_{i,g-1} \coloneqq \max\left\{\dfrac{1}{\Delta^{opt}_{i,g-1}}, \dfrac{\Delta^{opt}_{i,g-1}}{\left(\Delta^{chg}_{i,g-1}\right)^2}\right\}, \qquad
    %%%%%%%%%%%%%%%%%%%%
    H^{(2)}_{i,g} \coloneqq \max_{j\in[K]^+}\dfrac{\Delta^{opt}_{i, g}}{\left(\Delta^{chg}_{j,g}\right)^2},
\end{align*}
and substituting the value of $d_g$. Finally, $(c)$ follows by using the definition of 
\begin{align*}
   H^{(3)}_{i,g} \coloneqq \dfrac{\Delta^{opt}_{\max, g}}{\max\{\Delta^{opt}_{i,g}, \Delta^{opt}_{0,g} - \Delta^{opt}_{i,g}\}}.
\end{align*}
and dropping the low probability event term as it is dominated by other terms. The claim of the theorem follows.
% bounding the total detection delay term Part (F), and summing over each the changepoint set of arms for each segment $\rho_g$. 
%The quantity $D^{bse}_{0, g-1}$, $D^{thd}_{j, g-1}$ is as defined in \Cref{thm:no-change} but restricted to $\rho_{g-1}$. The claim of the theorem follows.
\end{proof}

\subsection{Proof of Regret bound for Safety Aware Local Restart}
\label{app:regret-loc-change}

\begin{customtheorem}{3}\textbf{(Restatement)} 
Let $H^{(1)}_{i,g}, \overline{H^{(2)}_{i,g}}, H^{(3)}_{i,g}$ is defined above for the segment $\rho_{g}$. 
Then the expected regret of \loc is bounded by
\begin{align*}
    \E[R_T] &\!\leq\! O\bigg(\left(\sum_{i=1}^{K^+}\sum_{g=1}^{G^i_T}\left(H^{(1)}_{i,g-1} 
    + \overline{H^{(2)}_{i,g}} \right) + \sum_{i=1}^K\sum_{g=1}^{G^i_T}\dfrac{1}{\alpha\mu_{0,g-1}}\sum_{i=1}^K H^{(3)}_{i,g-1} + \sum_{i=1}^K\sum_{g=1}^{G^i_T}\dfrac{1}{\alpha\mu_{0,g}}\sum_{i=1}^K H^{(3)}_{i,g}\right)\log\left(\dfrac{\log_2T}{\delta}\right)
    %%%%%%%%%%%%%%%%%%%%%%%%%
    % + KG_T\Delta^{opt}_{\max}\sum_{t=1}^T\delta
    \bigg) + \gamma T
\end{align*}
\end{customtheorem}

\begin{proof}
\textbf{Step 1 (Regret Decomposition):} Recall that $G^i_T$ is the total number of changepoints for the arm $i$. 
We define the expected regret till round $T$ as follows:
\begin{align}
    &\E[R_T] = \sum_{s=1}^T\left(\mu_{i^*}(s) - \E[X_{I_s}(s)]\right) \nonumber\\
    %%%%%%%%%%%%%%%%%%%%%%%
    % &\leq \sum_{s=1}^T\left(\mu_{i^*}(s) - \E[X_{I_s}(s)|\xi^{del}_{i,g}]\Pb(\xi^{del}_{i,g}(s))\right) + \sum_{s=1}^T\left(\mu_{i^*}(s) - \E[X_{I_s}(s)|\overline{\xi^{del}_{i,g}}(s)]\Pb(\overline{\xi^{del}_{i,g}(s)})\right) + \gamma T \nonumber\\
    %%%%%%%%%%%%%%%%%%%%%%%
    &\leq \sum_{g=1}^{G^i_T}\left[\sum_{s=\tau_{c^i_{g-1}}}^{t^i_{c_g}-1}\left(\mu_{i^*}(s) - \E[\mu_{I_s}(s)|\xi_{i,g}^{del}(s)]\right)\Pb(\xi_{i,g}^{del}(s)) + \sum_{s=t^i_{c_g} }^{\tau^i_{c_{g}}-1}\left(\mu_{i^*}(s) - \E[\mu_{I_s}(s)|\xi_{i,g}^{del}(s)]\right)\Pb(\xi_{i,g}^{del}(s)) + \Pb(\overline{\xi_{i,g}^{del}(s)}) \right] + \gamma T \nonumber\\
% \end{align}
% \begin{align}
    %%%%%%%%%%%%%%%%%%%%%
    &\overset{(a)}{=} \sum_{g=1}^{G^i_T}\bigg[\sum_{s\in\Qs(\tau^i_{c_{g-1}}:t^i_{c_{g}}-1)} \left(\mu_{i^*}(s) - \E[\mu_{I_s}(s)|\xi_{i,g}^{del}(s)]\right)\Pb\left(\xi_{i,g}^{del}(s)\right) + \sum_{s\in\overline{\Qs}(\tau^i_{c_{g-1}}:t^i_{c_{g}}-1)} \left(\mu_{i^*}(s) - \E[\mu_{I_s}(s)|\xi_{i,g}^{del}(s)]\right)\Pb\left(\xi_{i,g}^{del}(s)\right) \nonumber\\
    %%%%%%%%%%%%%%%%%%%%%%%
    &\qquad + \sum_{s=t^k_{c_g}}^{\tau_{c_{g}} - 1}\left(\mu_{i^*}(s) - \E[\mu_{I_s}(s)|\xi_g^{del}(s)]\right)\Pb\left(\xi_{i,g}^{del}(s)\right) + \Pb(\overline{\xi_{i,g}^{del}(s)})\bigg] + \gamma T\nonumber\\
    %%%%%%%%%%%%%%%%%%%%%%%%%
    &\overset{(b)}{=} \underbrace{\sum_{i=1}^K\sum_{g=1}^{G^i_T}\sum_{\substack{s\in\Qs(\tau^i_{c_{g-1}}:t^i_{c_{g}}-1)}} \Delta^{opt}_i(s)\E[N_i(s)|\xi_{i,g}^{del}(s)]\Pb\left(\xi_{i,g}^{del}(s)\right)}_{\textbf{Part (A), UCB arm pulled, Safe budget $\wZ(\tau_{c_{g-1}}:s) \geq 0$}} 
    + \underbrace{\sum_{g=1}^{G^0_T}\sum_{\substack{s\in\overline{\Qs}(\tau^0_{c_{g-1}}:t^0_{c_{g}}-1)}} \Delta^{opt}_0(s)\E[N_0(s)|\xi_{0,g}^{del}(s)]\Pb\left(\xi_{i,g}^{del}(s)\right)}_{\textbf{Part (B), Baseline arm pulled, Safe budget $\wZ(\tau_{c_{g-1}}:s) < 0$}} \nonumber\\
    %%%%%%%%%%%%%%%%%%%%%%%%
    & + \underbrace{\sum_{i=1}^K\sum_{g=1}^{G^i_T}\sum_{s\in\Qs(t_{c_g}:\tau^i_{c_g}-1)}^{}\Delta^{opt}_i(s) \E[N_i(s)|\xi_{i,g}^{del}]\Pb\left(\xi_{i,g}^{del}(s)\right)}_{\textbf{Part (C), Changepoint Pulls, Safe budget $\wZ(\tau_{c_{g-1}}:s) \geq 0$}} + \underbrace{\sum_{g=1}^{G^0_T}\sum_{s\in\overline{\Qs}(t_{c_g}:\tau^0_{c_g}-1)}^{}\Delta^{opt}_0(s) \E[N_0(s)|\xi_{0,g}^{del}]\Pb\left(\xi_{0,g}^{del}(s)\right)}_{\textbf{Part (D), Changepoint Baseline Pulls, Safe budget $\wZ(\tau_{c_{g-1}}:s) < 0$}}\nonumber\\
    %%%%%%%%%%%%%%%%%%%%
    & + \underbrace{\sum_{i=1}^K\sum_{g=1}^{G^i_T}\Pb(\overline{\xi_{i,g}^{del}(s)})}_{\textbf{Part (E), Total Detection Delay Error}}
    %%%%%%%%%%%%%%%%%
    + \underbrace{\gamma T}_{\textbf{Part (F), Forced Exploration Pulls}}.
    \label{eq:g-regret-decomp-loc-chg}
\end{align}

\textbf{Step 2 (Bounding Part (A) - (D)):} Note that under \Cref{assm:lgr} the changepoints of each arm $i$ are separated enough such that the detection delay for the $g$-th changepoint is defined as
\begin{align}
\hspace{-1em} d_{i,g} \!\coloneqq\! \left\lceil\! \dfrac{K}{\gamma} \!+\! \dfrac{4}{\gamma}\left( \frac{B(T,\delta)}{(\Delta^{chg}_{i,g})^{2}} \!+\!  \frac{B(T,\delta)}{(\Delta^{chg}_{0,g})^{2}} +N^{bse}_{0,g}\right) \!\right\rceil\!\!.\label{eq:local-delay}
\end{align}
where $B(T,\delta) = 16\log(4\log_2(T/\delta))$ and $\gamma$ is the exploration rate of $\loc$ and two consecutive changepoints of $i$ are separated by
\begin{align*}
    t^i_{c_{g+1}} - t^i_{c_{g}} \geq 2\max\{d_{i,g}, d_{i,g+1}\}.
\end{align*}
% We further define the quantity $d^{',i}_{g}$ as follows:
% \begin{align*}
%     \hspace{-1em} d^{',i}_{g} \!\coloneqq\! \left\lceil\! \left( \frac{\beta_i(\!1\!:\!T,\! \delta)}{\gamma(\Delta^{chg}_{i,g})^{2}} \!+\!  \frac{\beta_i(\!1\!:\!T,\! \delta)}{\gamma(\Delta^{thd}_{i,g})^{2}} + \dfrac{\beta_0(\!1\!:\!T,\! \delta)}{\gamma\Delta^{bse}_{0,g}}\right) \!\right\rceil\!\!. 
% \end{align*}
So when the changepoint $t^i_{c_g}$ is detected and the total detection delay is controlled for the arm $i$ (shown in step 4) we can show using the \eqref{eq:g-regret-decomp-glb-chg} and steps $2$-$4$ that
\begin{align*}
    \textbf{Part (A)} & \leq  \max\{N^{opt}_{i,g-1}, N^{chg}_{i,g-1}\} + d_{i,g} + \Pb\left(\overline{\xi^{opt}_{i,g-1}(s)}\right)\\
    %%%%%%%%%%%%%%%%%%%%
    \textbf{Part (B)} &\leq  \max\{N^{bse}_{0,g-1}, N^{chg}_{0,g-1}\} + N^{bse}_{0,g-1} + d_{i,g} + \Pb\left(\overline{\xi^{opt}_{0,g-1}(s)}\right)\\
    %%%%%%%%%%%%%%%%%%%%
    \textbf{Part (C)} & \leq  N^{chg}_{i, g} + \sum_{s=1}^{T}\Pb\left(\overline{\xi^{chg}_{i, g}(s)}\right)  + 4d_{i,g} + \Pb(\overline{\xi^{del}_{i,g}(s)})\\
    %%%%%%%%%%%%%%%%%%%%
    \textbf{Part (D)} &\leq N^{chg}_{0, g} + \sum_{s=1}^{T}\Pb\left(\overline{\xi^{chg}_{0, g}(s)}\right)  + 4d_{i,g} + \Pb(\overline{\xi^{del}_{i,g}(s)})
\end{align*}

\textbf{Step 3 (Controlling the total delay):} We want to control the total detection delay of individual arms $i$. Again we define the total detection delay good event for the $i$-th arm as follows:
\begin{align*}
    \xi^{del}_i \coloneqq\left\{\forall i \in\{1, \ldots, K\}, \forall g \in\left\{0, 1,2, \ldots, G^i_T\right\}, \tau^{i}_{c_g} \in\left[t^i_{c_g} + 1, t^i_{c_g} + d_{i,g}\right]\right\}
\end{align*}
such that the learner detects all the changepoint $g\leq G^i_T$ with a detection delay of at most $d_{i,g}$ where
\begin{align}
    d_{i,g} \!\coloneqq\! \left\lceil\! \dfrac{K}{\gamma} \!+\! \dfrac{4}{\gamma}\left( \frac{B(T,\delta)}{(\Delta^{chg}_{i,g})^{2}} \!+\!  \frac{B(T,\delta)}{(\Delta^{chg}_{0,g})^{2}} +N^{bse}_{0,g}\right) \!\right\rceil.
\end{align}
where $B(T,\delta) = 16\log(4\log_2(T/\delta))$. 
We define a slightly stronger event than $\xi^{del}_i$ as follows:
\begin{align*}
    \xi^{del}_{i,g} \coloneqq\left\{\forall g' \leq g , \tau^{i}_{c_{g'}} \in\left[t_{c_{g'}} + 1, t_{c_{g'}} + d_{i,g'}\right]\right\}
\end{align*}
which signifies that all changepoints before $g$ has been successfully detected. It follows that $\xi^{del}_i \subseteq \xi^{del}_{i,g}$ and $\xi^{del}_{i,g}$ is $\F_{\tau^{i}_{c_g}} - 1$-measurable. Then we can show that
\begin{align*}
     \Pb(\overline{\xi^{del}_{i,g}}) &\leq \sum_{g': g' \leq g}\E\left[\I{\overline{\xi^{del}_{i,g'}(s)}} \sum_{s=\tau^{i}_{c_{g'}}}^{t^i_{c_{g'+1}}} \I{\left(\cup_{i\in[K] }\overline{\xi^{chg}_{i,g'}}(\tau^i_{c_{g'}}:s)\right)} \mid \F_{\tau^{i}_{c_g}}\right] \\
    %%%%%%%%%%%%%%%%%%%%%
    &\overset{(a)}{\leq} \sum_{g': g' \leq g}\I{\overline{\xi^{del}_{i,g'}(s)}} \E\left[\sum_{s=\tau^{i}_{c_{g'}}}^{t_{c_{g'+1}}} \I{\left(\cup_{i\in[K] }\overline{\xi^{chg}_{i,g'}}(\tau^{i}_{c_{g'}}:s)\right)} \mid \F_{\tau^{i}_{c_g}}\right]\\
    %%%%%%%%%%%%%%%%%%
    &\leq \sum_{g': g' \leq g}\I{\overline{\xi^{del}_{i,g'}(s)}} \left[\sum_{s=\tau^{i}_{c_{g'}}}^{t^i_{c_{g'+1}}} \Pb{\left(\cup_{i\in[K] }\overline{\xi^{chg}_{i,g'}}(\tau^{i}_{c_{g'}}:s)\right)} \right]\\
%%%%%%%%%%%%%%%%%%
&\overset{(b)}{\leq} G^i_T K\sum_{t=1}^{T} \frac{1}{t \ln (t)} \leq G^i_T K\log (\log (T))
\end{align*}
where, $(a)$ follows from the definition of $\xi^{chg}_{i,g}(\tau^{i}_{c_g}:t^i_{c_g}:s)$ in \Cref{lemma:conc-chg}, and $(b)$ follows from \Cref{lemma:conc-chg} and setting $\delta = 1/t$.

\textbf{Step 4 (Combining everything): } Combining everything and plugging it in \eqref{eq:g-regret-decomp-loc-chg} we can show that
\begin{align*}
    \E[R_T] & \leq \underbrace{\sum_{i=1}^K\sum_{g=1}^{G^i_T}\bigg[\max\{N^{opt}_{i,g-1}, N^{chg}_{i,g-1}\}\Delta^{opt}_{i, g-1} + d_{i,g}\Delta^{opt}_{i, g-1} + \sum_{s=\tau^i_{c_{g-1}} }^{t^i_{c_g}-1}\Pb\left(\overline{\xi^{opt}_{i,g-1}(s)}\right)\Delta^{opt}_{i, g-1}\bigg]}_{\textbf{Part (A)}} \\
    %%%%%%%%%%%%%%%%%
    &\qquad + \underbrace{\sum_{g=1}^{G^0_T}\bigg[\max\{N^{opt}_{0,g-1}, N^{chg}_{0,g-1}\}\Delta^{opt}_{0, g-1} + N^{bse}_{0,g-1}\Delta^{opt}_{0, g-1} + d_{0,g}\Delta^{opt}_{0, g-1} +  \sum_{s=\tau^i_{c_{g-1}} }^{t^i_{c_g}-1}\Pb\left(\overline{\xi^{opt}_{0,g-1}(s)}\right)\Delta^{opt}_{0, g-1}\bigg]}_{\textbf{Part (B)}} \\
    %%%%%%%%%%%%%%%%%
    &\qquad +  \underbrace{\sum_{i=1}^K\sum_{g=1}^{G^i_T}\bigg[N^{chg}_{i, g} + 4d_{i,g} +  \sum_{s=t^i_{c_g} }^{\tau^i_{c_{g}}-1}\Pb\left(\xi^{chg}_{i, g}(s)\right)
    %+ \sum_{j=1}^K\sum_{s=t_{c_g} + 1}^{\tau_{c_{g}}-1}\Pb\left(\xi^{chg}_{j, g}\right)\Delta^{opt}_{i, g}
    \bigg]\Delta^{opt}_{i, g} }_{\textbf{Part (C)}}
    %%%%%%%%%%%%%%%%%
     + \sum_{g=1}^{G^0_T}\bigg[\underbrace{N^{chg}_{0, g} + 4d_{i,g} +  \sum_{s=t^i_{c_g} }^{\tau^i_{c_{g}}-1}\Pb\left(\xi^{chg}_{0, g}(s)\right) %+ \sum_{j=1}^K\sum_{s=t_{c_g} + 1}^{\tau^i_{c_{g}}-1}\Pb\left(\xi^{chg}_{j, g}\right)
    \bigg]\Delta^{opt}_{0, g}}_{\textbf{Part (D)}}\\  
    %%%%%%%%%%%%%%%%%%%
    &\qquad + \underbrace{\sum_{i=1}^K\sum_{g=1}^{G^i_T}\left(d_{i,g}\Delta^{opt}_{i, g} + \sum_{s=t^i_{c_g} }^{\tau^i_{c_{g}}-1}\Pb(\overline{\xi^{del}_{i,g}(s)})\Delta^{opt}_{i, g}\right)}_{\textbf{Part (E)}} + \underbrace{\gamma T}_{\textbf{Part (F)}}%\\
\end{align*}
\begin{align*}
    %%%%%%%%%%%%%%%%%%%%%%%%
    %%%%%%%%%%%%%%%%%%%%%%%%
    &\overset{(a)}{\leq} \sum_{i=1}^K\sum_{g=1}^{G^i_T}\max\left\{\dfrac{1}{\Delta^{opt}_{i,g-1}}, \dfrac{\Delta^{opt}_{i, g-1}}{\left(\Delta^{chg}_{i,g-1}\right)^2}\right\}8\log\left(\dfrac{4\log_2(t+1)}{\delta}\right) + \sum_{i=1}^K\sum_{g=1}^{G^i_T}d_{i,g} K \Delta^{opt}_{\max,g-1} + KG_T\Delta^{opt}_{\max, g-1}\sum_{t=1}^T\delta \\
    %%%%%%%%%%%%%%%%%%%
    &\qquad + \sum_{g=1}^{G^0_T}\max\left\{\dfrac{1}{\Delta^{opt}_{0,g-1}}, \dfrac{\Delta^{opt}_{0, g-1}}{\left(\Delta^{chg}_{0,g-1}\right)^2}\right\}8\log\left(\dfrac{4\log_2(t+1)}{\delta}\right) + N^{bse}_{0,g-1}\Delta^{opt}_{\max, g-1} + \sum_{g=1}^{G^0_T}d_{i,g}\Delta^{opt}_{0,g-1} + K\Delta^{opt}_{\max, g-1}\sum_{t=1}^T\delta\\
    %%%%%%%%%%%%%%%%%%
    &\qquad + \sum_{i=1}^K\sum_{g=1}^{G^i_T}\left( \dfrac{\Delta^{opt}_{i, g}}{\left(\Delta^{chg}_{i,g}\right)^2} 
    %+ \dfrac{\Delta^{opt}_{i, g}}{\left(\Delta^{chg}_{j,g}\right)^2}
    \right)8\log\left(\dfrac{4\log_2(t+1)}{\delta}\right)
    %%%%%%%%%%%%%%%%%%%%%%%
     +  \sum_{g=1}^{G^0_T}\left(\dfrac{\Delta^{opt}_{0, g}}{\left(\Delta^{chg}_{0,g}\right)^2} 
     %+ \dfrac{\Delta^{opt}_{j, g}}{\left(\Delta^{chg}_{j,g}\right)^2}
     \right)8\log\left(\dfrac{4\log_2(t+1)}{\delta}\right)\\
    %%%%%%%%%%%%%%%%%%%
    & \qquad + 8\sum_{i=1}^{K}\sum_{g=1}^{G^i_T}d_{i,g}\Delta^{opt}_{\max, g} + 2\left(\sum_{i=1}^K\sum_{g=1}^{G^i_T}\sum_{t=1}^T\delta %+ K^2G_T\sum_{t=1}^T\delta  
    + \sum_{i=1}^{K}\sum_{g=1}^{G^i_T}d_{i,g} + 6 \sum_{i=1}^K\sum_{g=1}^{G^i_T}\log (\log (T))\right)\Delta^{opt}_{i, g} + \gamma T\\
% \end{align*}
% \begin{align*}
    & \leq O\bigg(\sum_{i=1}^{K^+}\sum_{g=1}^{G^i_T}\left(\max\left\{\dfrac{1}{\Delta^{opt}_{i,g-1}}, \dfrac{\Delta^{opt}_{i,g-1}}{\left(\Delta^{chg}_{i,g-1}\right)^2}\right\} 
    %+ \sum_{i=1}^{K^+}\max_{j\in[K]^+}
    + \dfrac{\Delta^{opt}_{i, g}}{\left(\Delta^{chg}_{i,g}\right)^2} \right)\log\left(\dfrac{\log_2T}{\delta}\right)
    %%%%%%%%%%%%%%%%%%%%%%%%%
    + \sum_{i=1}^K\sum_{g=1}^{G^i_T}\Delta^{opt}_{\max, g}d_{i,g}\\
    %%%%%%%%%%%%%%%%%%%%
    &\qquad + N^{bse}_{0,g-1}\Delta^{opt}_{0, g-1} + \sum_{i=1}^K\sum_{g=1}^{G^i_T}\Delta^{opt}_{i, g}\sum_{t=1}^T\delta\bigg) + \gamma T\\
    %%%%%%%%%%%%%%%%%%%%%%%%%
    %%%%%%%%%%%%%%%%
    & \overset{(b)}{=} O\bigg(\sum_{i=1}^{K^+}\sum_{g=1}^{G^i_T}\left(H^{(1)}_{i,g-1} 
    + \overline{H^{(2)}_{i,g}} \right)\log\left(\dfrac{\log_2T}{\delta}\right) +  \sum_{i=1}^K\sum_{g=1}^{G^i_T}\Delta^{opt}_{\max, g}\dfrac{1}{\alpha\mu_{0,g-1}}\sum_{i\in[K]}\dfrac{\log(\log_2T/\delta)}{\max\{\Delta^{opt}_{i,g-1}, \Delta^{opt}_{0,g-1} - \Delta^{opt}_{i,g-1}\}}\\
    %%%%%%%%%%%%%%%%%%%%%%%%%
    &\qquad + \sum_{i=1}^K\sum_{g=1}^{G^i_T}\Delta^{opt}_{\max, g}\dfrac{1}{\alpha\mu_{0,g}}\sum_{i\in[K]}\dfrac{\log(\log_2T/\delta)}{\max\{\Delta^{opt}_{i,g}, \Delta^{opt}_{0,g} - \Delta^{opt}_{i,g}\}} + KG_T\Delta^{opt}_{\max}\sum_{t=1}^T\delta\bigg) + \gamma T\\
    %%%%%%%%%%%%%%%%%%%%%%%%%
    %%%%%%%%%%%%%%%%%%%%%
    % & \overset{(c)}{=} O\bigg(\sum_{g=1}^{G^i_T}\left(H_{1,g-1} 
    % + K H_{2,g} \right)\log\left(\dfrac{\log_2T}{\delta}\right)
    % %%%%%%%%%%%%%%%%%%%%%%%%%
    % + K\sum_{g=1}^{G_T} H_{3,g} \log\left(\dfrac{\log_2( t + 1)}{\delta}\right)
    % + KG_T\Delta^{opt}_{\max, g}\sum_{t=1}^T\delta\bigg)\\
    % %%%%%%%%%%%%%%%%%%%%%%%%%%%
    % %%%%%%%%%%%%%%%%%%%%%%%%%%
    % & = O\bigg(\sum_{i=1}^{K^+}\sum_{g=1}^{G^i_T}\left(H^{(1)}_{i,g-1} 
    % + H^{(2)}_{i,g} \right)\log\left(\dfrac{\log_2T}{\delta}\right)
    % %%%%%%%%%%%%%%%%%%%%%%%%%
    % + K\sum_{g=1}^{G^i_T}\dfrac{1}{\alpha\mu_{0,g}}\sum_{i=1}^K H^{(3)}_{i,g} \log\left(\dfrac{\log_2( t + 1)}{\delta}\right)
    % + KG_T\Delta^{opt}_{\max, g}\sum_{t=1}^T\delta\bigg)\\
    %%%%%%%%%%%%%%%%%%%%%
    & \overset{(c)}{=} O\bigg(\left(\sum_{i=1}^{K^+}\sum_{g=1}^{G^i_T}\left(H^{(1)}_{i,g-1} 
    + \overline{H^{(2)}_{i,g}} \right) + \sum_{i=1}^K\sum_{g=1}^{G^i_T}\dfrac{1}{\alpha\mu_{0,g-1}}\sum_{i=1}^K H^{(3)}_{i,g-1} + \sum_{i=1}^K\sum_{g=1}^{G^i_T}\dfrac{1}{\alpha\mu_{0,g}}\sum_{i=1}^K H^{(3)}_{i,g}\right)\log\left(\dfrac{\log_2T}{\delta}\right)
    %%%%%%%%%%%%%%%%%%%%%%%%%
    % + KG_T\Delta^{opt}_{\max}\sum_{t=1}^T\delta
    \bigg) + \gamma T
\end{align*}
where, $(a)$ follows from combining all the steps above, $(b)$ follows by using the definition of 
\begin{align*}
    H^{(1)}_{i,g-1} \coloneqq \max\left\{\dfrac{1}{\Delta^{opt}_{i,g-1}}, \dfrac{\Delta^{opt}_{i,g-1}}{\left(\Delta^{chg}_{i,g-1}\right)^2}\right\}, \qquad
    %%%%%%%%%%%%%%%%%%%%
    \overline{H^{(2)}_{i,g}} \coloneqq \dfrac{\Delta^{opt}_{i, g}}{\left(\Delta^{chg}_{i,g}\right)^2},
\end{align*}
and substituting the value of $d_{i,g}$. Finally, $(c)$ follows by using the definition of 
\begin{align*}
   H^{(3)}_{i,g} \coloneqq \dfrac{\Delta^{opt}_{\max, g}}{\max\{\Delta^{opt}_{i,g}, \Delta^{opt}_{0,g} - \Delta^{opt}_{i,g}\}}.
\end{align*}
The claim of the theorem follows.
% where, $(a)$ follows from definition of $d^i_g$, bounding the total detection delay term Part (F), and summing over each the safe set of arms for each segment $\rho^i_g$. The quantity $D^{k, bse}_{0, g-1}$, $D^{k, thd}_{j, g-1}$ is as defined in \Cref{thm:no-change} but restricted to $\rho^k_{g-1}$. The claim of the theorem follows.
\end{proof}

\subsection{Proof of Corollary 1}
\label{app:corollary-glb-loc-change}

\begin{customcorollary}{1}\textbf{(Gap independent bound, Restatement)}
Setting $\Delta^{opt}_{i,g} = \Delta^{chg}_{i,g} = \sqrt{\frac{K\log T}{T}}$ for all $i \in [K]^+$ and exploration rate  $\gamma=\sqrt{\frac{\log T}{T}}$ we obtain the gap independent regret upper bound of \glb and \loc as
\begin{align*}
    \E[R_T] &\!\leq\! O\!\left(G_TK\sqrt{KT\log T}\! +\! \dfrac{G_T \log T}{\alpha \mu_{0,\min}}\!\right), \textbf{(\glb)}\\
    %%%%%%%%%%%%%%%%%%
    \E[R_T] &\!\leq\! O\!\left(G_T\sqrt{KT\log T}\! +\! \dfrac{G_T \log T}{\alpha \mu_{0,\min}}\!\right),\textbf{(\loc\!\!)}
\end{align*}
where $\alpha$ is the risk parameter.
\end{customcorollary}

\begin{proof}
The proof directly follows from the result of \Cref{thm:glb-change} and \Cref{thm:loc-change}. We first recall the result of \Cref{thm:glb-change} below
\begin{align}
    \E[R_t] \!\leq\! O\bigg(\left(\sum_{g=1}^{G_T}\sum_{i=1}^{K^+}\left(H^{(1)}_{i,g-1}
    + H^{(2)}_{i,g} \right) + K\sum_{g=1}^{G_T}\dfrac{1}{\alpha\mu_{0,g-1}}\sum_{i=1}^K H^{(3)}_{i,g-1} + K\sum_{g=1}^{G_T}\dfrac{1}{\alpha\mu_{0,g}}\sum_{i=1}^K H^{(3)}_{i,g}\right)\log\left(\dfrac{\log_2T}{\delta}\right)
    %%%%%%%%%%%%%%%%%%%%%%%%%
    % + KG_T\Delta^{opt}_{\max, g}\sum_{t=1}^T\delta
    \bigg). \label{eq:corollary-glb}
\end{align}
Then using the fact that $\Delta^{opt}_{i,g} = \Delta^{chg}_{i,g} = \sqrt{\frac{K\log T}{T}}$ for all $i \in [K]^+$ we get that 
\begin{align*}
    H^{(1)}_{i,g-1} &\coloneqq \max\left\{\dfrac{1}{\Delta^{opt}_{i,g-1}}, \dfrac{\Delta^{opt}_{i,g-1}}{\left(\Delta^{chg}_{i,g-1}\right)^2}\right\} = \sqrt{\frac{T}{K\log T}},\\
    %%%%%%%%%%%%%%%%%%%%
    %%%%%%%%%%%%%%
    %D^{bse}_{0, g-1}  =  \dfrac{\sqrt{KT}}{\alpha\sqrt{\log T}}, \\
    %%%%%%%%%%%%%%%%%%%%
    H^{(2)}_{i,g} &\coloneqq \max_{j\in[K]^+}\dfrac{\Delta^{opt}_{i, g}}{\left(\Delta^{chg}_{j,g}\right)^2} = \sqrt{\frac{T}{K\log T}}, \\
    %%%%%%%%%%%%%%
    H^{(3)}_{i,g} &\coloneqq \dfrac{\Delta^{opt}_{\max, g}}{\max\{\Delta^{opt}_{i,g}, \Delta^{opt}_{0,g} - \Delta^{opt}_{i,g}\}} = 1.
\end{align*}
Substituting all of the above back in \Cref{eq:corollary-glb} and setting $\delta = \frac{1}{T}$ we get that
\begin{align*}
    \E[R_T] \leq O\left(G_TK\sqrt{KT\log T} + \dfrac{G_T \log T}{\alpha \mu_{0,\min}}\right).
\end{align*}
% where $\S_{\max} = \{i\in[K]: \exists g\in \{0,1,\ldots, G_T\}, i\in \S(t_{c_{g-1}}: t_{c_{g}}-1) \}$ is the set of all threshold arms across all segments $\rho_0, \rho_1, \ldots, \rho_{G_T}$. 
Next we recall the result of \Cref{thm:loc-change} below
\begin{align}
    \E[R_t] \!\leq\! O\bigg(\left(\sum_{i=1}^{K^+}\sum_{g=1}^{G^i_T}\left(H^{(1)}_{i,g-1} 
    + \overline{H^{(2)}_{i,g}} \right) + \sum_{i=1}^K\sum_{g=1}^{G^i_T}\dfrac{1}{\alpha\mu_{0,g-1}}\sum_{i=1}^K H^{(3)}_{i,g-1} + \sum_{i=1}^K\sum_{g=1}^{G^i_T}\dfrac{1}{\alpha\mu_{0,g}}\sum_{i=1}^K H^{(3)}_{i,g}\right)\log\left(\dfrac{\log_2T}{\delta}\right)
    %%%%%%%%%%%%%%%%%%%%%%%%%
    % + KG_T\Delta^{opt}_{\max}\sum_{t=1}^T\delta
    \bigg) + \gamma T. \label{eq:corollary-loc}
\end{align}
Again using the fact that $\Delta^{opt}_{i,g} = \Delta^{chg}_{i,g} = \sqrt{\frac{K\log T}{T}}$ for all $i \in [K]^+$ it follows that $\overline{H^{(2)}_{i,g}} = \sqrt{T/K\log T}$. 
%Let $\S^k_{\max} = \{i\in[K]: \exists g\in \{0,1,\ldots, G^k_T\}, i\in \S(t^k_{c_{g-1}}: t^k_{c_{g}}-1), \mu_{i, g} \geq  \alpha B \}$ be the set of all threshold arms across all segments $\rho^k_0, \rho^k_1, \ldots ,\rho^k_{G_T}$. It follows that $\S^k_{\max} \leq K\S_{\max}$.  
Substituting this back in \eqref{eq:corollary-loc} we get that
\begin{align*}
    \E[R_T] \leq O\left(G_T\sqrt{KT\log T} + \dfrac{G_T \log T}{\alpha \mu_{0,\min}}\right).
\end{align*}
The claim of the corollary follows.
\end{proof}

\subsection{Support Lemma}

\subsubsection{Concentration of the Optimality Event}
\begin{lemma}
\label{lemma:conc-opt}
Let $X_i(1),X_i(2),\ldots,X_i(t)$ be $t$ samples observed for arm $i$. Define the optimality good event $\xi^{opt}_i(1:t) \coloneqq \{\forall s\in [1:t], \mid\widehat{\mu}_i(1:s) - \mu_i(1:s) \mid \leq \beta_i(1:s,\delta)\}$ where $\beta_i(1:s,\delta) \coloneqq \sqrt{\dfrac{2\log (4\log_2( t + 1) / \delta)}{s}}$. Then the probability of the event $\xi^{opt}_i(1:t)$ is bounded by
\begin{align*}
    \Pb\left(\xi^{opt}_i(1:t)\right) \geq 1 - \delta.
\end{align*}
\end{lemma}

\begin{proof} We will use the peeling argument to get the desired bound. We first restate the bad event of deviation as
\begin{align}
    \overline{\xi^{opt}_{i}}(1:t) \coloneqq \left\{\exists s\in [1:t], |\widehat{\mu}_i(1:s) - \mu_i(1:s)| >  \sqrt{\dfrac{2\log(4\log_2( t + 1)/\delta)}{s}}\right\}\label{eq:dev-opt-event}.
\end{align}
Next we bound the one side of the inequality as follows:
\begin{align}
    \Pb\!\left(\exists s\geq 1, \widehat{\mu}_i(1:s) - \mu_i(1:s) \geq \sqrt{\dfrac{2\log(4\log_2( t + 1) /\delta)}{s}} \!\right) \!&=\! \Pb\left(\exists s\geq 1, s\underbrace{(\widehat{\mu}_i(1:s) - \mu_i(1:s))}_{\coloneqq Y_s} \geq \sqrt{\dfrac{2s^2 \log(4\log_2( t + 1)/\delta)}{s}} \right)\nonumber\\
    %%%%%%%%%%%%%%%%%%%%%%%%
    %&= \Pb\left(\exists s\geq 1, sM_s \geq \sqrt{\dfrac{s^2 \log((\log t + 1)/\delta)}{s}} \right)\nonumber\\
    %%%%%%%%%%%%%%%%%%%%%%%%%
    &\overset{(a)}{\leq} \sum_{u=0}^{\lfloor \log_2 t \rfloor}\Pb\left(\exists s\in [2^{u},2^{u+1}], M_s  \geq \sqrt{2s\log(4\log_2( t + 1)/\delta)} \right)\nonumber\\
    %%%%%%%%%%%%%%%%%%%%%%%%%%
    &\overset{(b)}{\leq} \sum_{u=0}^{\lfloor \log_2 t \rfloor}\Pb\left(\exists s\leq 2^{u+1}, M_s\geq  \sqrt{2.2^{u}\log(4\log_2( t + 1)/\delta)} \right)\nonumber\\
    %%%%%%%%%%%%%%%%%%%%%%%%%
    &\overset{(c)}{\leq}\sum_{u=0}^{\lfloor \log_2 t \rfloor}\exp\left(-\dfrac{( \sqrt{2.2^{u}\log(4\log_2( t + 1)/\delta)})^2}{2.2^{u}\sigma^2}\right)\nonumber\\
    %%%%%%%%%%%%%%%%%%%%%%%%%
    &\overset{(d)}{\leq}\sum_{u=0}^{\lfloor \log t \rfloor}\exp\left(-\dfrac{ 2^{u}\log(4\log_2( t + 1)/\delta)}{2^{u}}\right) \nonumber\\
    %%%%%%%%%%%%%%%%%%%%%%%%%
    % &\overset{}{=} \sum_{u=0}^{\lfloor \log_2 t \rfloor}\exp\left(-4\log((\log_2( t + 1))/\delta)\right)\nonumber \\
    %%%%%%%%%%%%%%%%%%%%%%%%%%%%
    &\overset{(e)}{\leq} (\log_2( t + 1)) \exp\left(-\log(4\log_2( t + 1)/\delta)\right)\leq \delta/4,  \label{eq:conc-bound:1}
\end{align}
where, $(a)$ follows by dividing the rounds till $t$ into geometric grid of size $[2^u, 2^{u+1}]$, $s\leq\lceil \log t\rceil$ and $M_s = \sum_{t'=1}^s (X_i(t') - \E[X_i(t')]) = \sum_{t'=1}^s Y_{t'}$. Further note that $\E[Y_{t'}] = 0$, so we can apply the maximal inequality. Then $(b)$ follows by upper bounding $s$ by $2^{u+1}$, $(c)$ follows by applying \Cref{prop:Doob}, $(d)$ follows as $\sigma^2 = 1$, %$(e)$ follows as $2^{2u}\mu^2_i(1:s) \geq 0$ for $s\leq 2^{u+1}$, 
$(e)$ follows as with $s\leq\lceil \log t\rceil$ we can have $(\log t + 1)$ such combinations. Similarly, we can show that,
\begin{align}
   \Pb\left(\exists s\geq 1, \widehat{\mu}_i(1:s) - \mu_i(1:s) \leq -  \sqrt{\dfrac{2\log(4\log_2( t + 1) /\delta)}{s}} \right) \leq \delta/4 \label{eq:conc-bound:2}
\end{align}
Combining the equations \eqref{eq:conc-bound:1} and \eqref{eq:conc-bound:2} we get that,
\begin{align*}
    \Pb\left(\overline{\xi^{opt}_{i}}(1:t)\right) \leq \dfrac{\delta}{2} < \delta.
\end{align*}
It follows then $\Pb\left(\xi^{opt}_i(1:t)\right) \geq 1 - \delta$. Hence the claim of the lemma follows.
\end{proof}

\subsubsection{Critical Samples for Optimality Detection}

\begin{lemma}
\label{lemma:cric-sample-opt}
Define the optimality event $\xi^{opt}_i(s:t)$ as in \Cref{lemma:conc-opt}. Then the expected number of times the sub-optimal arm $i$ is sampled based on the UCB sampling rule from round $1$ till $t$ is given by
\begin{align*}
    \E[N_i(1:t)] \leq N^{opt}_i(1:t) + \sum_{s=N^{opt}_i(1:t) + 1}^t\Pb\left(\overline{\xi^{opt}_i}(s:t)\right)
\end{align*}
where, $N^{opt}_i(1:t) = \dfrac{8\log(4\log_2(t + 1)/\delta)}{(\Delta^{opt}_{i})^2}$.
\end{lemma}

\begin{proof}
First note that if a sub-optimal arm $i$ is chosen at round $t$ then $U_i(1:s) > U_{i^*}(1:s)$. This is possible under the following three events
\begin{align}
    \{\wmu_i(1:s) \geq \mu_{i}(1:s) + \beta_i(1:s,\delta)\}, \quad \{\wmu_{i^*}(1:s) \leq \mu_{i^*}(1:s) - \beta_i(1:s,\delta)\},  \quad \{\mu_{i^*}(1:s) - \mu_{i}(1:s) < 2\beta_i(1:s,\delta)\}. \label{eq:cric-sample-opt}
\end{align}
Now recall from Lemma \ref{lemma:conc-opt} that at round $s\in[t]$, the event $\xi^{opt}_i(1:t)$ holds with $1-\delta$ probability. We define the optimal stopping time $\tau^{opt}_{i,s}(1:t)$ as follows:
\begin{align*}
    \tau^{opt}_{i,s}(1:t) = \min\{t'\in[1:t]: N_i(1:t') = s\}.
\end{align*}
It follows immediately that $\tau^{opt}_{i,s}(1:t)$ is $\F_{t-1}$ measurable. Also note that that the third event in \eqref{eq:cric-sample-opt} is not possible for
\begin{align*}
    \mu_{i^*}(1:t) - \mu_{i}(1:t) < 2\beta_i(1:t,\delta) \implies \Delta^{opt}_{i} < 2\sqrt{\dfrac{2\log(4\log_2( t + 1)/\delta)}{N^{opt}_{i}(1:t)}} \implies N^{opt}_{i}(1:t) > \dfrac{8\log(4\log_2(t + 1)/\delta)}{(\Delta^{opt}_{i})^2}. 
\end{align*}
Then we can bound the expected number of pulls for any sub-optimal arm $i$ as follows:
\begin{align*}
    \E[N_i(1:t)] &= \E[\sum_{s=1}^{N^{opt}_{i}(1:t)}\I{\tau^{opt}_{i,s}(1:t) \leq N^{opt}_i(1:t)}] + \E[\sum_{s = N^{opt}_i(1:t) + 1}^t \I{\tau^{opt}_{i,s}(1:t) > N^{opt}_i(1:t)}] \\
    %%%%%%%%%%%%%%%%%%%%%
    &= N^{opt}_i(1:t) + \sum_{s = N^{opt}_i(1:t) + 1}^t \Pb\left(\tau^{opt}_{i,s}(1:t) > N^{opt}_i(1:t)\right) \\
    %%%%%%%%%%%%%%%%%%%%
    &\leq  N^{opt}_i(1:t) + \sum_{s = N^{opt}_i(1:t) + 1}^t\Pb\left(\left\{I_s = i, \wmu_i(1:s-1) + \beta_{i}(1:s-1,\delta) > \wmu_{i^*}(1:s-1) + \beta_{i^*}(1:s-1,\delta)\right\}\right)\\
    %%%%%%%%%%%%%%%%%%%
    &\leq  N^{opt}_i(1:t) + \sum_{s = N^{opt}_i(1:s-1) + 1}^t\Pb\bigg(\left\{\wmu_i(1:s-1) \geq \mu_{i}(1:s-1) + \beta_i(1:s-1,\delta)\right\}\\
    %%%%%%%%%%%%%%%%%
    &\quad\quad\bigcup\left\{\wmu_{i^*}(1:s-1) \leq \mu_{i^*}(1:s-1) - \beta_{i^*}(1:s-1,\delta)\right\}\bigg)\\
    %%%%%%%%%%%%%%%%%%%
    &\leq N^{opt}_i(1:t) + \sum_{s=N^{opt}_i(1:t) + 1}^t\Pb\left(\overline{\xi^{opt}_i}(s:t)\right) + \sum_{s=N^{opt}_i(1:t) + 1}^t\Pb\left(\overline{\xi^{opt}_{i^*}}(s:t)\right).
\end{align*}
The claim of the lemma follows.
\end{proof}

\subsubsection{Concentration of Changepoint Event}

\begin{lemma}
\label{lemma:conc-chg}
Let $\wmu_{i}(1:s)$ be the empirical mean of $s$ i.i.d. observations with mean $\mu_{i}(1:s)$, and $\wmu_{i}(s+1:t)$ be the empirical mean of $(t - s)$ i.i.d. observations with mean $\mu_{i}(s+1:t)$. Let $t^i_{c_g}$ be the round such that $\mu_{i}(t^i_{c_g}) \neq \mu_{i}(t^i_{c_g}+1)$. Let at round $s$ the policy raises an alarm following the changepoint good event $\xi^{chg}_{i,g}(1:t)$. Define the changepoint good event 
$$
\xi^{chg}_{i,g}(1:t) \coloneqq \{\exists s\in [1:t], \mid\widehat{\mu}_i(1:s) - \widehat{\mu}_i(s+1:t) \mid \geq \beta_i(1:s,\delta) + \beta_i(s+1:t,\delta)\}
$$ 
where $\beta_i(1:s,\delta) \coloneqq \sqrt{\dfrac{2\log (4\log_2( t + 1) / \delta)}{s}}$ and $\beta_i(s+1:t,\delta) \coloneqq \sqrt{\dfrac{2\log (4\log_2( t + 1) / \delta)}{t-s}}$.  Then the probability of the event $\xi^{chg}_i(1:t)$ is bounded by
\begin{align*}
    \Pb\left(\xi^{chg}_{i,g}(1:t)\right) \geq 1 - \delta.
\end{align*}
\end{lemma}

\begin{proof} Recall that $\wmu_{i}(1:s)$ is the empirical mean of $s$ i.i.d. observations with mean $\mu_{i}(1:s)$, and $\wmu_{i}(s+1:t)$ is the empirical mean of $(t - s)$ i.i.d. observations with mean $\mu_{i}(s+1:t)$. We will again use the peeling argument to get the desired bound. We first restate the good event as follows:
\begin{align}
    \xi^{chg}_{i,g}(1:t) &\coloneqq \bigg\{\exists s\in [1:t], \bigg|\widehat{\mu}_i(1:s) - \widehat{\mu}_i(s+1:t)\bigg| >  \sqrt{\dfrac{2\log(4\log_2( t + 1)/\delta)}{s}} + \sqrt{\dfrac{2\log(4\log_2( t + 1)/\delta)}{t - s}}\bigg\}\nonumber\\
    %%%%%%%%%%%%%%%%%%%%%%
    &= \bigg\{\exists s\in [1:t], \widehat{\mu}_i(1:s) - \sqrt{\dfrac{2\log(4\log_2( t + 1)/\delta)}{s}}  >  \widehat{\mu}_i(s+1:t) + \sqrt{\dfrac{2\log(4\log_2( t + 1)/\delta)}{t - s}}\bigg\} \nonumber\\ 
    %%%%%%%%%%%%%%%%%%%%
    & \quad \quad \bigcup\bigg\{\widehat{\mu}_i(1:s) + \sqrt{\dfrac{2\log(4\log_2(t + 1)/\delta)}{s}}  <  \widehat{\mu}_i(s+1:t) - \sqrt{\dfrac{2\log(4\log_2(t + 1)/\delta)}{t - s}} \bigg\}\nonumber
\end{align}
We can then redefine the good event $\xi^{chg}_{i,g}(1:t)$ as follows:
\begin{align}
    %%%%%%%%%%%%%%%%%%%%%%%
    \xi^{chg}_{i,g}(1:t) &\coloneqq \bigg\{\forall s\in [1:t], |\widehat{\mu}_i(1:s) - \mu_i(1:s)| \leq  \sqrt{\dfrac{2\log(4\log_2(t + 1)/\delta)}{s}}\bigg\} \nonumber\\ 
    %%%%%%%%%%%%%%%%%%%%
    & \quad \quad \bigcup\bigg\{\forall s\in [1:t], |\widehat{\mu}_i(s+1:t) - \mu_i(s+1:t)| \leq  \sqrt{\dfrac{2\log(4\log_2(t + 1)/\delta)}{s}}\bigg\}\label{eq:dev-chg-event}.
\end{align}
We can then define the bad event from \cref{eq:dev-chg-event} as follows:
\begin{align}
    %%%%%%%%%%%%%%%%%%%%%%%
    \overline{\xi^{chg}_{i,g}}(1:t^i_{c_g}:t) &\coloneqq \bigg\{\exists s\in [1:t], |\widehat{\mu}_i(1:s) - \mu_i(1:s)| >  \sqrt{\dfrac{2\log(4\log_2(t + 1)/\delta)}{s}}\bigg\} \nonumber\\ 
    %%%%%%%%%%%%%%%%%%%%
    & \quad \quad \bigcup\bigg\{\exists s\in [1:t],|\widehat{\mu}_i(s+1:t) - \mu_i(s+1:t)| >  \sqrt{\dfrac{2\log(4\log_2(t + 1)/\delta)}{s}}\bigg\}\label{eq:dev-chg-event1}.
\end{align}
% Next we can define the bad event $\left(\xi^{chg}_i(1:t)\right)^c$ using \cref{eq:dev-chg-event} as follows:
% \begin{align}
%     \left(\xi^{chg}_{i,g}(1:t^i_{c_g}:t)\right)^c &\coloneqq \bigg\{\exists s\in [1:t],  |\widehat{\mu}_i(1:s) - \mu_i(1:s)| >  \sqrt{\dfrac{2\log(4\log_2(t + 1)/\delta)}{s}}\bigg\}  \nonumber\\ 
%     %%%%%%%%%%%%%%%%%%%%
%     & \quad \quad \bigcup \bigg\{|\widehat{\mu}_i(s+1:t) - \mu_i(s+1:t)| \leq  \sqrt{\dfrac{2\log(4\log_2(t + 1)/\delta)}{s}} \bigg\}\label{eq:dev-chg-event1}
% \end{align}
Next we bound the first event of the inequality in \cref{eq:dev-chg-event1} in the same way as \Cref{lemma:conc-opt} as follows:
\begin{align}
    &\Pb\left(\exists s\geq 1, \widehat{\mu}_i(1:s) - \mu_i(1:s) \geq \sqrt{\dfrac{2\log(4\log_2(t + 1) /\delta)}{s}} \right) 
    %&= \Pb\left(\exists s\geq 1, s\widehat{\mu}_i(1:s) - s\mu_i(1:s) \geq \sqrt{\dfrac{2s^2 \log(4\log_2(t + 1)/\delta)}{s}} \right)
    % \nonumber\\
    % %%%%%%%%%%%%%%%%%%%%%%%%%
    % &\overset{(a)}{\leq} \sum_{u=0}^{\lfloor \log t \rfloor}\Pb\left(\exists s\in [2^{u},2^{u+1}], M_s - s\mu_i(1:s) \geq \sqrt{s\log((\log t + 1)/\delta)} \right)\nonumber\\
    % %%%%%%%%%%%%%%%%%%%%%%%%%%%%
    % &\overset{(b)}{\leq} (\log t + 1) \exp\left(-\log((\log t + 1)/\delta)\right)
    \leq \dfrac{\delta}{4}\label{eq:conc-bound-chg1}\\
    %%%%%%%%%%%%%%
    &\Pb\left(\exists s\geq 1, \widehat{\mu}_i(1:s) - \mu_i(1:s) \leq -  \sqrt{\dfrac{2\log(4\log_2(t + 1) /\delta)}{s}} \right) \leq \dfrac{\delta}{4}.\label{eq:conc-bound-chg2}
\end{align}
% where, $(a)$ follows by dividing the rounds till $t$ into geometric grid of size $[2^u, 2^{u+1}]$, $s\leq\lceil \log t\rceil$ and $M_s = \sum_{t'=1}^s X_i(t')$, $(b)$ follows the same way as in steps $(b)-(f)$ in \Cref{lemma:conc-opt}. 
Similarly, we can show that the second event of \cref{eq:dev-chg-event1} the inequality,
\begin{align}
%   &\Pb\left(\exists s\geq 1, \widehat{\mu}_i(1:s) - \mu_i(1:s) \leq -  \sqrt{\dfrac{\log((\log t + 1) /\delta)}{s}} \right) \leq \dfrac{\delta}{4}\label{eq:conc-bound-chg2}\\
   %%%%%%%%%%%%%%%%%%%%%%%%%%%%
   &\Pb\left(\exists s\geq 1, \widehat{\mu}_i(s+1:t) - \mu_i(s+1:t) \geq \sqrt{\dfrac{2\log(4\log_2(t + 1) /\delta)}{t-s}} \right) \leq \dfrac{\delta}{4}\label{eq:conc-bound-chg3}\\
   %%%%%%%%%%%%%%%%%%%%%%%%%%%%%
   &\Pb\left(\exists s\geq 1, \widehat{\mu}_i(s+1:t) - \mu_i(s+1:t) \leq -  \sqrt{\dfrac{2\log(4\log_2(t + 1) /\delta)}{t-s}} \right) \leq \dfrac{\delta}{4}.\label{eq:conc-bound-chg4}
\end{align}
Combining the equations \eqref{eq:conc-bound-chg1}, \eqref{eq:conc-bound-chg2}, \eqref{eq:conc-bound-chg3} and \eqref{eq:conc-bound-chg4} we get that,
\begin{align*}
    \Pb\left(\overline{\xi^{chg}_{i,g}}(1:t)\right) \leq \delta.
\end{align*}
It follows then $\Pb\left(\xi^{chg}_{i,g}(1:t)\right) \geq 1 - \delta$. 
The claim of the lemma follows.
\end{proof}

\subsubsection{Critical Number of samples for Changepoint Detection}

\begin{lemma}
\label{lemma:cric-sample-chg}
Let $\wmu_{i}(1:s)$ be the empirical mean of $s$ i.i.d. observations with mean $\mu_{i}(1:s) =: \mu_{i,g-1}$, and $\wmu_{i}(s+1:t)$ be the empirical mean of $(t - s)$ i.i.d. observations with mean $\mu_{i}(s+1:t) =: \mu_{i,g}$. Let $\Delta^{chg}_{i,g}\coloneqq |\mu_{i,g-1} - \mu_{i,g}|$ and the changepoint at $s+1$ round as $t^i_{c_g}$. Define the changepoint event $\xi^{chg}_{i,g}(1:t)$ as in \Cref{lemma:conc-chg}. Then the expected number of times the sub-optimal arm $i$ is sampled before a changepoint $t^i_{c_g}$ is detected is given by
\begin{align*}
    \E[N_i(1:t)] \leq N^{chg}_{i,g}(1:t) + \sum_{s=N^{chg}_{i,g}(1:t) + 1}^t\Pb\left(\overline{\xi^{chg}_{i,g}}(s:t)\right)
\end{align*}
where, $N^{chg}_{i,g}(1:t) = \dfrac{8\log (4\log_2(t+1)/\delta)}{\left(\Delta^{chg}_{i,g}\right)^2}$.
\end{lemma}

\begin{proof}
Note that the changepoint $t^i_{c_g}$ lie between the round $1$ and $t$. Also note that the mean of the arm $i$ for the rounds $1:s$ is given by $\mu(1:s) =: \mu_{i,g-1}$ and for the rounds $s+1:t$ is given by $\mu(s+1:t) =: \mu_{i,g}$. Then for a sub-optimal arm $i$ it is possible to detect the changepoint at some round $s\in [1:t]$ under the following two events
\begin{align}
    \{\wmu_i(1:s) - \beta_i(1:s,\delta) > \wmu_{i}(s+1:t) + \beta_i(s+1:t,\delta)\}, \quad\{\wmu_i(1:s) + \beta_i(1:s,\delta) < \wmu_{i}(s+1:t) - \beta_i(s+1:t,\delta)\}.\label{eq:cric-sample-chg}
\end{align}
Again note that the events in \cref{eq:cric-sample-chg} is not possible if the following events holds true:
\begin{align}
    &\left\{\bigg|\widehat{\mu}_i(1:s) - \widehat{\mu}_i(s+1:t)\bigg| \leq  \sqrt{\dfrac{2\log(4\log_2(t + 1)/\delta)}{s}} + \sqrt{\dfrac{2\log(4\log_2(t + 1)/\delta)}{t - s}}\right\},\nonumber\\
    %%%%%%%%%%%%%%%%
    &\text{and,}\qquad |\mu_{i, g-1} - \mu_{i, g}| < \max\{\beta_i(1:s,\delta), \beta_i(s+1:t,\delta)\}\label{eq:cric-sample-chg1}.
\end{align}
Now recall from \Cref{lemma:conc-opt} that at round $t$, the event $\left(\xi^{chg}_i(1:t)\right)$ holds with $1-\delta$ probability. We define the changepoint stopping time $\tau^{}_{c_g}$ as follows:
\begin{align*}
    \tau^{}_{c_g}=\min \left\{t: \exists s \in[1, t],|\widehat{\mu}_i(1:s) - \widehat{\mu}_i(s+1:t)| \leq  \sqrt{\dfrac{2\log(4\log_2(t + 1)/\delta)}{s}} + \sqrt{\dfrac{2\log(4\log_2(t + 1)/\delta)}{t - s}}\right\}
\end{align*}
It follows that $\tau^{}_{c_g}$ is $\F_{t-1}$ measurable. Also note that that the second event in \eqref{eq:cric-sample-chg1} is not possible for
\begin{align*}
    |\mu_{i, g-1} - \mu_{i, g}| < \max\{\beta_i(1:s,\delta), \beta_i(s+1:t,\delta)\} &\implies \Delta^{chg}_{i,g} < 2\sqrt{\dfrac{2\log(4\log_2(t + 1)/\delta)}{N^{chg}_{i,g}(1:t)}}\\
    %%%%%%%%%%%%%%%%
    &\implies N^{chg}_{i,g}(1:t) > \dfrac{8\log(4\log_2(t + 1)/\delta)}{(\Delta^{chg}_{i})^2}. 
\end{align*}

Let $N^{chg}_{i,g}(1:t)$ be the total number of samples of arm $i$ before the changepoint $t^i_{c_g}$ is detected. Then we can bound the expected number of pulls of $N^{chg}_{i,g}(1:t)$ as follows:
\begin{align*}
    &\E[N^{chg}_{i,g}(1:t)] = \E[\sum_{s=1}^{N^{chg}_{i,g}(1:t)}\I{\tau^{}_{c_g} \leq N^{chg}_{i,g}(1:t)}] + \E[\sum_{s = N^{chg}_{i,g}(1:t) + 1}^t \I{\tau^{}_{c_g} > N^{chg}_{i,g}(1:t)}] \\
    %%%%%%%%%%%%%%%%%%%%
    &\leq  N^{chg}_{i,g}(1:t) + \sum_{s = N^{chg}_{i,g}(1:t) + 1}^t\Pb\left(\left\{\tau^{}_{c_g} > N^{chg}_{i,g}(1:t)\right\}\right)\\
    %%%%%%%%%%%%%%%%%%%
    &\leq  N^{chg}_{i,g}(1:t) + \sum_{s = N^{chg}_{i,g}(1:t) + 1}^t\Pb\left(\left\{\left|\wmu_i(1:s) - \wmu_i(s+1:t)\right| >  \sqrt{\dfrac{2\log(4\log_2(t + 1)/\delta)}{s}} + \sqrt{\dfrac{2\log(4\log_2(t + 1)/\delta)}{t - s}}\right\}\right)\\
    %%%%%%%%%%%%%%%%%%%
    &\leq N^{chg}_{i,g}(1:t) + \sum_{s=N^{chg}_{i,g} + 1}^t\Pb\left(\overline{\xi^{chg}_{i,g}}(s:t)\right).
\end{align*}
The claim of the lemma follows.
\end{proof}

\subsubsection{Total Number of Samples of Baseline Arm}

\begin{lemma}
\label{lemma:cric-sample-bse}
Let us consider a single time segment $\rho_g$ starting from round $1$ till round $t$. When the optimality event $\xi^{opt}(1:t)$ holds with high probability $1-\delta$ in $\rho_g$ then the total pulls of the baseline arm is bounded by
\begin{align*}
    N^{bse}_0(1:t) &\overset{}{\leq} \dfrac{1}{\alpha\mu_{0,g}}\sum_{i\in[K]}\dfrac{16\log(4\log_2(\tau + 1)/\delta)}{\max\{\Delta^{opt}_{i,g}, \Delta^{opt}_{0,g} - \Delta^{opt}_{i,g}\}}.
\end{align*}
% where, $D_0 = (1-\alpha)\Delta^{opt}_{0}(1:t) + \alpha\Delta^{thd}_0(1:t)$.
\end{lemma}

\begin{proof}
%Recall that $S(1:t) \coloneqq \{i\in[K]: \mu_i(1:t)\}$
We first define the estimated budget as follows:
\begin{align}
    \wZ(1:t) &\coloneqq \sum_{s=1}^{t-1}L_{I_s}(1:s) + L_{u_t}(1:t) - (1 - \alpha)tU_0(1:t-1)\nonumber\\
    %%%%%%%%%%%%%%%%%
    &= \sum_{i=1}^K N_i(1:t-1) L_{i}(1:t-1) + L_{u_t}(1:t) + N_0(1:t-1)U_0(1:t-1) - (1 - \alpha)tU_0(1:t-1) \label{eq:budget-sample1}
\end{align}
Now note that in \cref{eq:budget-sample1} if $N_{0}(1:t-1)<(1-\alpha) t,$ then the last term is negative and $U_{0}(1:t) \geq \mu_{0}(1:t)$. Conversely, if $N_{0}(1:t-1) \geq (1-\alpha) t$ then the constraint is satisfied as:
\begin{align*}
    \sum_{s=1}^{t} \mu_{I_{s}}(1:t) \geq N_{0}(1:t-1) \mu_{0}(1:t-1) \geq(1-\alpha) \mu_{0}(1:t) t
\end{align*}
Let $\tau^{bse}_{0}(1:t) \coloneqq \max\{t'\in[1:t]: I_{t'} = 0, \wZ(1:t'-1) < 0\}$ be the last round the baseline arm is pulled. 
% We now define the baseline stopping time $\tau^{bse}_{0}(1:t)$ as follows:
% \begin{align*}
%     \tau^{bse}_{0,s}(1:t) \coloneqq \max\{t'\in[1:t]: N_0(1:t') = s, \wZ(1:t'-1) < 0\}.
% \end{align*}
Note that at $\tau^{bse}_{0}(1:t)-1$ the safety budget  $\wZ(\tau^{bse}_{0,s}(1:t) -  1) < 0$. In the following proof we denote $\tau = \tau^{bse}_{0}(1:t)$ for notational convenience. It then follows that:
\begin{align*}
    &\sum_{i = 1}^K N_i(1:\tau - 1) L_{i}(1:\tau - 1) + L_{u_{\tau}}(1:\tau) + N_0(1:\tau-1)U_0(1:\tau-1)
    %%%%%%%%%%%%%%%%%%%%
     - (1 - \alpha)\tau U_0(1:\tau) < 0\\
    %%%%%%%%%%%%%%%%%
    &\overset{(a)}{\implies} \sum_{i= 1}^K N_i(1:\tau-1) L_{i}(1:\tau - 1) +  N_0(1:\tau-1)U_0(1:\tau-1)
    %%%%%%%%%%%%%%%%%%
     - (1 - \alpha)\tau U_0(1:\tau) < 0\\
    %%%%%%%%%%%%%%%%%%%
    &\overset{(b)}{\implies} \sum_{i= 1}^K N_i(1:\tau-1) L_{i}(1:\tau) +  N_0(1:\tau-1)U_0(1:\tau)\\
    %%%%%%%%%%%%%%%%%%
    &\quad\quad - (1 - \alpha)\left(\sum_{i=0}^K N_{i}(1:\tau - 1) + 1\right)U_0(1:\tau) < 0\\
    %%%%%%%%%%%%%%%%%%%
    &\overset{}{\implies} \sum_{i= 1}^K N_i(1:\tau - 1)\left[L_{i}(1:\tau) - (1 - \alpha)U_0(1:\tau)\right]  - (1-\alpha)U_0(1:\tau)\\
    %%%%%%%%%%%%%%%%%%
    &\quad\quad + \alpha N_0(1:\tau - 1)U_0(1:N_0(1:\tau)  < 0\\
    %%%%%%%%%%%%%%%%%%%%%%%%
    &\overset{}{\implies}  \underbrace{\alpha N_0(1:\tau - 1)\left[U_0(1:\tau)\right]}_{\textbf{Part A}}    
    %%%%%%%%%%%%%%%%%%
    %\quad\quad 
    < \underbrace{\sum_{i= 1}^K N_i(1:\tau - 1)\bigg[ (1 - \alpha)U_0(1:\tau) - L_{i}(1:\tau)\bigg] + (1-\alpha)U_0(1:\tau) }_{\textbf{Part B}}
\end{align*}
where, $(a)$ follows by dropping $L_{u_t}(1:\tau) > 0$, and $(b)$ follows by introducing 
\begin{align*}
    \tau = \left(\sum_{i=0}^K N_{i}(1:\tau - 1) + 1\right).
\end{align*}
Now note that at the event $\bigcap_{i=1}^K\xi_i^{opt}(1:t)$, $\mu_0(1:\tau) < U_0(1:\tau)$. Hence, we can lower bound Part A as follows:
\begin{align*}
    \text{Part A} \leq \alpha N_0(1:\tau - 1)\left[\mu_0(1:\tau) \right].
\end{align*}
Similarly we can upper bound Part B as follows:
\begin{align*}
    \text{Part B} \! &\overset{}{=} \sum_{i=1}^K N_i(1:\tau - 1)\bigg[ (1 - \alpha)U_0(1:\tau) - L_{i}(1:\tau)\bigg] + (1-\alpha)U_0(1:\tau)  \\
    %%%%%%%%%%%%%%%%%%%
    %&\!\!\!\!\!\!+\sum_{i\in\Vs^C(1:\tau)}\!\!\!\!\!\! N_i(1:\tau - 1)\bigg[ (1 - \alpha)U_0(1:\tau) + \alpha B - L_{i}(1:\tau)\bigg] + (1-\alpha)U_0(1:\tau) + \alpha B \\
    %%%%%%%%%%%%%%%%%%%
    % &\overset{(a)}{=} \!\!\!\!\!\!\sum_{i\in\Vs^C(1:\tau)}\!\!\!\!\!\!\!\!\! N_i(1:\tau - 1)\bigg[ (1 - \alpha)U_0(1:\tau) + \alpha B - L_{i}(1:\tau)\bigg] + (1-\alpha)U_0(1:\tau) + \alpha B\\
    % %%%%%%%%%%%%%%%%%%%%
    % &\overset{(b)}{=} \!\!\!\!\!\!\sum_{i\in\S^C(1:\tau)}\!\!\!\!\!\!\!\!\! N_i(1:\tau - 1)\bigg[ (1 - \alpha)U_0(1:\tau) + \alpha B - L_{i}(1:\tau)\bigg] + (1-\alpha)U_0(1:\tau) + \alpha B
\end{align*}
%where, $(a)$ follows as the learner samples the baseline arm when $s\in\Vs(1:\tau)$. 
Now note that at $\bigcap_{i=1}^K\xi_i^{opt}(1:\tau)$ we have
\begin{align*}
    U_0(1:\tau) \leq \mu_0(1:\tau) + \sqrt{\dfrac{2\log(4\log_2(\tau + 1)/\delta)}{N_i(1:\tau)}} \overset{(a)}{\leq} \mu_0(1:\tau) + \dfrac{\Delta^{opt}_{0}(1:\tau)}{2}
\end{align*}
where, $(a)$ follows for $N_i(1:\tau) \geq \dfrac{8\log(4\log_2(\tau + 1)/\delta)}{(\Delta^{opt}_{i}(1:\tau))^2}$ with probability greater than $1-2\delta$. Then for the Part B we can show that %\S^C(1:\tau)
\begin{align*}
    \text{Part B} &\leq \sum_{i\in[K]} N_i(1:\tau - 1)\bigg[ (1 - \alpha)\left(\mu_0(1:\tau) + \dfrac{\Delta^{opt}_{0}(1:\tau)}{2}\right) - L_{i}(1:\tau)\bigg]
    %%%%%%%%%%%%%
    + (1-\alpha)\left(\mu_0(1:\tau) + \dfrac{\Delta^{opt}_{0}(1:\tau)}{2}\right) \\
    %%%%%%%%%%%%%%%%%%%%
    &\leq \sum_{i\in[K]} N_i(1:\tau - 1)\bigg[ (1 - \alpha)\left(\mu_0(1:\tau) + \dfrac{\Delta^{opt}_{0}(1:\tau)}{2}\right)  - \mu_{i}(1:\tau) \\
    %%%%%%%%%%%%%
    &+ \sqrt{\dfrac{2\log(4\log_2(\tau + 1)/\delta)}{N_i(1:\tau-1)}}\bigg] + (1-\alpha)\left(\mu_0(1:\tau) + \dfrac{\Delta^{opt}_{0}(1:\tau)}{2}\right) \\
    %%%%%%%%%%%%%%
    &\overset{(a)}{=} \sum_{i\in[K]}N_i(1:\tau-1)c_i + \sqrt{N_i(1:\tau-1) \left(2\log(4\log_2(\tau + 1)/\delta)\right)}
    %%%%%%%%%%%%%%%%%
    + (1 - \alpha)\left(\mu_0(1:\tau) + \dfrac{\Delta^{opt}_0(1:\tau)}{2}\right) \\
    %%%%%%%%%%%%%%%%%%
    &\overset{(b)}{\leq} \sum_{i\in[K]}\dfrac{8c_i\log(4\log_2(\tau + 1)/\delta)}{(\Delta^{opt}_{i}(1:\tau))^2} + \dfrac{8\log(4\log_2(\tau + 1)/\delta)}{(\Delta^{opt}_{i}(1:\tau))} + (1 - \alpha)\left(\mu_0(1:\tau) + \dfrac{\Delta^{opt}_0(1:\tau)}{2}\right) 
\end{align*}
where, in $(a)$ we define $c_i = (1 - \alpha)\left(\mu_0(1:\tau) + \dfrac{\Delta_0(1:\tau)}{2}\right)  - \mu_i(1:\tau)$, and $(b)$ follows by setting $N_i(1:\tau-1) \geq \dfrac{8\log(4\log_2(\tau + 1)/\delta)}{(\Delta^{opt}_{i}(1:\tau))^2}$. Now for $c_i \geq 0$ we can show that
\begin{align*}
    &(1 - \alpha)(\mu_0(1:\tau) + \dfrac{\Delta^{opt}_0(1:\tau)}{2}) - \mu_i(1:\tau) \geq 0\\
    %%%%%%%%%%%%%%%%
    \implies& \mu_{i^*}(1:\tau) - (1-\alpha)\left(\mu_0(1:\tau) + \dfrac{\Delta^{opt}_0(1:\tau)}{2}\right) \leq \mu_{i^*}(1:\tau) - \mu_i(1:\tau)\\
    %%%%%%%%%%%%%%%%
    \implies& \Delta_i(1:\tau) \geq (1+\alpha)\dfrac{\Delta^{opt}_0(1:\tau)}{2} + \alpha(\mu_0(1:\tau))
\end{align*}
Again for $c_i < 0$ we can show that
\begin{align*}
    &(1 - \alpha)(\mu_0(1:\tau) + \dfrac{\Delta^{opt}_0(1:\tau)}{2})  - \mu_i(1:\tau) < 0\\
    %%%%%%%%%%%%%%%%
    \implies& \mu_{i^*}(1:\tau) - (1-\alpha)\left(\mu_0 + \dfrac{\Delta^{opt}_0(1:\tau)}{2} \right) > \mu_{i^*}(1:\tau) - \mu_i(1:\tau)\\
    %%%%%%%%%%%%%%%%
    \implies& \Delta_i(1:\tau) < (1+\alpha)\dfrac{\Delta^{opt}_0(1:\tau)}{2} + \alpha(\mu_0(1:\tau) )
\end{align*}
Combining everything we can show that
\begin{align*}
    \alpha N_0(1:\tau - 1)\left[\mu_0(1:\tau)\right] &\leq \sum_{i\in[K]}\dfrac{8c_i\log(4\log_2(\tau + 1)/\delta)}{(\Delta^{opt}_{i}(1:\tau))^2} + \dfrac{8\log(4\log_2(\tau + 1)/\delta)}{(\Delta^{opt}_{i}(1:\tau))}
    %%%%%%%%%%%%%%%%%%%%
    + (1 - \alpha)\left(\mu_0(1:\tau) + \dfrac{\Delta^{opt}_0(1:\tau)}{2}\right)\\
    %%%%%%%%%%%%%%%%%%%
    N_0(1:\tau) &\overset{(a)}{\leq} \dfrac{1}{\alpha\mu_{0}(1:\tau)}\sum_{i\in[K]}\dfrac{16\log(4\log_2(\tau + 1)/\delta)}{\max\{\Delta^{opt}_{i}(1:\tau), 0.5(1-\alpha)\Delta^{opt}_{0}(1:\tau) + \alpha\mu_0(1:\tau) - \Delta^{opt}_i(1:\tau)\}}\\
    %%%%%%%%%%%%%%%%%%%%%%%%
    N_0(1:\tau) &\overset{}{\leq} \dfrac{1}{\alpha\mu_{0}(1:\tau)}\sum_{i\in[K]}\dfrac{16\log(4\log_2(\tau + 1)/\delta)}{\max\{\Delta^{opt}_{i}(1:\tau), \Delta^{opt}_0(1:\tau) - \Delta^{opt}_i(1:\tau)\}}
\end{align*}
where, $(a)$ follows as $\!\!\alpha(\mu_0(1:t) - B) \!\!\geq  \!\!\dfrac{\alpha(\mu_0(1:t)\! - \! B)}{2}$.
% and in $(b)$ we substitute $D_0 = (1-\alpha)\Delta^{opt}_{0}(1:t) + \alpha\Delta^{thd}_0(1:t)$.
Now considering the $g$-th changepoint we can show that
\begin{align*}
    N^{bse}_0(1:\tau) = N^{opt}_{0,g} =  \dfrac{1}{\alpha\mu_{0,g}}\sum_{i\in[K]}\dfrac{16\log(4\log_2(\tau + 1)/\delta)}{\max\{\Delta^{opt}_{i,g}, \Delta^{opt}_{0,g} - \Delta^{opt}_{i,g}\}}.
\end{align*}
The claim of the lemma follows.
\end{proof}

\subsection{Lower Bound in Safe Global Changepoint Setting}
\label{app:lower-bound}
\begin{customtheorem}{3}\textbf{(Restatement)}
Let $\mathcal{E}$, $\overline{\mathcal{E}}$ be two bandit environment and there exits a global changepoint at $t_{c_1} = T/2$. Let $\alpha>0$ be the safety parameter and $\mu_{0,\min}$ be the mean of the minimum safety mean over the changepoint segments. Then the lower bound is given by
\begin{align*}
    \E_{\mathcal{E}, \overline{\mathcal{E}}} [R_{T}] \geq \left\{\frac{K}{(16 e+8) \alpha \mu_{0, \min}} + \dfrac{\log T}{\alpha \mu_{0,\min}}, \frac{\sqrt{K T}}{\sqrt{32 e + 16}} + \dfrac{\log T}{\alpha \mu_{0,\min}}\right\}.
\end{align*}
\end{customtheorem}

\textbf{Step 1 (Safe Risk):}
Consider a setting with a single changepoint $g$. For the segment $\rho_0$ from $t=1$ to $t_{c_1}-1$ define the environment $\mathcal{E}\in[0,1]^K$ such that $\mu_{i}(1:t_{c_1}-1)=\mu_{0}(1:t_{c_1}-1)-\Delta$ for all $i \in[K]$. We assume that $\mu_{0}(1:t_{c_1}-1)$ and $\Delta$ are such that $\mu_{i}(1:t_{c_1}-1) \geq 0$. Define environment $\mathcal{E}^{(i)}$ for each $i=1, \ldots, K$ by
\begin{align*}
\mathcal{E}_{j}^{(i)}= \begin{cases}\mu_{0}(1:t_{c_1}-1)+\Delta, & \text { for } j=i \\ \mu_{0}(1:t_{c_1}-1)-\Delta, & \text { otherwise. }\end{cases}
\end{align*}
Similarly for the segment $\rho_1$ from $t_{c_1}$ to $T$ define the environment $\overline{\mathcal{E}}\in[0,1]^K$ such that $\mu_{i}(t_{c_1}:T)=\mu_{0}(t_{c_1}:T)-\Delta$ for all $i \in[K]$. Again, we assume that $\mu_{0}(t_{c_1}:T)$ and $\Delta$ are such that $\mu_{i}(t_{c_1}:T) \geq 0$. Define environment $\overline{\mathcal{E}^{(i)}}$ for each $i=1, \ldots, K$ by
\begin{align*}
\overline{\mathcal{E}}_{j}^{(i)}= \begin{cases}\mu_{0}(t_{c_1}:T)+\Delta', & \text { for } j=i \\ \mu_{0}(t_{c_1}:T)-\Delta', & \text { otherwise. }\end{cases}
\end{align*}
where, $\Delta' > 0$. 
Note that $\mathcal{E}$ and $\mathcal{E}^{(i)}$ differ only in the $i$-th component: $\mu_{i}(1:t_{c_1}-1)=\mu_{0}(1:t_{c_1}-1)-\Delta$ whereas $\mu_{i}^{(i)}(1:t_{c_1}-1)=\mu_{0}(1:t_{c_1}-1)+\Delta$. Then the KL divergence between the reward distributions of the $i$ th arms is given by
\begin{align*}
    \operatorname{KL}\left(\mathcal{E}_{i}, \mathcal{E}_{i}^{(i)}\right)=(2 \Delta)^{2} / 2=2 \Delta^{2}
\end{align*}
% $\operatorname{KL}\left(\mu_{i}, \mu_{i}^{(i)}\right)=(2 \Delta)^{2} / 2=2 \Delta^{2}$.
Using Theorem $9$ of \citep{wu2016conservative} (see \Cref{prop:conservative-lower}) we can show that in $\mathcal{E}$ the lower bound to the safe regret is given by
\begin{align}
\E_{\mathcal{E}} [R_{t_{c_1}}] \geq\left\{\begin{array}{l}\frac{\sqrt{K t_{c_1}}}{\sqrt{16 e+8}}, \quad \text { when } \alpha \geq \frac{\sqrt{K}}{\mu_{0}(1:t_{c_1}-1) \sqrt{(16 e+8) t_{c_1}}} \\ \frac{K}{(16 e+8) \alpha \mu_{0}(1:t_{c_1}-1)}, \quad \text { otherwise }\end{array}\right. \label{eq:regret:1}
\end{align}
and for $\overline{\mathcal{E}}$ the lower bound is given by
\begin{align}
\E_{\overline{\mathcal{E}}} [R_{T-t_{c_1}}] \geq\left\{\begin{array}{l}\frac{\sqrt{K (T-t_{c_1})}}{\sqrt{16 e+8}}, \quad \text { when } \alpha \geq \frac{\sqrt{K}}{\mu_{0}(t_{c_1}:T) \sqrt{(16 e+8) (T -t_{c_1})}} \\ \frac{K}{(16 e+8) \alpha \mu_{0}(t_{c_1}:T)}, \quad \text { otherwise }\end{array}\right. \label{eq:regret:2}
\end{align}

\textbf{Step 2 (Changepoint risk):} Now we introduce the changepoint detection framework. 
In this framework define a false alarm rate constraint specified by a tuple $(m, \alpha)$, where $m \in \mathbb{Z}^{+}$denotes an a priori fixed time such that the probability of an admissible change detector stopping before time $m$ is at most $\delta \in(0,1)$, namely, $\mathbb{P}^{}[\tau<m] \leq \delta .$ This framework is similar to what defined in \citep{gopalan2021bandit}. Then we know from Theorem $2$ of \citep{gopalan2021bandit} (see \Cref{prop:lower-bound-changepoint}) that
for any bandit changepoint algorithm satisfying $\mathbb{P}^{}[\tau<m] \leq \delta$, we have
\begin{align}
    \mathbb{E}_{\mathcal{E}, \overline{\mathcal{E}}}[\tau] \geq \min \left\{\frac{\frac{1}{20} \log \frac{1}{\delta}}{\max _{i \in [K]} \operatorname{KL}\left(\mathcal{E}_{i}, \overline{\mathcal{E}}_{i}^{}\right)}, \frac{m}{2}\right\} = \min \left\{\frac{\frac{1}{10} \log \frac{1}{\delta}}{\underline{\Delta}^2}, \frac{m}{2}\right\}.
    % \overset{(a)}{=} \min \left\{\frac{ 2\alpha^2\mu_0^{2}(1:t_{c_1}-1)T^2\log \frac{1}{\delta}}{K^2}, \frac{m}{2}\right\} 
    \label{eq:regret:3}
\end{align}
where $\underline{\Delta} > 0$ is the minimum changepoint gap between any arms in $\mathcal{E}, \overline{\mathcal{E}}$. 
% where, $(a)$ follows by setting $\Delta=K / 4 \alpha \mu_{0} n$, the minimal gap used in \citep{wu2016conservative}.

\textbf{Step 3:} Combining the two steps above and using \eqref{eq:regret:1}, \eqref{eq:regret:2}, and \eqref{eq:regret:3} we can show that the total regret lower bound is 
\begin{align*}
&\E_{\mathcal{E}, \overline{\mathcal{E}}} [R_{T}] \geq\left\{\begin{array}{l} \E_{\overline{\mathcal{E}}} [R_{t_{c_1}}] + \E_{\overline{\mathcal{E}}} [R_{T-t_{c_1}}] + \underline{\Delta}\E_{\mathcal{E}, \overline{\mathcal{E}}}[\tau] , \quad \text { when } \alpha \geq \frac{\sqrt{K}}{\mu_{0}(1:t_{c_1}-1) \sqrt{(16 e+8) t_{c_1}}} \\ \E_{\overline{\mathcal{E}}} [R_{t_{c_1}}] + \E_{\overline{\mathcal{E}}} [R_{T-t_{c_1}}] + \underline{\Delta}\E_{\mathcal{E}, \overline{\mathcal{E}}}[\tau], \quad \text { otherwise }\end{array}\right. \\
%%%%%%%%%%%%%%%%%%%%%%%%%%%%%%%%
%%%%%%%%%%%%%%%%%%%%%%%%%%%%
&\implies \E_{\mathcal{E}, \overline{\mathcal{E}}} [R_{T}] \geq\left\{\begin{array}{l} \frac{\sqrt{K (T-t_{c_1})}}{\sqrt{16 e+8}} + \frac{\sqrt{K (T-t_{c_1})}}{\sqrt{16 e+8}} + \min \left\{\frac{\frac{1}{10} \log \frac{1}{\delta}}{\underline{\Delta}}, \frac{m\underline{\Delta}}{2}\right\} , \quad \text { when } \alpha \geq \frac{\sqrt{K}}{\mu_{0}(1:t_{c_1}-1) \sqrt{(16 e+8) t_{c_1}}} \\ \frac{K}{(16 e+8) \alpha \mu_{0}(1:t_{c_1}-1)} + \frac{K}{(16 e+8) \alpha \mu_{0}(t_{c_1}:T)} + \min \left\{\frac{\frac{1}{10} \log \frac{1}{\delta}}{\underline{\Delta}}, \frac{m\underline{\Delta}}{2}\right\}, \quad \text { otherwise }\end{array}\right.\\
%%%%%%%%%%%%%%%%%%%%%%%%%%%%%
%%%%%%%%%%%%%%%%%%%%%%%%%%%%%%%%%
&\implies \E_{\mathcal{E}, \overline{\mathcal{E}}} [R_{T}] \geq\left\{\begin{array}{l} \frac{\sqrt{K (T-t_{c_1})}}{\sqrt{16 e+8}} + \frac{\sqrt{K (T-t_{c_1})}}{\sqrt{16 e+8}} + \dfrac{\log T}{\alpha \mu_{0,\min}} , \quad \text { when } \alpha \geq \frac{\sqrt{K}}{\mu_{0}(1:t_{c_1}-1) \sqrt{(16 e+8) t_{c_1}}} \\ \frac{K}{(16 e+8) \alpha \mu_{0}(1:t_{c_1}-1)} + \frac{K}{(16 e+8) \alpha \mu_{0}(t_{c_1}:T)} + \dfrac{\log T}{\alpha \mu_{0,\min}}, \quad \text { otherwise }\end{array}\right.
\end{align*}
where, $(a)$ follow by setting $\underline{\Delta} = \alpha\mu_{0,\min}$, $m=T$ as the last time the changepoint needs to be detected, and $\delta = \frac{1}{T}$.
Combining everything and setting $t_{c_1} = T/2$ we can show that
\begin{align*}
    E_{\mathcal{E}, \overline{\mathcal{E}}} [R_{T}] \geq \left\{\frac{K}{(16 e+8) \alpha \mu_{0, \min}} + \dfrac{\log T}{\alpha \mu_{0,\min}}, \frac{\sqrt{K T}}{\sqrt{32 e + 16}} + \dfrac{\log T}{\alpha \mu_{0,\min}}\right\}.
\end{align*}

\subsection{Additional Experiment Details}
\label{app:addl-expt}

\textbf{Setting for choosing  different $\alpha$:} In this experiment we illustrate the tension between the risk parameter $\alpha$ and the cumulative regret for various values of $\alpha$. The \lcs environment consist of $4$ arms (including baseline) and the evolution of means are shown in \Cref{fig:expt3} (Left). We plot the regret versus various values of $\alpha\in[0,1]$ in \Cref{fig:expt3} (Right). We see that \loc has high regret in high risk ($\alpha \rightarrow 0$) setting. However its performance lies between \ducb and \glrucb which are safety oblivious algorithms. \loc outperforms both \umoss and safety aware \cucb algorithm. We see that \cucb only performs well when risk is low ($\alpha \rightarrow 1$). This is the changepoints occur too fast and when ($\alpha \rightarrow 0$) the \cucb is forced to choose the baseline arm always. 

%Similarly in \Cref{fig:expt3} (Right) we plot the regret versus various values of $\alpha\in(0,1]$ for the environment in \Cref{fig:expt3} (Left). 
%We see that \loc has high regret in high risk setting ($B \rightarrow 1$) for a fixed risk $\alpha = 0.6$. In both the environment \cucb performs badly as it is not suited for the new safety constraints under piecewise i.i.d. setting.

% \begin{figure*}[!ht]
% \centering
% \begin{tabular}{cc}
% \label{fig:global_env}\includegraphics[scale = 0.39]{img/global_env.png} &
% %%%%%%%%%%
% \label{fig:local_env}\includegraphics[scale = 0.39]{img/local_env.png}  
% \end{tabular}
% \caption{(Left) Global changepoint environment with $T=8000$, $K^+ = 6$ and changepoints at $t = 2000, 4000$ and $6000$. (Middle) Local changepoint environment with $T=8000$, $K^+ = 6$ and changepoints at $t = 2000, 4000$ and $6000$. Note that some arms do not change at these changepoints. (Right) Local changepoint environment with $T=1500$, $K^+ = 4$ and changepoints at $t = 500, 1000$ and $1500$. Again note that some arms do not change at these changepoints.}
% \label{fig:exptenv}
% %\vspace{-1.em}
% \end{figure*}

\begin{figure*}[!ht]
\centering
\begin{tabular}{cc}
%%%%%%%%%%%
%%%%%%%%%%
\label{fig:alpha_env}\includegraphics[scale = 0.39]{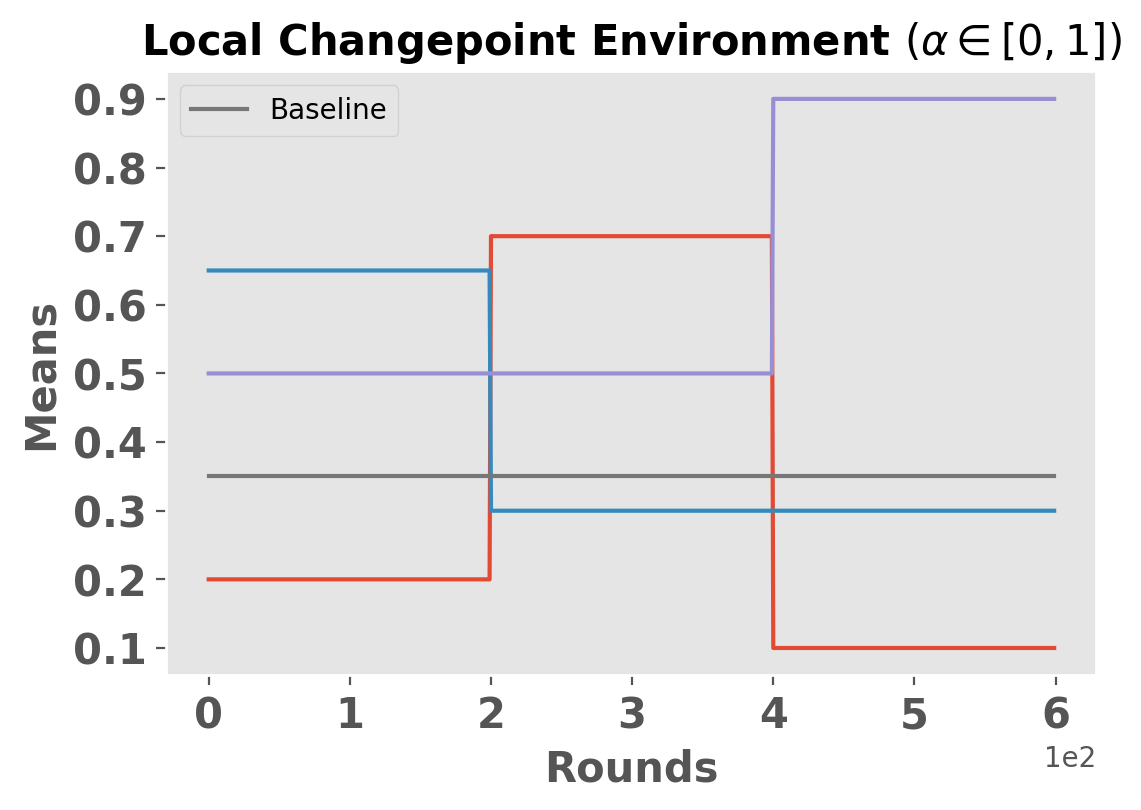} & %%%%%%%%%%
\label{fig:alpha}\includegraphics[scale = 0.39]{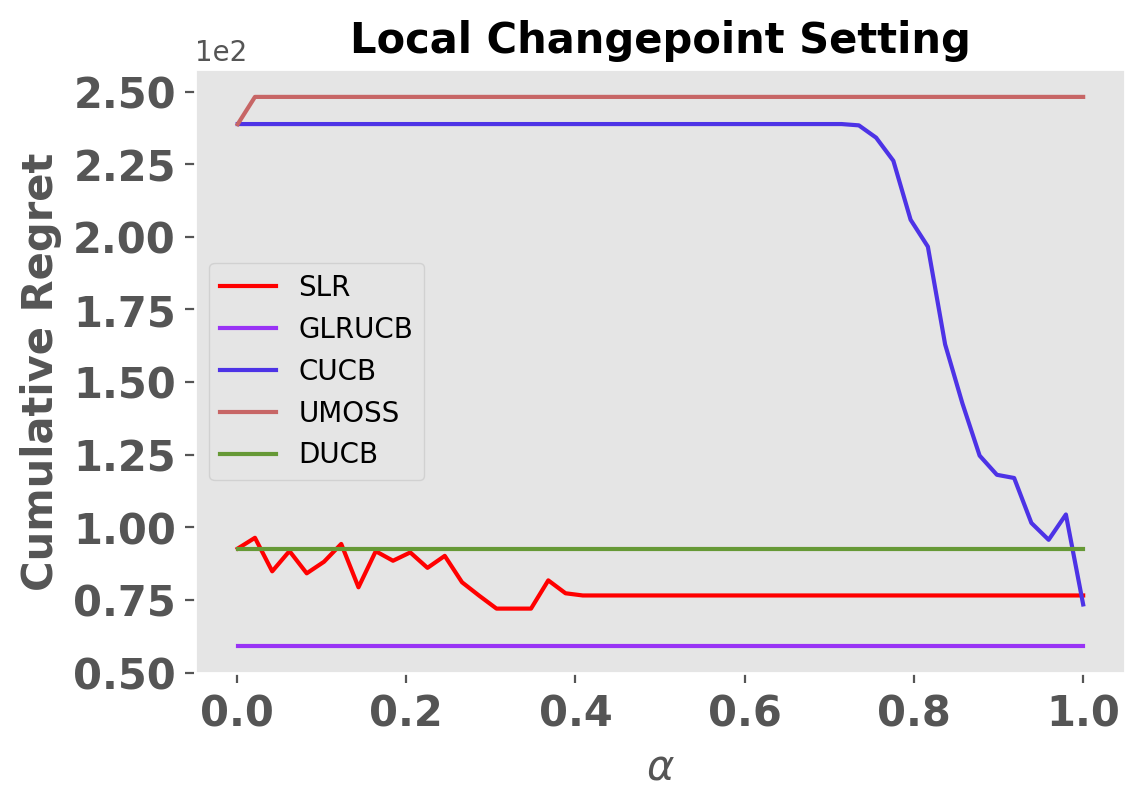}
%%%%%%%%%%%
% \label{fig:threshold}\includegraphics[scale=0.39]{img/threshold.png}
\end{tabular}
\caption{(Left) Local changepoint environment with $T=1500$, $K^+ = 4$ and changepoints at $t = 500, 1000$ and $1500$. Again note that some arms do not change at these changepoints. (Right) Shows the regret vs. $\alpha\in[0,1]$ plot.}
\label{fig:expt3}
%\vspace{-1.em}
\end{figure*}

\begin{remark}
In all the environments we tested against adversarial \cucb as recommended in \citet{wu2016conservative} for adversarial setting. The adversarial \cucb again always samples the baseline arm as it is not suited for the safety constraint \eqref{eq:constraint1} and achieves linear regret similar to \cucb. So we omit it's plot from the figures.

\end{remark}

\textbf{Time Complexity} Note that \adswitch suffers an expensive time complexity of $\Theta(KT^4)$ for performing the changepoint detection test whereas \glrucb, \ucbcpd, \mucb, \glb, \loc suffers a time complexity of $\Theta(KT^2)$.

\textbf{Algorithmic Implementations:} We implement the following algorithms:
%\begin{itemize}
%     \item \textbf{\glrucb \citep{besson2019generalized}:} This algorithm is an actively adaptive changepoint detection algorithm. It uses the GLRT statistic to detect changepoints and restart. It is implemented with the Gaussian KL function with known variance to calculate the GLRT statistic. It works for both \gcs and \lcs environment and requires the horizon $T$ as input.
%     \item \textbf{\ucbcpd \citep{mukherjee2019distribution}:} This algorithm is also an actively adaptive changepoint detection algorithm. It uses confidence based scan statistic and we use the $\beta_i(t,\delta)$ in \eqref{eq:beta-delta} to implement its changepoint detector as opposed to the time-uniform bound used by \citet{mukherjee2019distribution}. The performance difference is negligible for such a short horizon.
%     \item \textbf{\ducb \citep{garivier2011kl}:}
%     \item \textbf{\cucb \citep{wu2016conservative}:}
    
%\end{itemize}

1) \textbf{\glrucb \citep{besson2019generalized}:} This algorithm is an actively adaptive changepoint detection algorithm. It uses the GLRT statistic to detect changepoints and restart. It is implemented with the Gaussian KL function with known variance to calculate the GLRT statistic. It works for both \gcs and \lcs environment and requires the horizon $T$ as input.

2) \textbf{\ucbcpd \citep{mukherjee2019distribution}:} This algorithm is also an actively adaptive changepoint detection algorithm. It uses confidence based scan statistic and we use the $\beta_i(t,\delta)$ in \eqref{eq:beta-delta} to implement its changepoint detector as opposed to the time-uniform bound used by \citet{mukherjee2019distribution}. The performance difference is negligible for such a short horizon. It requires the horizon $T$ as input.

3) \textbf{\ucbcpde:} This algorithm is also an actively adaptive changepoint detection algorithm. \ucbcpde is similar to \ucbcpd but suited for \lcs as it conducts forced exploration of arms. We again use the same $\beta_i(t,\delta)$ in \eqref{eq:beta-delta} to implement its changepoint detector. It also requires the horizon $T$ as input.

4) \textbf{\ducb \citep{garivier2011kl}:} This is a passive algorithm that does not detect the changepoints and restart. We set the exploration parameter $\gamma=1-\frac{1}{4} \sqrt{\frac{1}{T}}$ as recommended in \citet{garivier2011kl}.

5) \textbf{\cucb \citep{wu2016conservative}:} This is the safety aware algorithm for the stochastic bandit setting. It uses the new safety constraint in \eqref{eq:constraint1} instead of the constraint in \citet{wu2016conservative} (see \Cref{prop:conservative-upper}). So it calculates the new empirical safety budget $\wZ(1:t)$ \eqref{eq:safe-budget} and when the safety budget is positive it uses the UCB-index \citep{auer2002finite} to sample an arm, otherwise it samples the baseline.

6) \textbf{Adversarial \cucb \citep{wu2016conservative}:} This is a safety aware algorithm for the adversarial bandit setting. It uses the new safety constraint in \eqref{eq:constraint1} instead of the constraint in \citet{wu2016conservative}. So it calculates the new empirical safety budget $\wZ(1:t)$ \eqref{eq:safe-budget} and when the safety budget is positive it uses the EXP3 \citep{auer2002nonstochastic} to sample an arm, otherwise it samples the baseline.

7) \textbf{\umoss \citep{lattimore2015pareto}}: The Unbalanced Moss (\umoss) is a conservative exploration bandit. The \umoss is tuned with the parameter
\begin{align*}
    \tilde{B}_{0} = \frac{T K}{\sqrt{T K}+\frac{K}{\alpha \mu_{0}}}, \quad \tilde{B}_{i} = \tilde{B}_{K}=\sqrt{T K}+\frac{K}{\alpha \mu_{0}}.
\end{align*}
The quantity $\tilde{B}_{i}$ determines the regret of \umoss with respect to arm $i$ up to constant factors, and must be chosen to lie inside the Pareto frontier given by \citep{lattimore2015pareto}. \umoss has no guarantees for the safety constraint \eqref{eq:constraint1}. It was found to perform comparably to the highly constrained \cucb in \citet{wu2016conservative}. \umoss requires $\tilde{B}_{0}, \ldots,\tilde{B}_{K}$ as inputs and so it requires the knowledge of the mean of the baseline.

%\fxnote{Time and memory complexity}
% \fxnote{mention cucb follows new safety constraint}
% \fxnote{note on algo implementation}

\newpage
\subsection{Table of Notations}
\label{app:notation}
\begin{table}[!th]
    \centering
    \begin{tabular}{c|c}
    \textbf{Notation} & \textbf{Definition}\\\hline\\
        $K$ & Total number of arms\\
        $T$ & Horizon \\
        $\delta$ & Probability of error\\
        % $\tau_\delta$ & Stopping time \\
        $\beta_i(1:t,\delta)$ & $\sqrt{2\log(4\log_2(t+1)/\delta)/N_i(1:t)}$\\
        $B(T,\delta)$ & $16\log(4\log_2(T/\delta))$\\
        $\alpha$ & Risk Parameter \\
        % $B$ & Threshold\\
        $G_T$ & Total Changepoints till horizon $T$\\
        $I_s$ & arm sampled at round $s$\\
        $1:t$ & Rounds from $1$ to $t$ \\
        $\rho_g$ & $t_{c_g}:t_{c_{g+1}} - 1$ \\
        $\Delta^{opt}_{i,g}$ & $\mu_{i^*,g} - \mu_{i,g}$ for segment $\rho_g$ \\
        $\Delta^{chg}_{i,g}$ & $|\mu_{i,g} - \mu_{i,g+1}|$ for segment $\rho_g$\\
        % $\Delta^{thd}_{i,g}$ & $|\mu_{i,g} - \alpha B|$ for segment $\rho_g$\\
        % $D^{bse}_0$ & $\dfrac{1}{\alpha\Delta^{thd}_0}\sum\limits_{i=1}^K\dfrac{\Delta^{opt}_{i}(1:t)}{\max\{\Delta^{opt}_{i}(1:t), D_0 - \Delta^{opt}_i(1:t)\}}$\\
        %%%%%%%%%%%%%%%
        $H^{(1)}_{i,g-1}$ & $\max\left\{\frac{1}{\Delta^{opt}_{i,g-1}}, \frac{\Delta^{opt}_{i,g-1}}{\left(\Delta^{chg}_{i,g-1}\right)^2}\right\}$\\
        %%%%%%%%%%%%%%
        $H^{(2)}_{i,g}$ & $\max_{j\in[K]^+}\frac{\Delta^{opt}_{i, g}}{\left(\Delta^{chg}_{j,g}\right)^2}$\\
        %%%%%%%%%%%%%%%
        $\overline{H^{(2)}_{i,g}}$ & $\frac{\Delta^{opt}_{i, g}}{\left(\Delta^{chg}_{i,g}\right)^2}$\\
        %%%%%%%%%%%%%%
        $H^{(3)}_{i,g}$ & $\frac{\Delta^{opt}_{\max, g}}{\max\{\Delta^{opt}_{i,g}, \Delta^{opt}_{0,g} - \Delta^{opt}_{i,g}\}}$\\
        %%%%%%%%%%%%%%
        $N^{opt}_{i,g}$ & Critical number of samples for Optimality Detection at $\rho_g$\\
        $N^{chg}_{i,g}$ & Critical number of samples for Changepoint Detection at $\rho_g$\\
        % $N^{thd}_{i,g}$ & Critical number of samples for Thresholding Detection at $\rho_g$\\
        $N^{bse}_{0,g}$ & Total number of samples for baseline at $\rho_g$\\
        % $N^{sfe}_{i,g}$ & Total number of samples for safe arm at $\rho_g$\\
        \\\hline
    \end{tabular}
    \caption{Table of Notations}
    \label{tab:notations}
\end{table}

\end{document}